\documentclass{article}
\usepackage[a4paper, total={5.5in, 8.5in}]{geometry}
\usepackage{amsmath,amsfonts,bm}

\def\eqref#1{equation~\ref{#1}}

\def\1{\bm{1}}

\DeclareMathAlphabet{\mathsfit}{\encodingdefault}{\sfdefault}{m}{sl}
\SetMathAlphabet{\mathsfit}{bold}{\encodingdefault}{\sfdefault}{bx}{n}

\newcommand{\E}{\mathbb{E}}

\DeclareMathOperator*{\argmax}{arg\,max}

\DeclareMathOperator{\sign}{sign}

\usepackage{url}
\usepackage[hidelinks,colorlinks=true,linkcolor=blue,citecolor=blue]{hyperref}

\usepackage{amsmath}
\usepackage{amssymb}
\usepackage{mathtools}
\usepackage{amsthm}

\usepackage{url}

 \usepackage{mathrsfs}

\usepackage{amsmath}
\usepackage{amssymb}
\usepackage{mathtools}
\usepackage{amsfonts}
\usepackage{amsthm,bm}

\usepackage{thmtools,thm-restate}

\usepackage{natbib}

  \usepackage{enumitem}

\newtheorem{theorem}{Theorem}[section]

\newtheorem{lemma}{Lemma}[section]

\newtheorem{assumption}{Assumption}

\newtheorem{parametrization}{Parametrization}[section]

\usepackage{algorithm}
\usepackage[noend]{algorithmic}

\usepackage{letltxmacro}
\newlength{\commentindent}
\makeatletter
\renewcommand{\algorithmiccomment}[1]{\unskip\hfill\makebox[\commentindent][l]{//~#1}\par}
\LetLtxMacro{\oldalgorithmic}{\algorithmic}
\renewcommand{\algorithmic}[1][0]{%
  \oldalgorithmic[#1]%
  \renewcommand{\ALC@com}[1]{%
    \ifnum\pdfstrcmp{##1}{default}=0\else\algorithmiccomment{##1}\fi}%
}
\makeatother

\usepackage{microtype}
\usepackage{graphicx,subcaption}
\usepackage{booktabs} 

\usepackage[textwidth=1.2in]{todonotes}

\usepackage{cleveref}

\AtBeginDocument{%
}

 \usepackage{pifont}

\DeclareMathOperator{\sig}{Sigmoid}
 \DeclareMathOperator{\polylog}{polylog}
  \DeclareMathOperator{\polyloglog}{polyloglog}
    \DeclareMathOperator{\polylogloglog}{polylogloglog}
 \DeclareMathOperator{\poly}{poly}

\usepackage{nicefrac}       

\usepackage{wrapfig}

\title{Dissecting adaptive methods in GANs}

\author{Samy Jelassi$^{1}$, David Dobre$^{2}$, Arthur Mensch$^{3}$, Yuanzhi Li$^{4}$, Gauthier Gidel$^{2,5}$}
\date{%
    $^1$Princeton University\\
    $^2$ Mila \& Universit\'{e} de Montréal\\
    $^3$ Deepmind\\
    $^4$ Carnegie Mellon University\\
    $^5$ Canada CIFAR AI Chair
}

\begin{document}

\maketitle

\begin{abstract}
  Adaptive methods are a crucial component widely used for training generative adversarial networks (GANs). 
While there has been some work to pinpoint the “marginal value of adaptive methods” in standard tasks, it remains unclear why they are still critical for GAN training. 
In this paper, we formally study how adaptive methods help train GANs; inspired by the grafting method proposed in \cite{agarwal2020disentangling}, we separate the magnitude and direction components of the Adam updates, and graft them to the direction and magnitude of SGDA updates respectively. 
By considering an update rule with the magnitude of the Adam update and the normalized direction of SGD, we empirically show that the adaptive magnitude of Adam is key for GAN training. 
This motivates us to have a closer look at the class of normalized stochastic gradient descent ascent (nSGDA) methods in the context of GAN training. 
We propose a synthetic theoretical framework to compare the performance of nSGDA and SGDA for GAN training with neural networks. 
We prove that in that setting, GANs trained with nSGDA recover all the modes of the true distribution, whereas the same networks trained with SGDA (and any learning rate configuration) suffer from mode collapse. 
The critical insight in our analysis is that normalizing the gradients forces the discriminator and generator to be updated at the same pace. 
We also experimentally show that for several datasets, Adam's performance can be recovered with nSGDA methods.

\end{abstract}
\vspace{-4mm}
\section{Introduction}
\vspace{-1mm}
Adaptive algorithms have become a key component in training modern neural network architectures in various deep learning tasks. 
Minimization problems that arise in natural language processing \citep{vaswani2017attention}, fMRI \citep{zbontar2018fastmri}, or min-max problems such as generative adversarial networks (GANs) \citep{goodfellow2014generative} almost exclusively use adaptive methods, and it has been empirically observed that Adam \citep{kingma2014adam} yields a solution with better generalization than stochastic gradient descent (SGD) in such problems \citep{choi2019empirical}. 
Several works have attempted to explain this phenomenon in the minimization setting. 
Common explanations are that adaptive methods train faster \citep{zhou2018convergence}, escape flat ``saddle-point''--like plateaus faster \citep{orvieto2021vanishing}, or handle heavy-tailed stochastic gradients better \citep{zhang2019adaptive, gorbunov2022clipped}.
However, much less is known about why adaptive methods are so critical for solving min-max problems such as GANs. 

Several previous works attribute this performance to the superior convergence speed of adaptive methods.
For instance, \cite{liu2019towards} show that an adaptive variant of Optimistic Gradient Descent \citep{daskalakis2017training} converges faster than SGDA for a class of non-convex, non-concave min-max problems. 
However, contrary to the minimization setting, convergence to a stationary point is not guaranteed, nor is it even a requirement to ensure a satisfactory GAN performance. 
\cite{mescheder2018training} empirically shows that popular architectures such as Wasserstein GANs (WGANs) \citep{arjovsky2017wasserstein} do not always converge, and yet they produce realistic images. 
We support this observation with our own experiments in \autoref{sec:motivation} (see Fig.~\ref{fig:adam_cv}.) 
This finding motivates the central question in this paper: \emph{what factors of Adam contribute to better quality solutions than SGDA when training GANs?} 

In this paper, we investigate why GANs trained with adaptive methods outperform those trained using stochastic gradient descent ascent (SGDA).
Directly analyzing Adam is challenging due to the highly non-linear nature of its gradient oracle and its path-dependent update rule. 
Inspired by the grafting approach in \citep{agarwal2020disentangling}, we disentangle the adaptive magnitude and direction of Adam and show evidence that an algorithm made of the adaptive magnitude of Adam and the direction of SGDA (Ada-nSGDA) recovers the performance of Adam in GANs. 
The adaptive magnitude in Adam is thus key for the performance, and the standard nSGDA direction . 
Our contributions are as follows:

\begin{itemize}[leftmargin=*, itemsep=1pt, topsep=1pt, parsep=1pt]
    \item In \autoref{sec:motivation}, we present the Ada-nSGDA algorithm and the standard normalized SGDA (nSGDA). We further show that for some architectures and datasets, nSGDA can be used to model the dynamics of the performance of Adam in GAN training.
    \item In \autoref{sec:theory}, we prove that for a synthetic dataset consisting of two modes, a model trained with SGDA suffers from \textit{mode collapse} (producing only a single type of output), while a model trained with nSGDA does not. This provides an explanation for why GANs trained with nSGDA outperform those trained with SGDA.
    \item In \autoref{sec:num_exp}, we empirically confirm that Ada-nSDGA recovers the performance of Adam when using different GAN architectures on a wide range of datasets.
\end{itemize}

Our key theoretical insight is that when using SGDA and any step-size configuration, either the generator $G$ or discriminator $D$ updates much faster than its counterpart.
By normalizing the gradients as done in nSGDA, $D$ and $G$ are forced to update at the same speed throughout training.
The consequence is that whenever $D$ learns a mode of the distribution, $G$ learns it right after, which makes both of them learn all the modes of the distribution separately.

\vspace{-2mm}

\subsection{Related work}

\textbf{Adaptive methods in games optimization.} Several works designed adaptive algorithms and analyzed their convergence to show their benefits relative to SGDA e.g. in variational inequality problems, \cite{gasnikov2019adaptive,antonakopoulos2019adaptive,bach2019universal,antonakopoulos2020adaptive,liu2019towards,barazandeh2021solving}.  
\cite{heusel2017gans} show that Adam locally converges to a Nash equilibrium in the regime where the step-size of the discriminator is much larger than the one of the generator. 
Our work differs as we do not focus on the convergence properties of Adam, but rather on the fit of the trained model to the \textit{true} (and not empirical) data distribution.

\textbf{Statistical results in GANs.} Early works studied whether GANs memorize the training data or actually learn the distribution \citep{arora2017generalization,liang2017well,feizi2017understanding,zhang2017discrimination,arora2018gans,bai2018approximability,dumoulin2016adversarially}.  
Some works explained GAN performance through the lens of optimization.  
\cite{lei2020sgd,balaji2021understanding} show that GANs trained with SGDA converge to a global saddle point when the generator is one-layer neural network and the discriminator is a specific quadratic/linear function.  
Our contribution differs as i) we construct a setting where SGDA converges to a locally optimal min-max equilibrium but still suffers from mode collapse, and
ii) we have a more challenging setting since we need at least a degree-3 discriminator to learn the distribution, which is discussed in \autoref{sec:theory}. 



\textbf{Normalized gradient descent.} Introduced by \cite{nesterovngd}, normalized gradient descent has been widely used in minimization problems. 
Normalizing the gradient remedies the issue of iterates being stuck in flat regions such as spurious local minima or saddle points \citep{hazan2015beyond,levy2016power}.  
Normalized gradient descent methods outperforms their non-normalized counterparts in multi-agent coordination \citep{cortes2006finite} and deep learning tasks \citep{cutkosky2020momentum}. 
Our work considers the min-max setting and shows that nSGDA outperforms SGDA as it forces discriminator and generator to update at  same rate.
\vspace{-2mm}


\begin{figure*}[tbp]
    \vspace{-.2cm}
    \hspace{0.5cm}
    \begin{subfigure}{0.45\textwidth}
     \includegraphics[width=1.1\linewidth]{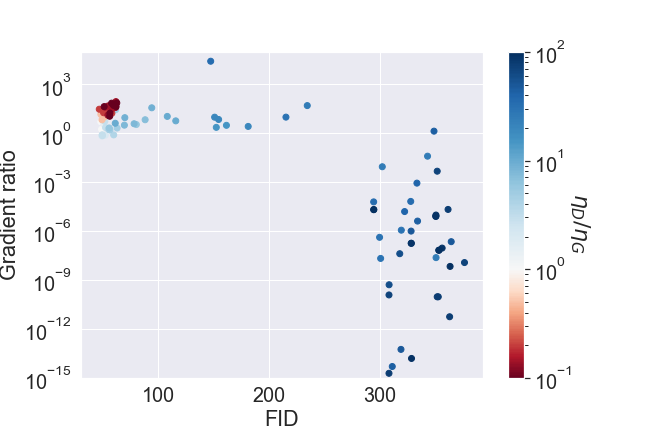}
     \caption{\footnotesize 
        Each circle corresponds to a specific step-size configuration $\eta_D/\eta_G$. The best-performing models  have step-size ratios between $10^{-1}$ and $1$, and do not converge. As  $\eta_D/\eta_G$ increases, the models perform worse but get closer to an equilibrium.}\label{fig:ratiofid}
    \end{subfigure}
    \hfill
    \begin{subfigure}{0.4\textwidth}
    \includegraphics[width=1.1\linewidth]{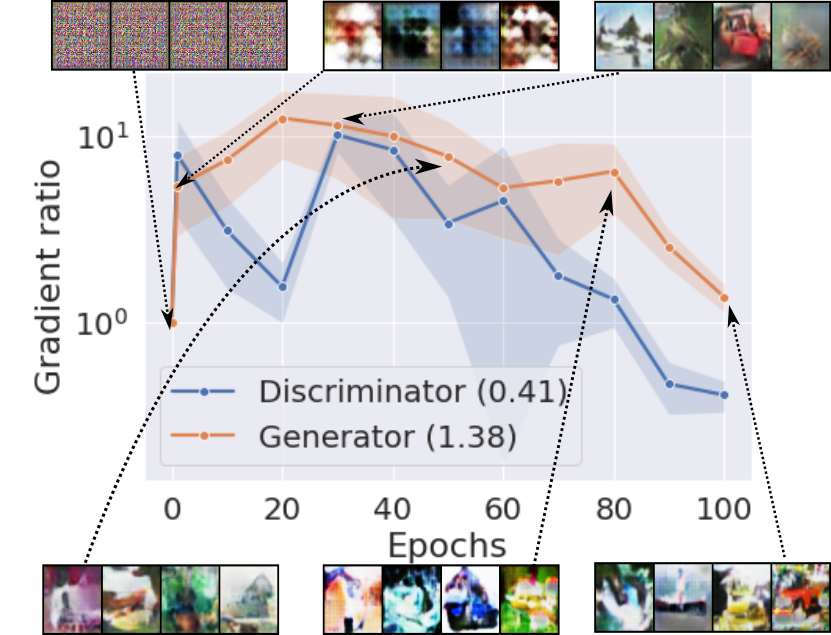} 
    \caption{ \footnotesize  
           shows that during training, the gradient ratio of a well-performing GAN approximately stays constant to 1. We also display the images produced by the model during training.} \label{fig:adam_cv}
    \end{subfigure}
    \hspace{1cm}
    \vspace{-.05cm}
    \caption{\small Gradient ratio against FID score (a) and number of epochs (b) obtained with DCGAN on CIFAR-10. This ratio is equal to $\|\mathrm{grad}_G^{(t)}\|_2/\|\mathrm{grad}_G^{(0)}\|_2+\|\mathrm{grad}_D^{(t)}\|_2/\|\mathrm{grad}_D^{(0)}\|_2$, where $ \mathrm{grad}_G^{(t)}$ (resp. $ \mathrm{grad}_D^{(t)}$)  and $\mathrm{grad}_G^{(0)}$ (resp. $ \mathrm{grad}_D^{(0)}$) are the current and initial gradients of $G$ (resp. $D$). Note that $\|\cdot\|_2$ refers to the sum of all the parameters norm in a network.
     For all the plots, the models are trained for 100 epochs using a batch-size 64. For (b), the results are averaged over 5 seeds. } \label{fig:image_sgd}
     \vspace{-2mm}
\end{figure*}

\subsection{Background}

\textbf{Generative adversarial networks.} Given a training set sampled from some target distribution $\mathcal{D}$, a GAN learns to generate new data from this distribution. 
The architecture is comprised of two networks: a generator that maps points in the latent space $\mathcal{D}_z$ to samples of the desired distribution, and a discriminator which evaluates these samples by comparing them to samples from $\mathcal{D}$. 
More formally, the generator is a mapping $G_{\mathcal{V}}\colon\mathbb{R}^k\rightarrow\mathbb{R}^d$ and the discriminator is a mapping $D_{\mathcal{W}}\colon\mathbb{R}^d\rightarrow \mathbb{R}$, where $\mathcal{V}$ and $\mathcal{W}$ are their corresponding parameter sets.
To find the optimal parameters of these two networks, one must solve a min-max optimization problem of the form
\begin{equation}
    \begin{aligned}\label{eq:gan_formulation_theory}
    \min_{\mathcal{V}} \max_{\mathcal{W}} &\;\mathbb{E}_{X\sim p_{data}}[\log (D_{\mathcal{W}}(X))] +\mathbb{E}_{z\sim p_z}[\log(1-D_{\mathcal{W}}(G_{\mathcal{V}}(z)))] :=f(\mathcal{V},\mathcal{W}),
\end{aligned}
\tag{GAN}
\end{equation}
where $p_{data}$ is the distribution of the training set, $p_z$ the latent distribution, $G_{\mathcal{V}}$ the generator and $D_{\mathcal{W}}$ the discriminator.
Contrary to minimization problems where convergence to a local minimum is \emph{required} for high generalization, we empirically verify  that most of the well-performing GANs do not converge to a stationary point.

\paragraph{Convergence and performance are decorrelated in GANs.} We support this observation through the following experiment.
We train a DCGAN \citep{radford2015unsupervised} using Adam and set up the step-sizes for $G$ and $D$ as $\eta_D,\eta_G$, respectively. 
Note that $D$ is usually trained faster than $G$ i.e. $\eta_D\geq \eta_G.$  \autoref{fig:ratiofid} displays the GAN convergence measured by the ratio of gradient norms and the GAN's performance measured in FID score \citep{heusel2017gans}. 
We observe that when $\eta_D/\eta_G$ is close to $1$, the algorithm does not converge, and yet, the model produces high-quality solutions. 
On the other hand, when $\eta_D/\eta_G\gg 1$, the model converges to an equilibrium; a similar statement has been proved by \cite{jin2020local} and \cite{fiez2020gradient} in the case of SGDA. 
However, the trained GAN produces low-quality solutions at this equilibrium, so simply comparing the convergence speed of adaptive methods and SGDA cannot explain the performance obtained with adaptive methods.




\textbf{SGDA and adaptive methods.} The most simple algorithm to solve the  min-max (\ref{eq:gan_formulation_theory})  is SGDA, which is defined as follows: 
\begin{align}
    \mathcal{W}^{(t + 1)} &= \mathcal{W}^{(t)} +  \eta_D  \mathbf{M}_{\mathcal{W},1}^{(t)} ,\quad \mathcal{V}^{(t + 1)} = \mathcal{V}^{(t)} - \eta_G  \mathbf{M}_{\mathcal{V},1}^{(t)}\,, 
\end{align}
where $\mathbf{M}_{\mathcal{W},1}^{(t)},\mathbf{M}_{\mathcal{V},1}^{(t)}$ are the first-order momentum gradients as defined in \autoref{alg:adaptive}. While this method has been used in the first GANs \citep{radford2015unsupervised}, most modern GANs are trained with adaptive methods such as Adam \citep{kingma2014adam}.

The definition of this algorithm  for game optimizations is given in \autoref{alg:adaptive}.
The hyperparameters $\beta_1, \beta_2\in [0,1)$ control the weighting of the exponential moving average of the first and second-order moments.
In many deep-learning tasks, practitioners have found that setting $\beta_2=0.9$ works for most problem settings.
It has been empirically observed that having no momentum (i.e., $\beta_1 \approx 0$) is optimal for many popular architectures \citep{karras2020analyzing,brock2018large}. Thus, in what follows, we only consider the case where $\beta_1=0$.

Optimizers such as Adam (\autoref{alg:adaptive}) are \textit{adaptive} because they keep updating step-sizes while training the model. 
There are two components that contribute to the update step: the adaptive magnitude $\|\mathbf{A}_{\mathcal{Y}}^{(t)}\|_2$
and the adaptive direction $\mathbf{A}_{\mathcal{Y}}^{(t)}/\|\mathbf{A}_{\mathcal{Y}}^{(t)}\|_2$.
The two components are entangled and it remains unclear how they contribute to the superior performance of adaptive methods relative to SGDA in GANs.

\begin{algorithm}[tbp]
    \begin{algorithmic}
        \small
        \STATE \textbf{Input}: initial points $\mathcal{W}^{(0)},\mathcal{V}^{(0)}$, step-size schedules $\{(\eta_G^{(t)},\eta_D^{(t)})\}$ , hyperparameters $\{\beta_1,\beta_2,\varepsilon\}$. \\
        \vspace{.1cm}
        Initialize $\mathbf{M}_{\mathcal{W},1}^{(0)}$, $\mathbf{M}_{\mathcal{W},2}^{(0)}$, $\mathbf{M}_{\mathcal{V},1}^{(0)}$ and $\mathbf{M}_{\mathcal{V},2}^{(0)}$ to zero.
        \vspace{.1cm}
        \FOR{$t=0\dots T-1$}
        \vspace{.1cm}
        \STATE Receive stochastic gradients $\mathbf{g}_{\mathcal{W}}^{(t)},\mathbf{g}_{\mathcal{V}}^{(t)}$ evaluated at ${\mathcal{W}}^{(t)}$ and ${\mathcal{V}}^{(t)}$.
        \vspace{.1cm}
        \STATE Update  for $\mathcal{Y}\in\{\mathcal{W},\mathcal{V}\}$: 
        $ \mathbf{M}_{\mathcal{Y},1}^{(t+1)}=\beta_1\mathbf{M}_{\mathcal{Y},1}^{(t)}+\mathbf{g}_{\mathcal{Y}}^{(t)} $ and $ \mathbf{M}_{\mathcal{Y},2}^{(t+1)}=\beta_2\mathbf{M}_{\mathcal{Y},2}^{(t)}+{\mathbf{g}_{\mathcal{Y}}^{(t)}}^2. $
        \vspace{.2cm}
        \STATE Compute gradient oracles for $Y\in \{V,W\}$:
        $\mathbf{A}_{\mathcal{Y}}^{(t+1)}=\nicefrac{\mathbf{M}_{\mathcal{Y},1}^{(t+1)}}{\sqrt{\mathbf{M}_{\mathcal{Y},2}^{(t+1)}+\varepsilon}}.$
        \vspace{.2cm}
        \STATE Update:  
            $\mathcal{W}^{(t+1)}=\mathcal{W}^{(t)}+ \eta_D^{(t)}\mathbf{A}_{\mathcal{W}}^{(t+1)} , \qquad
            \mathcal{V}^{(t+1)}=\mathcal{V}^{(t)}- \eta_G^{(t)}\mathbf{A}_{\mathcal{V}}^{(t+1)}.$
        \ENDFOR
    \end{algorithmic}
\caption{Adam \citep{kingma2014adam} for games. 
All operations on vectors are element-wise. 
}\label{alg:adaptive}
\end{algorithm}

\section{nSGDA as a model to analyze Adam in GANs}\label{sec:motivation}
In this section, we show that normalized stochastic gradient descent-ascent (nSGDA) is a suitable proxy to study the learning dynamics of Adam.

To decouple the adaptive magnitude and direction in Adam, we adopt the step-size grafting approach proposed by 
\cite{agarwal2020disentangling}.
At each iteration, we compute stochastic gradients, pass them to two optimizers $\mathcal{A}_1,\mathcal{A}_2$ and make a grafted step that combines the \textit{magnitude} of $\mathcal{A}_1$'s step and \textit{direction} of $\mathcal{A}_2$'s step. 
We focus on the optimizer defined by grafting the Adam magnitude onto the SGDA direction, i.e:
\begin{equation}\label{eq:nsgda_def}
    \begin{aligned}
     \hspace{-.5cm}\mathcal{W}^{(t+1)}&=\mathcal{W}^{(t)}+\eta_D^{(t)}\|\mathbf{A}_{\mathcal{W}}^{(t)}\|_2\frac{ \mathbf{g}_{\mathcal{W}}^{(t)}}{\|  \mathbf{g}_{\mathcal{W}}^{(t)}\|_2+\varepsilon},  
        \quad
        \mathcal{V}^{(t+1)}= \mathcal{V}^{(t)}-\eta_G^{(t)}\|\mathbf{A}_{\mathcal{V}}^{(t)}\|_2\frac{ \mathbf{g}_{\mathcal{V}}^{(t)}}{\|  \mathbf{g}_{\mathcal{V}}^{(t)}\|_2+\varepsilon},
\end{aligned}
\end{equation}
where $\mathbf{A}_{\mathcal{V}}^{(t)},\mathbf{A}_{\mathcal{W}}^{(t)}$ are the Adam gradient oracles as in \autoref{alg:adaptive} and $\bm{g}_{\mathcal{V}}^{(t)},\bm{g}_{\mathcal{W}}^{(t)}$ the stochastic gradients. 
We refer to this algorithm as \textit{Ada-nSGDA} (combining the Adam magnitude and SGDA direction). 
There are two natural implementations for nSDGA. 
In the \emph{layer-wise} version, $\mathcal{Y}^{(t)}$ is a single parameter group (typically a layer in a neural network), and the updates are applied to each group. 
In the \emph{global} version, $\mathcal{Y}^{(t)}$ contains all of the model's weights. 

In Fig.~\ref{fig:curves}, we show that Ada-nSGDA and Adam appear to have similar learning dynamics in terms of the FID score.
Both Adam and Ada-nSGDA significantly outperform SGDA as well as AdaDir, which is the alternate case of (\ref{eq:nsgda_def}) where we instead graft the magnitude of the SGDA update to the direction of the Adam update. 
AdaDir diverged after a single step so we omit it in Fig.~\ref{fig:curves}.
This confirms that the critical component of Adam is the \emph{adaptive magnitude}, and that the standard update direction recovered by SGDA is sufficient to recover a good solution.
However, from a theoretical perspective, directly analyzing Ada-nSGDA is difficult due to the adaptive magnitudes  $\|\mathbf{A}_{\mathcal{V}}^{(t)}\|_2,\|\mathbf{A}_{\mathcal{W}}^{(t)}\|_2$. 
Therefore, in \autoref{sec:theory}, we analyze normalized SGDA (nSGDA) which is Ada-nSGDA (\ref{eq:nsgda_def}) where we omit the adaptive magnitudes.  
Although nSGDA does not consider the adaptive magnitude, it still recovers the performance of Adam for some architectures such as WGAN-GP \citep{arjovsky2017wasserstein} as we show in Fig.~\ref{fig:curves}. 
This may come from the fact that the adaptive magnitudes stay within a constant range and do not  fluctuate across time, as shown in Figs. \ref{fig:gen_graft_lr} and \ref{fig:disc_graft_lr}.

\begin{figure*}[tbp]
    \begin{subfigure}{0.3\textwidth}
        \includegraphics[width=\linewidth]{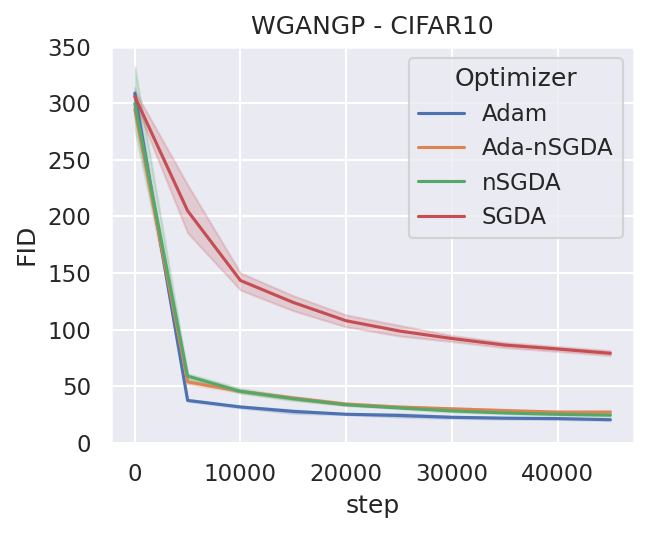}
        \vspace{-5mm}
        \caption{}\label{fig:cifar10_wgan_curve}
    \end{subfigure}
    \begin{subfigure}{0.325\textwidth}
        \includegraphics[width=1.\linewidth]{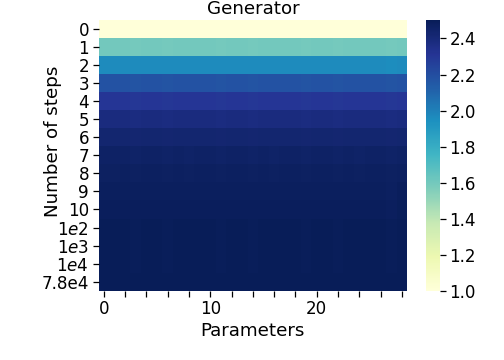}
        \vspace{-5mm}
        \caption{}\label{fig:gen_graft_lr}
    \end{subfigure}
    \begin{subfigure}{0.325\textwidth}
        \includegraphics[width=1.\linewidth]{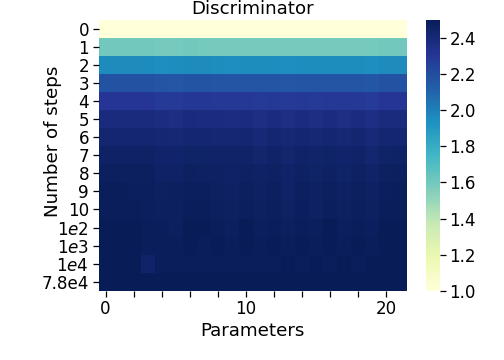}
        \vspace{-5mm}
        \caption{}\label{fig:disc_graft_lr}
    \end{subfigure} 

    \vspace{-0.3cm}
    \caption{
        \small 
        (a) shows the FID training curve for a WGAN-GP ResNet, averaged over 5 seeds.
        We see that Ada-nSGDA and nSGDA have very similar performance to Adam for a WGAN-GP.
        (b, c) displays the fluctuations of Ada-nSGDA adaptive magnitude. 
        We plot the ratio $\|\mathbf{A}_{\mathcal{Y}}^{(t)}\|_2/\|\mathbf{A}_{\mathcal{Y}}^{(0)}\|_2$ for each discriminator's (b) and generator's (c) layers. 
        At early stages, this ratio barely increases and remains constant after 10 steps.
        }\label{fig:curves}
\vspace{-3mm}
\end{figure*}

\section{Why does nSGDA perform better than SGDA in GANs?}\label{sec:theory}


In \autoref{sec:motivation}, we numerically showed that nSGDA recovers the performance Adam.
Similar to other optimization works in minimization, we use nSGDA as a model to understand Adam in GANs.
Our goal is to construct a dataset and model where we can prove that a model trained with nSGDA generates samples from the true training distribution while SGDA fails. 
To this end, we consider a dataset where the underlying distribution consists of two modes, defined as vectors $u_1, u_2$, that are slightly correlated (See \autoref{ass:data}) and consider the standard GANs' training objective. 
We show that a GAN trained with SGDA using any reasonable\footnote{Reasonable simply means that the learning rates are bounded to prevent the training from diverging.} step-size configuration suffers from \emph{mode collapse} (\autoref{thm:sgda}); it only outputs samples from a single mode which is a weighted average of $u_1$ and $u_2$.  
Conversely, nSGDA-trained GANs learn the two modes separately (\autoref{thm:nsgda}).

\paragraph{Notation} We set the GAN 1-sample loss $L_{\mathcal{V},\mathcal{W}}^{(t)}(X,z)=\log(D_{\mathcal{W}}^{(t)}(X)) + \log(1 - D_{\mathcal{W}}^{(t)}(G_{\mathcal{V}}^{(t)}(z))).$ 
We denote $\mathbf{g}_{\mathcal{Y}}^{(t)}=\nabla_{\mathcal{Y}} L_{\mathcal{V},\mathcal{W}}^{(t)}(X, z)$ as the 1-sample stochastic gradient. We use the asymptotic complexity notations when defining the different constants e.g.\ $\mathrm{poly}(d)$ refers to any polynomial in the dimension $d$, $\mathrm{polylog}(d)$ to any polynomial in $\log(d)$, and $o(1)$ to a constant $\ll d$. 
We denote $a \propto b$ for vectors $a$ and $b$ in $\mathbb{R}^d$ if there is a positive scaling factor $c > 0$ such that $\|a - cb\|_2 = o(\| b \|_2)$.

\subsection{Setting}\label{sec:setting}

In this section, we present the setting to sketch our main results in \autoref{thm:sgda} and \autoref{thm:nsgda}. 
We first define the distributions for the training set and latent samples, and specify our GAN model and the algorithms we analyze to solve (\ref{eq:gan_formulation_theory}). 
Note that for many assumptions and theorems below, we present informal statements which are sufficient to capture the main insights. 
The precise statements  can be found in \autoref{app:technical}.

Our synthetic theoretical framework considers a bimodal data distribution with two correlated modes: 
\begin{assumption}[$p_{data}$ structure]\label{ass:data} 
    Let $\gamma=\frac{1}{\mathrm{polylog}(d)}$. We assume that the modes are correlated. 
    This means that $\langle u_1, u_2 \rangle = \gamma >0$ and the generated data point $X$ is either $X=u_1$ or $X=u_2.$
\end{assumption}


 Next, we define the latent distribution $p_z$ that $G_{\mathcal{V}}$ samples from and maps to $p_{data}$.
 Each sample from $p_{z}$ consists of a data-point $z$ that is a binary-valued vector $z\in\{0,1\}^{m_G}$, where $m_G$ is the number of neurons in $G_{\mathcal{V}}$, and has non-zero support, i.e.\ $\| z\|_0 \geq 1$.
 Although the typical choice of a latent distributions in GANs is either Gaussian or uniform, we choose $p_z$ to be a binary distribution because it models the weights' distribution of a hidden layer of a deep generator; \cite{allen2021forward} argue that the distributions of these hidden layers are sparse, non-negative, and non-positively correlated.  
 We now make the following assumptions on the coefficients of $z$:
 
\begin{assumption}[$p_{z}$ structure]\label{ass:latent_dist}
    Let $z\sim p_{z}$.  
    We assume that with probability $1-o(1)$, there is only one non-zero entry in $z$. 
    The probability that the entry $i\in[m_G]$ is non-zero is $\Pr[z_i = 1 ] = \Theta(1/m_G).$  
\end{assumption}
In \autoref{ass:latent_dist}, the output of $G_{\mathcal{V}}$ is only made of one mode with probability $1 - o(1)$.
This avoids summing two of the generator's neurons, which may cause mode collapse.

To learn the target distribution  $p_{data}$, we use a linear generator $G_{\mathcal{V}}$ with $m_G$ neurons and a non-linear neural network with $m_D$ neurons:
\begin{equation}
    G_{\mathcal{V}}(z)=Vz = \sum_{i=1}^{m_G} v_i z_i\,,
    \qquad 
     D_{\mathcal{W}}(X) =\mathrm{sigmoid}\Big( a   \sum_{i =1}^{m_D}\langle w_i, X \rangle^3   +  
     \frac{b}{\sqrt{d}} \Big).
\end{equation}
where $V = [v_1^{\top}, v_2^{\top}, \cdots, v_{m_G}^{\top}]\in\mathbb{R}^{m_G\times d}$, $z\in \{0,1\}^{m_G}$, $W=[w_1^{\top},\dots,w_{m_D}^{\top}]\in\mathbb{R}^{m_D\times d}$, and $a,b\in\mathbb{R}$.
Intuitively, $G_{\mathcal{V}}$ outputs linear combinations of the modes $v_i$. 
We choose a cubic activation as it is the smallest monomial degree for the discriminator's non-linearity that is sufficient for the generator to recover the modes $u_1, u_2$.\footnote{\cite{li2020making} show that when using linear or quadratic activations, the generator can fool the discriminator by only matching the first and second moments of $p_{data}$.}

We now state the SGDA and nSGDA algorithms used to solve the GAN training problem (\ref{eq:gan_formulation_theory}).  
For simplicity, we set the batch-size to 1.   
The resultant update rules for SGDA and nSGDA are:\footnote{In the nSGDA algorithm defined in (\ref{eq:nsgda_def}), the step-sizes were time-dependent.
Here, we assume for simplicity that the step-sizes $\eta_{D},\eta_G>0$ are \emph{constant}.}

\underline{SGDA}: at each step $t >0,$ sample $X\sim p_{data}$ and $z\sim p_z$ and update as
\begin{equation}
    \begin{aligned}\label{eq:SGDA}
        \mathcal{W}^{(t + 1)} &= \mathcal{W}^{(t)} +  \eta_D  \mathbf{g}_{\mathcal{W}}^{(t)} ,\quad \mathcal{V}^{(t + 1)} = \mathcal{V}^{(t)} - \eta_G  \mathbf{g}_{\mathcal{V}}^{(t)}, 
\end{aligned}
\end{equation}

\underline{nSGDA}: at each step $t >0,$ sample $X\sim p_{data}$ and $z\sim p_z$ and update as
\begin{equation}
    \begin{aligned}\label{eq:nSGDA}
        \mathcal{W}^{(t + 1)} &= \mathcal{W}^{(t)} +  \eta_D \tfrac{ \mathbf{g}_{\mathcal{W}}^{(t)}}{\|   \mathbf{g}_{\mathcal{W}}^{(t)} \|_2 },\quad
        \mathcal{V}^{(t + 1)} = \mathcal{V}^{(t)} - \eta_G\tfrac{ \mathbf{g}_{\mathcal{V}}^{(t)} }{\|     \mathbf{g}_{\mathcal{V}}^{(t)}\|_2 }.
\end{aligned}
\end{equation}

Compared to the versions of SGDA and Ada-nSGDA that we introduced in \autoref{sec:motivation}, we have the same algorithms except that we set $\beta_1=0$ and omit $\varepsilon$ in (\ref{eq:SGDA}) and (\ref{eq:nSGDA}).
Lastly, we detail how to set the optimization parameters for SGDA and nSGDA in (\ref{eq:SGDA}) and (\ref{eq:nSGDA}). 
\begin{parametrization}[Informal]\label{parametrization}
    When running SGDA and nSGDA on~(\ref{eq:gan_formulation_theory}), we set:
    \vspace{.1cm}
     
    \hspace{.1cm}-- \textbf{Initialization}:   $b^{(0)}=0,$ and  $a^{(0)},\;w_i^{(0)} (i\in [m_D]),\; v_j^{(0)} (j\in [m_G])$ are initialized with a Gaussian with small variance.  
    
    \vspace{.1cm}
    
    \hspace{.1cm}-- \textbf{Number of iterations}: we run SGDA for $t\leq T_0$ iterations where $T_0$ is the first iteration such that the algorithm converges to an approximate first order local minimum.
    For nSGDA, we run for $T_1=\Tilde{\Theta}(1/\eta_D)$ iterations.
    \vspace{.1cm}   
    
    \hspace{.1cm}-- \textbf{Step-sizes}: For SGDA, $\eta_D,\eta_G\in (0,\frac{1}{\mathrm{poly}(d)})$ can be arbitrary. For nSGDA, $\eta_D\in (0,\frac{1}{\mathrm{poly}(d)}]$, and $\eta_G$ is slightly smaller than $\eta_D.$  
    
    \vspace{.1cm}
    
    \hspace{.1cm}-- \textbf{Over-parametrization}: For SGDA,  $m_D, m_G = \mathrm{polylog}(d)$ are arbitrarily chosen i.e. $m_D$ may be larger than $m_G$ or the opposite. For nSGDA, we set $m_D = \log(d)$ and $m_G = 2\log(d).$
\end{parametrization}
Our theorem holds when running SGDA for any (polynomially) possible number of iterations; after $T_0$ steps, the gradient becomes inverse polynomially small and SGDA essentially stops updating the parameters. 
Additionally, our setting allows any step-size configuration for SGDA i.e. larger, smaller, or equal step-size for $D$ compared to $G$. 
Note that our choice of step-sizes for nSGDA is the one used in practice, i.e. $\eta_D$ slightly larger than $\eta_G.$

\subsection{Main results}

We state our main results on the performance of models trained using SGDA (\ref{eq:SGDA}) and nSGDA (\ref{eq:nSGDA}). 
We show that nSGDA learns the modes of the distribution $p_{data}$ while SGDA does not.

\begin{theorem}[Informal]\label{thm:sgda} 
    Consider a training dataset and a latent distribution as described above and let \autoref{ass:data} and \autoref{ass:latent_dist} hold. 
    Let $T_0$, $\eta_G,\eta_D$ and the initialization be as defined in \autoref{parametrization}. 
    Let $t$ be such that $t\leq T_0.$ 
    Run SGDA on the GAN problem defined in (\ref{eq:gan_formulation_theory}) for $t$ iterations with step-sizes $\eta_G,\eta_D$. Then, with probability at least $1 - o(1)$, the generator outputs  for all $z \in \{0, 1\}^{m_G}$:
    \begin{align}
G_{\mathcal{V}}^{(t)}(z) \propto \begin{cases}
            u_1 + u_2 & \text{if } \eta_D\geq \eta_G\\
            \xi^{(t)}(z) & \text{otherwise}
         \end{cases},
    \end{align}
    where $\xi^{(t)}(z)\in\mathbb{R}^d$ is some vector that is not correlated to any of the modes. 
    Formally, $\forall \ell \in [2]$, $\cos(\xi^{(t)}(z), u_{\ell}) = o(1)$ for all $z \in \{0,1\}^{m_G}$.
\end{theorem}

A formal proof can be found in \autoref{app:sgda1}. 
\autoref{thm:sgda} indicates that when training with SGDA and any step-size configuration, the generator either does not learn the modes at all ($G_{\mathcal{V}}^{(t)}(z) =\xi^{(t)}(z)$) or learns an average of the modes ($G_{\mathcal{V}}^{(t)}(z)  \propto u_1 + u_2$).
The theorem holds \emph{for any} time $t\leq T_0$ which is the iteration where SGDA converges to an approximate first-order locally optimal min-max equilibrium. Conversely, nSGDA succeeds in learning the two modes separately:

\begin{theorem}[Informal]\label{thm:nsgda} 
    Consider a training dataset and a latent distribution as described above and let \autoref{ass:data} and \autoref{ass:latent_dist} hold. 
    Let $T_1$, $\eta_G,\eta_D$ and the initialization as defined in \autoref{parametrization}. 
    Run nSGDA on (\ref{eq:gan_formulation_theory}) for $T_1$ iterations with step-sizes $\eta_G,\eta_D$.
    Then, the generator learns both modes $u_1,u_2$ i.e., for $\ell\in\{1,2\}$,
    \begin{align}
         \Pr{}_{z\sim p_z}[G_{\mathcal{V}}^{(T_1)}(z)\propto u_{\ell}] \quad \text{is non-negligible}.
    \end{align}
    \vspace*{-.3cm}
\end{theorem}
A formal proof can be found in \autoref{app:nsgda_app}.
\autoref{thm:nsgda} indicates that when we train a GAN with nSGDA in the regime where the discriminator updates slightly faster than the generator (as done in practice), the generator successfully learns the distribution containing the direction of both modes.

We implement the setting introduced in \autoref{sec:setting} and validate \autoref{thm:sgda} and \autoref{thm:nsgda} in Fig. \ref{fig:theory}. 
Fig. \ref{fig:relative_grad} displays the relative update speed $\eta\| \mathbf{g}_{\mathcal{Y}}^{(t)} \|_2 / \|\mathcal{Y}^{(t)} \|_2$, where $\mathcal{Y}$ corresponds to the parameters of either $D$ or $G$.
    Fig. \ref{fig:disc_weights} shows the correlation  $\langle w_i^{(t)} , u_{\ell } \rangle / \| w_i^{(t)} \|_2$ between \emph{one} of $D$'s neurons and a mode $u_{\ell}$ and Fig. \ref{fig:gen_weights} the correlation $\langle v_j^{(t)} , u_{\ell } \rangle / \| v_j^{(t)} \|_2$ between $G$'s neurons and  $u_{\ell}$. 
    We discuss the interpretation of these plots to the next section.

\subsection*{Why does SGDA suffer from mode collapse and nSGDA learn the modes?} \label{sec:overview}

\begin{figure*}[tbp]  
    \vspace{-.3cm}
    \begin{subfigure}{0.24\textwidth}
         \includegraphics[width=1\linewidth]{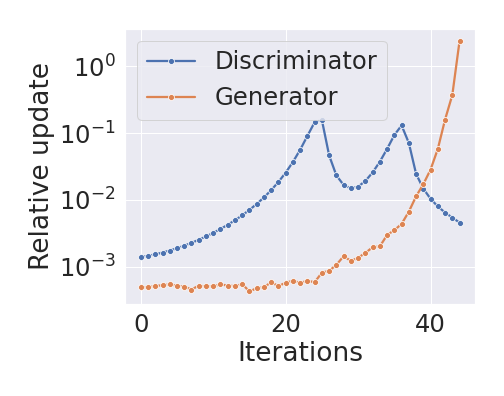}
         \vspace{-7mm}
         \caption{\scriptsize Relative gradients}\label{fig:relative_grad}
    \end{subfigure}
    \begin{subfigure}{0.24\textwidth}
        \hspace{.1cm}
        \includegraphics[width=1\linewidth]{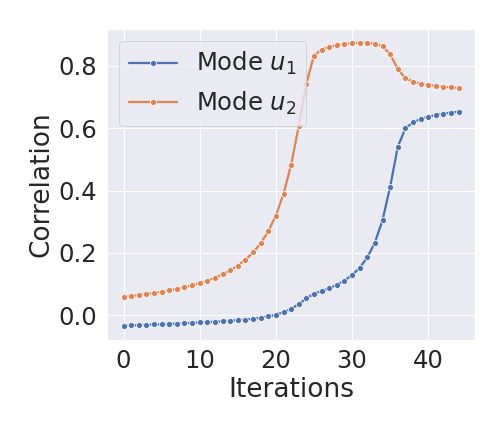}
        \vspace{-7mm}
        \caption{\scriptsize $D$ weight correlation}\label{fig:disc_weights}
    \end{subfigure}
    \begin{subfigure}{0.24\textwidth}
        \hspace{.1cm}
        \includegraphics[width=1\linewidth]{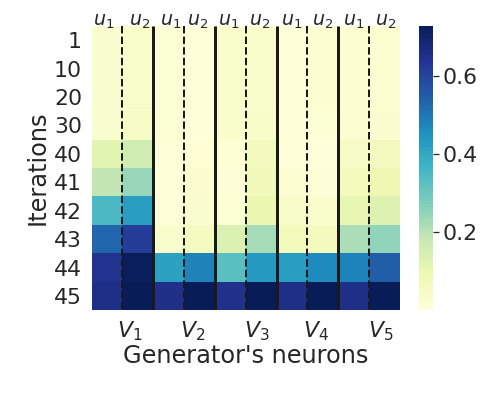}
        \vspace{-7mm}
        \caption{\scriptsize $\mathcal{V}$ learned via SGDA}\label{fig:gen_weights}
    \end{subfigure}
    \begin{subfigure}{0.24\textwidth}
        \hspace{.1cm}
        \includegraphics[width=1\linewidth]{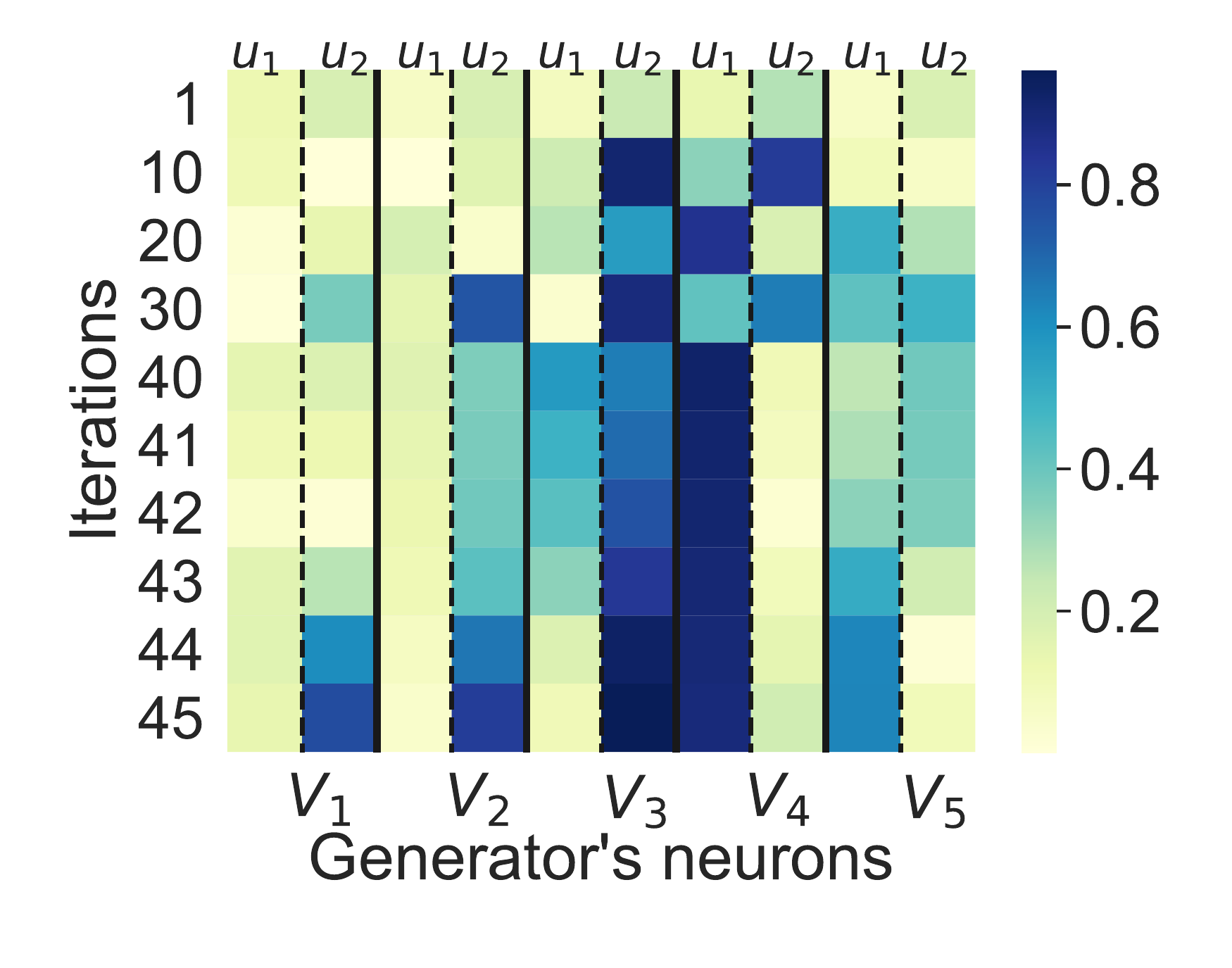}
        \vspace{-7mm}
        \caption{\scriptsize $\mathcal{V}$ learned via nSGDA}\label{fig:fgdngjrg}
    \end{subfigure}
    \vspace{-2mm}
    \caption{\footnotesize 
        (a) shows the relative gradient updates for SGDA. $D$ first updates its weights while $G$ does not move until iteration 20, then $G$ moves its weights. 
        (b) shows the correlation for one neuron of $D$ (with maximal correlation to $u_2$ at initialization) with the modes $u_1,u_2$ during the learning process of SGDA.
        (c, d) shows the correlations of the neurons of $G$ with the modes when trained with SGDA and nSGDA respectively. 
        This shows that for SGDA (c), the model ultimately learns the weighted average $u_1+u_2.$ 
        For nSGDA, we see from (d) that one of the neurons ($V_4$) is highly correlated with $u_1$ and another one ($V_3$) is correlated with $u_2.$  
    }
    \vspace{-.5cm}
    
    \label{fig:theory}
\end{figure*}

We now explain why SGDA suffers from mode collapse, which corresponds to the case where $\eta_D\geq \eta_G$. 
Our explanation relies on the interpretation of Figs.~\ref{fig:relative_grad}, \ref{fig:disc_weights}, and \ref{fig:gen_weights}, and on the updates around initialization that are defined as followed. 
There exists $i\in[m_D]$ such that $D$'s update is with high probability
\begin{align}\label{eq:disc_increaseSGDA}
   \mathbb{E}[w_i^{(t + 1)} |w_i^{(t)}]\approx w_i^{(t)} + \eta_D \sum_{l=1}^2 \mathbb{E}[\langle w_i^{(t)}, u_l \rangle^2 ]u_l  \,.
\end{align}
Thus, the weights of $D$ receive gradients directed by $u_1$ and $u_2$.
On the other hand, the weights of $G$ at early stages receive gradients directed by $w_j^{(t)}$:
\begin{align} \label{eq:update_GG1}
    v_i^{(t + 1)} \approx v_i^{(t)} + \eta_G \sum_{j}  \langle  v_i^{(t)},w_j^{(t)} \rangle^2 w_j^{(t)}.
\end{align}
We observe that the learning process in Figs.~\ref{fig:relative_grad}~\&~\ref{fig:disc_weights} has three distinct phases. 
In the first phase (iterations 1-20), $D$ learns one of the modes ($u_1$ or $u_2$) of $p_{data}$ (Fig.~\ref{fig:disc_weights}) and $G$ barely updates its weights (Fig.~\ref{fig:relative_grad}).
In the second phase (iterations 20-40), $D$ learns the weighted average $u_1+ u_2$ (Fig.~\ref{fig:disc_weights}) while $G$ starts moving its weights (Fig.~\ref{fig:relative_grad}).
In the final phase (iterations 40+), $G$ learns $u_1+u_2$ (Fig. \ref{fig:gen_weights}) from $D$. 
In more detail, the learning process is described as follows:

\textbf{Phase 1} : At initialization, $w_j^{(0)}$ and $v_j^{(0)}$ are small. 
Assume w.l.o.g. that $\langle w_i^{(0)},u_2\rangle > \langle w_i^{(0)},u_1\rangle$. 
Because of the $\langle w_i^{(t)}, u_l \rangle^2$ in front of $u_2$ in (\ref{eq:disc_increaseSGDA}), the parameter
 $w_i^{(t)}$ gradually  grows its correlation with $u_2$ (Fig. \ref{fig:disc_weights})  and $D$'s gradient norm thus increases (Fig.~\ref{fig:relative_grad}). 
 While  $\|w_j^{(t)}\| \ll 1 \, \forall j$, we have that $v_i^{(t)} \approx v_i^{(0)}$ (Fig.~\ref{fig:relative_grad}).
 
\textbf{Phase 2}: $D$ has learned $u_2$. 
Because of the sigmoid in the gradient of $w_i^{(t)}$ (that was negligible during Phase 1) and $\langle u_1,u_2\rangle = \gamma>0$, $w_i^{(t)}$ now mainly receives updates with direction $u_2$.
Since $G$ did not update its weights yet, the min-max problem (\ref{eq:gan_formulation_theory}) is approximately just a minimization problem with respect to $D$'s parameters. 
Since the optimum of such a problem is the weighted average $u_1+u_2$, $w_j^{(t)}$ slowly converges to this optimum.  
Meanwhile, $v_i^{(t)}$ start to receive some significant signal (Fig.~\ref{fig:relative_grad}) but mainly learn the direction $u_1 + u_2$ (Fig.~\ref{fig:gen_weights}), because $w_j^{(t)}$ is aligning with this direction.

\textbf{Phase 3:} The parameters of $G$ only receive gradient directed by $u_1 + u_2$. 
The norm of its relative updates stay large and $D$ only changes its last layer terms (slope $a$ and bias $b$).


In contrast to SGDA, nSGDA ensures that $G$ and $D$ always learn at the same speed with the updates:
\begin{align}\label{eq:disc_increase}
    w_i^{(t + 1)} \approx w_i^{(t)} + \eta_D \frac{\langle w_i^{(t)}, X \rangle^2 X}{\|\langle w_i^{(t)}, X \rangle^2 X\|_2}, \text{ and }
    v_i^{(t + 1)} \approx v_i^{(t)} + \eta_G \frac{\sum_{j} \langle w_j^{(t)}, v_i^{(t)} \rangle^2 w_j^{(t)} }{\|\sum_{j} \langle w_j^{(t)}, v_i^{(t)} \rangle^2 w_j^{(t)} \|_2}
\end{align}
\vspace{-.3cm}

No matter how large $\langle w_i^{(t)}, X \rangle$ is, $G$ still learns at the same speed with $D$. 
There is a tight window (iteration $25$, Fig.~\ref{fig:disc_weights}) where only one neuron of $D$ is aligned with $u_1$. 
This is when $G$ can also learn to generate $u_1$ by ``catching up'' to $D$ at that point, which avoids mode collapse.

\vspace{-2mm}

\section{Numerical performance of nSGDA}\label{sec:num_exp}

In \autoref{sec:motivation}, we present the Ada-nSGDA algorithm (\ref{eq:nsgda_def}) which corresponds to ``grafting'' the Adam magnitude onto the SGDA direction. 
In \autoref{sec:theory}, we construct a dataset and GAN model where we prove that a GAN trained with nSGDA can generate examples from the true training distribution, while a GAN trained with SGDA fails due to mode collapse.
We now provide more experiments comparing nSGDA and Ada-nSGDA with Adam on real GANs and datasets.

\begin{figure*}[tbp]
    \begin{subfigure}{0.245\textwidth}
        \includegraphics[width=1.\linewidth]{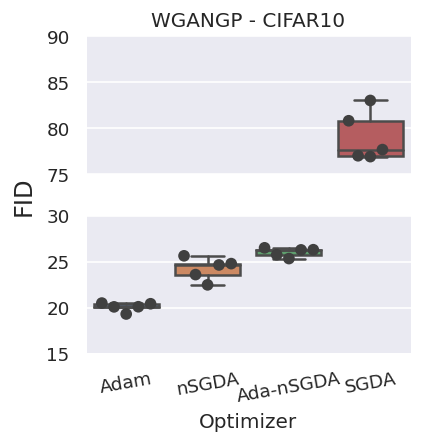} 
        \caption{}
        \label{fig:cifar10_training}
    \end{subfigure}
    \begin{subfigure}{0.245\textwidth}
        \includegraphics[width=1.\linewidth]{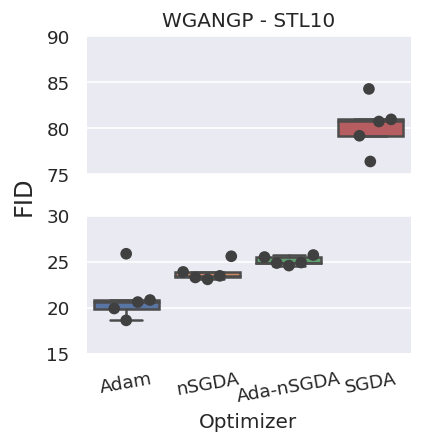}
        \caption{}
        \label{fig:stl10_training}
    \end{subfigure}
    \begin{subfigure}{0.245\textwidth}
        \includegraphics[width=1.\linewidth]{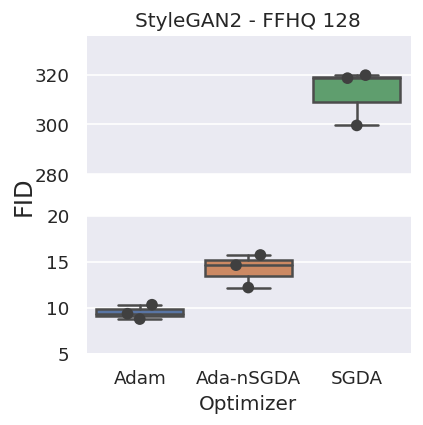}
        \caption{}
        \label{fig:ffhq_fid}
    \end{subfigure} 
    \begin{subfigure}{0.245\textwidth}
        \includegraphics[width=1.\linewidth]{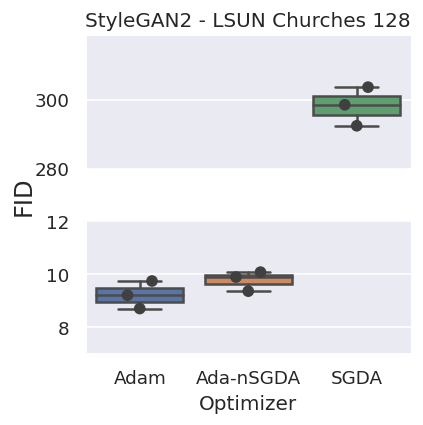}
        \caption{}
        \label{fig:lsun_fid}
    \end{subfigure} 
    \vspace{-0.7cm}
\caption{\small (a, b) are the final FID scores for a ResNet WGAN-GP model trained for 45k steps on CIFAR-10 and STL-10 respectively. (c, d) are the final FID scores for a StyleGAN2 model trained for 2600kimgs on FFHQ and LSUN Churches respectively. WGAN-GP results are obtained over 5 seeds, and StyleGAN2 results over 3 seeds.}\label{fig:wgangp}
\vspace{-3mm}
\end{figure*}

We train a ResNet WGAN with gradient penalty on CIFAR-10 \citep{krizhevsky2009learning} and STL-10 \citep{coates2011analysis} with Adam, Ada-nSDGA, SGDA, as well as nSGDA with a fixed learning rate as done in \autoref{sec:theory}.
We use the default architectures and training parameters specified in \cite{gulrajani2017improved} ($\lambda_{GP}=10$, $n_{dis} = 5$, learning rate decayed linearly to 0 over 100k steps).
We also train a StyleGAN2 model \citep{karras2020analyzing} on FFHQ \citep{karras2019style} and LSUN Churches \citep{yu2016lsun} (both resized to $128 \times 128$ pixels) with Adam, Ada-nSGDA, and SGDA.
We use the recommended StyleGAN2 hyperparameter configuration for this resolution (batch size = 32, $\gamma = 0.1024$, map depth = 2, channel multiplier = 16384).
We use the Fr\'{e}chet Inception distance (FID) \citep{heusel2017gans} to quantitatively assess the performance of the model.  
For each optimizer, we conduct a coarse log-space sweep over step sizes and optimize for FID. 
We train the WGAN-GP models for 2880 thousand images (kimgs) on CIFAR-10 and STL-10 (45k steps with a batch size of 64), and the StyleGAN2 models for 2600 kimgs on FFHQ and LSUN Churches.  
We average our results over 5 seeds for the WGAN-GP ResNets, and over 3 seeds for the StyleGAN2 models due to the computational cost associated with training GANs.

\paragraph{WGAN-GP}
Figures \ref{fig:cifar10_training} and \ref{fig:stl10_training} show the FID curves during training for the WGAN-GP model.  
We test Adam, Ada-nSGDA, AdaDir (that contains the adaptive direction), nSGDA and SGDA. 
The first remark is that Ada-nSGDA recovers the performance of Adam while AdaDir diverges (\texttt{NaN} loss) for these experiments and is hence omitted.
This means that the adaptive magnitude is key to the performance of adaptive methods. 
Additionally, nSGDA not only recovers the performance of Adam but obtains a final FID of $\sim$2-3 points lower than Ada-nSGDA. 
Such performance is possible because the adaptive magnitude
$\|\mathbf{A}_{\mathcal{Y}}^{(t)}\|_2$ stays within a constant range and does not fluctuate much (Figs. \ref{fig:gen_graft_lr}, \ref{fig:disc_graft_lr}) and a constant learning rate as in nSGDA is enough to recover Adam's performance. 
Thus, nSGDA is a valid model for Adam in the case of WGAN-GP.
In contrast, models trained with SGDA consistently perform significantly worse, with final FID scores $4\times$ larger than Adam.

\paragraph{StyleGAN2}
Figures \ref{fig:ffhq_fid} and \ref{fig:lsun_fid} show the FID curves during training for StyleGAN2. 
Similar to WGAN-GP, we observe that Ada-nSGDA recovers the Adam performance while SGDA performs much worse.
However, nSGDA was not able to recover the performance of Adam in this setting.
There are several hypotheses to explain this:
\begin{itemize}[leftmargin=*, itemsep=1pt, topsep=1pt, parsep=1pt]
    \item We observe in Figs. \ref{fig:disc_stylegan}, \ref{fig:gen_stylegan} that the adaptive magnitude the ratio $\|\mathbf{A}_{\mathcal{Y}}^{(t)}\|_2$ stays within a constant range but varies significantly across time. 
    This violates our theoretical setting of assuming a constant adaptive magnitude.
    In contrast, $\|\mathbf{A}_{\mathcal{Y}}^{(t)}\|_2$ stays within a constant range and does not fluctuate across time in the WGAN-GP experiment (Figs. \ref{fig:gen_graft_lr}, \ref{fig:disc_graft_lr}). 
    \item The GAN problem formulation in the case of StyleGAN2 is very different from the one in WGAN-GP. Specifically, StyleGAN2 has a drastically different generator architecture than the ResNet generator used in the WGAN-GP experiments (utilizing weight demodulation), as well as using adaptive data augmentation and additional regularizers (such as a path-length regularization).
\end{itemize}
Because our theoretical setting does not capture these observations, we do not expect nSGDA to work in this setting.
However in spite of these differences, Ada-nSGDA performed similarly to Adam indicating that the nSGDA direction that we theoretically study is still valid in modern real-world GAN architectures. 
We further validate our theory with additional experiments on DCGAN \citep{radford2015unsupervised} which more closely matches our theoretical setting in \autoref{sec:app_exp}, and find that nSGDA recovers the performance of Adam as in WGAN-GP. 

\begin{figure*}[tbp]
    \begin{subfigure}{0.325\textwidth}
        \includegraphics[width=.85\linewidth]{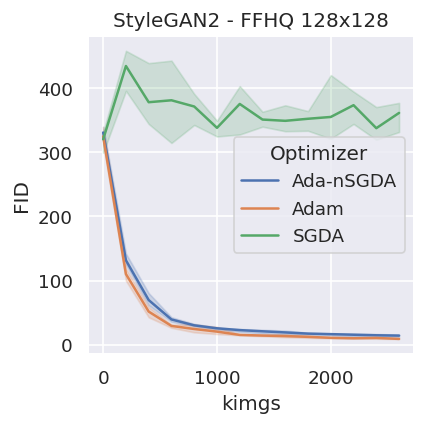} 
        \caption{}
        \label{fig:ffhq_training}
    \end{subfigure}
    \begin{subfigure}{0.325\textwidth}
        \hspace*{-.7cm}
        \includegraphics[width=1.2\linewidth]{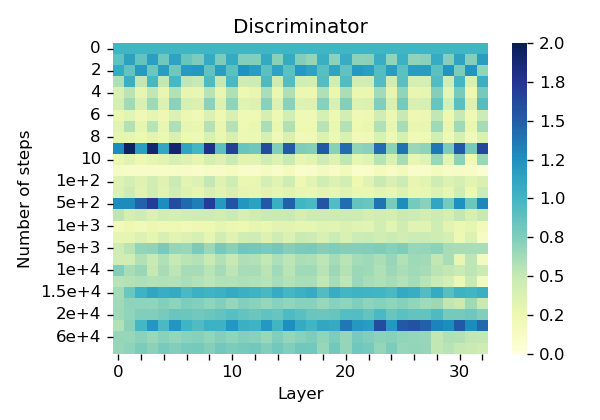}
        \caption{}
        \label{fig:disc_stylegan}
    \end{subfigure} 
    \begin{subfigure}{0.325\textwidth}
        \includegraphics[width=1.2\linewidth]{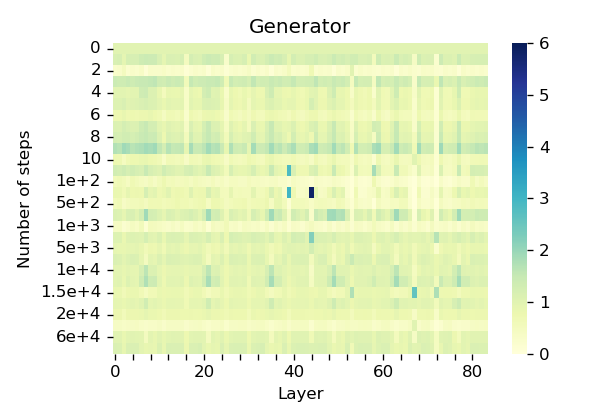}
        
        \caption{}
        \label{fig:gen_stylegan}
    \end{subfigure} 
    \vspace{-0.3cm}
    \caption{\small 
    (a) shows the FID curve for the StyleGAN2 model trained for 2600 kimgs on FFHQ. 
    (b, c) shows the fluctuations of the Ada-nSGDA adaptive magnitude. 
    Like Fig.~\ref{fig:curves}, we plot the ratio $\|\mathbf{A}_{\mathcal{Y}}^{(t)}\|_2/\|\mathbf{A}_{\mathcal{Y}}^{(0)}\|_2$ for each of the discriminator's (b) and generator's (c) layers. This ratio stays in a stable range but fluctuates a lot.
}\label{fig:stylegan}
\vspace{-3mm}
\end{figure*}

 
\section{Conclusion}\label{sec:ccl}

Our work addresses the question of how adaptive methods improve the GAN performance.
We empirically showed that the adaptive magnitude of Adam is responsible for the high performance of adaptive methods in GANs.  
We further show that the adaptive magnitude is almost constant in some settings (for instance in the case of WGAN-GP), and nSGDA is able to recover the learning dynamics of Adam. 
We constructed a setting where we proved that nSGDA --thanks to its balanced updates-- recovers the modes of the true distribution while SGDA fails to do it. 
Our theory provides insights on the effectiveness of adaptive methods in architectures such as DCGAN or WGAN-GP, however the limitation of our work is that we do not fully capture all of the components in modern architectures such as StyleGAN2. 
An exciting direction would be to find an optimization model that recovers the same perform for Adam on these modern GAN architectures.

\vfill


\bibliography{references}

\begin{thebibliography}{47}
\providecommand{\natexlab}[1]{#1}
\providecommand{\url}[1]{\texttt{#1}}
\expandafter\ifx\csname urlstyle\endcsname\relax
  \providecommand{\doi}[1]{doi: #1}\else
  \providecommand{\doi}{doi: \begingroup \urlstyle{rm}\Url}\fi

\bibitem[Agarwal et~al.(2020)Agarwal, Anil, Hazan, Koren, and
  Zhang]{agarwal2020disentangling}
Naman Agarwal, Rohan Anil, Elad Hazan, Tomer Koren, and Cyril Zhang.
\newblock Disentangling adaptive gradient methods from learning rates.
\newblock \emph{arXiv preprint arXiv:2002.11803}, 2020.

\bibitem[Allen{-}Zhu \& Li(2020)Allen{-}Zhu and Li]{az1}
Zeyuan Allen{-}Zhu and Yuanzhi Li.
\newblock Towards understanding ensemble, knowledge distillation and
  self-distillation in deep learning.
\newblock \emph{CoRR}, abs/2012.09816, 2020.
\newblock URL \url{https://arxiv.org/abs/2012.09816}.

\bibitem[Allen-Zhu \& Li(2021)Allen-Zhu and Li]{allen2021forward}
Zeyuan Allen-Zhu and Yuanzhi Li.
\newblock Forward super-resolution: How can gans learn hierarchical generative
  models for real-world distributions.
\newblock \emph{arXiv preprint arXiv:2106.02619}, 2021.

\bibitem[Antonakopoulos et~al.(2019)Antonakopoulos, Belmega, and
  Mertikopoulos]{antonakopoulos2019adaptive}
Kimon Antonakopoulos, Veronica Belmega, and Panayotis Mertikopoulos.
\newblock An adaptive mirror-prox method for variational inequalities with
  singular operators.
\newblock \emph{Advances in Neural Information Processing Systems},
  32:\penalty0 8455--8465, 2019.

\bibitem[Antonakopoulos et~al.(2020)Antonakopoulos, Belmega, and
  Mertikopoulos]{antonakopoulos2020adaptive}
Kimon Antonakopoulos, E~Veronica Belmega, and Panayotis Mertikopoulos.
\newblock Adaptive extra-gradient methods for min-max optimization and games.
\newblock \emph{arXiv preprint arXiv:2010.12100}, 2020.

\bibitem[Arjovsky et~al.(2017)Arjovsky, Chintala, and
  Bottou]{arjovsky2017wasserstein}
Martin Arjovsky, Soumith Chintala, and L{\'e}on Bottou.
\newblock Wasserstein generative adversarial networks.
\newblock In \emph{International conference on machine learning}, pp.\
  214--223. PMLR, 2017.

\bibitem[Arora et~al.(2017)Arora, Ge, Liang, Ma, and
  Zhang]{arora2017generalization}
Sanjeev Arora, Rong Ge, Yingyu Liang, Tengyu Ma, and Yi~Zhang.
\newblock Generalization and equilibrium in generative adversarial nets (gans).
\newblock In \emph{International Conference on Machine Learning}, pp.\
  224--232. PMLR, 2017.

\bibitem[Arora et~al.(2018)Arora, Risteski, and Zhang]{arora2018gans}
Sanjeev Arora, Andrej Risteski, and Yi~Zhang.
\newblock Do gans learn the distribution? some theory and empirics.
\newblock In \emph{International Conference on Learning Representations}, 2018.

\bibitem[Bach \& Levy(2019)Bach and Levy]{bach2019universal}
Francis Bach and Kfir~Y Levy.
\newblock A universal algorithm for variational inequalities adaptive to
  smoothness and noise.
\newblock In \emph{Conference on Learning Theory}, pp.\  164--194. PMLR, 2019.

\bibitem[Bai et~al.(2018)Bai, Ma, and Risteski]{bai2018approximability}
Yu~Bai, Tengyu Ma, and Andrej Risteski.
\newblock Approximability of discriminators implies diversity in gans.
\newblock \emph{arXiv preprint arXiv:1806.10586}, 2018.

\bibitem[Balaji et~al.(2021)Balaji, Sajedi, Kalibhat, Ding, St{\"o}ger,
  Soltanolkotabi, and Feizi]{balaji2021understanding}
Yogesh Balaji, Mohammadmahdi Sajedi, Neha~Mukund Kalibhat, Mucong Ding, Dominik
  St{\"o}ger, Mahdi Soltanolkotabi, and Soheil Feizi.
\newblock Understanding overparameterization in generative adversarial
  networks.
\newblock \emph{arXiv preprint arXiv:2104.05605}, 2021.

\bibitem[Barazandeh et~al.(2021)Barazandeh, Tarzanagh, and
  Michailidis]{barazandeh2021solving}
Babak Barazandeh, Davoud~Ataee Tarzanagh, and George Michailidis.
\newblock Solving a class of non-convex min-max games using adaptive momentum
  methods.
\newblock In \emph{ICASSP 2021-2021 IEEE International Conference on Acoustics,
  Speech and Signal Processing (ICASSP)}, pp.\  3625--3629. IEEE, 2021.

\bibitem[Brock et~al.(2018)Brock, Donahue, and Simonyan]{brock2018large}
Andrew Brock, Jeff Donahue, and Karen Simonyan.
\newblock Large scale gan training for high fidelity natural image synthesis.
\newblock \emph{arXiv preprint arXiv:1809.11096}, 2018.

\bibitem[Choi et~al.(2019)Choi, Shallue, Nado, Lee, Maddison, and
  Dahl]{choi2019empirical}
Dami Choi, Christopher~J Shallue, Zachary Nado, Jaehoon Lee, Chris~J Maddison,
  and George~E Dahl.
\newblock On empirical comparisons of optimizers for deep learning.
\newblock \emph{arXiv preprint arXiv:1910.05446}, 2019.

\bibitem[Coates et~al.(2011)Coates, Ng, and Lee]{coates2011analysis}
Adam Coates, Andrew Ng, and Honglak Lee.
\newblock An analysis of single-layer networks in unsupervised feature
  learning.
\newblock In \emph{Proceedings of the fourteenth international conference on
  artificial intelligence and statistics}, pp.\  215--223. JMLR Workshop and
  Conference Proceedings, 2011.

\bibitem[Cort{\'e}s(2006)]{cortes2006finite}
Jorge Cort{\'e}s.
\newblock Finite-time convergent gradient flows with applications to network
  consensus.
\newblock \emph{Automatica}, 42\penalty0 (11):\penalty0 1993--2000, 2006.

\bibitem[Cutkosky \& Mehta(2020)Cutkosky and Mehta]{cutkosky2020momentum}
Ashok Cutkosky and Harsh Mehta.
\newblock Momentum improves normalized sgd.
\newblock In \emph{International Conference on Machine Learning}, pp.\
  2260--2268. PMLR, 2020.

\bibitem[Daskalakis et~al.(2017)Daskalakis, Ilyas, Syrgkanis, and
  Zeng]{daskalakis2017training}
Constantinos Daskalakis, Andrew Ilyas, Vasilis Syrgkanis, and Haoyang Zeng.
\newblock Training gans with optimism.
\newblock \emph{arXiv preprint arXiv:1711.00141}, 2017.

\bibitem[Dumoulin et~al.(2016)Dumoulin, Belghazi, Poole, Mastropietro, Lamb,
  Arjovsky, and Courville]{dumoulin2016adversarially}
Vincent Dumoulin, Ishmael Belghazi, Ben Poole, Olivier Mastropietro, Alex Lamb,
  Martin Arjovsky, and Aaron Courville.
\newblock Adversarially learned inference.
\newblock \emph{arXiv preprint arXiv:1606.00704}, 2016.

\bibitem[Feizi et~al.(2017)Feizi, Farnia, Ginart, and
  Tse]{feizi2017understanding}
Soheil Feizi, Farzan Farnia, Tony Ginart, and David Tse.
\newblock Understanding gans: the lqg setting.
\newblock \emph{arXiv preprint arXiv:1710.10793}, 2017.

\bibitem[Fiez \& Ratliff(2020)Fiez and Ratliff]{fiez2020gradient}
Tanner Fiez and Lillian Ratliff.
\newblock Gradient descent-ascent provably converges to strict local minmax
  equilibria with a finite timescale separation.
\newblock \emph{arXiv preprint arXiv:2009.14820}, 2020.

\bibitem[Gasnikov et~al.(2019)Gasnikov, Dvurechensky, Stonyakin, and
  Titov]{gasnikov2019adaptive}
AV~Gasnikov, PE~Dvurechensky, FS~Stonyakin, and AA~Titov.
\newblock An adaptive proximal method for variational inequalities.
\newblock \emph{Computational Mathematics and Mathematical Physics},
  59\penalty0 (5):\penalty0 836--841, 2019.

\bibitem[Goodfellow et~al.(2014)Goodfellow, Pouget-Abadie, Mirza, Xu,
  Warde-Farley, Ozair, Courville, and Bengio]{goodfellow2014generative}
Ian Goodfellow, Jean Pouget-Abadie, Mehdi Mirza, Bing Xu, David Warde-Farley,
  Sherjil Ozair, Aaron Courville, and Yoshua Bengio.
\newblock Generative adversarial nets.
\newblock \emph{Advances in neural information processing systems}, 27, 2014.

\bibitem[Gorbunov et~al.(2022)Gorbunov, Danilova, Dobre, Dvurechensky,
  Gasnikov, and Gidel]{gorbunov2022clipped}
Eduard Gorbunov, Marina Danilova, David Dobre, Pavel Dvurechensky, Alexander
  Gasnikov, and Gauthier Gidel.
\newblock Clipped stochastic methods for variational inequalities with
  heavy-tailed noise.
\newblock \emph{arXiv preprint arXiv:2206.01095}, 2022.

\bibitem[Gulrajani et~al.(2017)Gulrajani, Ahmed, Arjovsky, Dumoulin, and
  Courville]{gulrajani2017improved}
Ishaan Gulrajani, Faruk Ahmed, Martin Arjovsky, Vincent Dumoulin, and Aaron
  Courville.
\newblock Improved training of wasserstein gans.
\newblock \emph{arXiv preprint arXiv:1704.00028}, 2017.

\bibitem[Hazan et~al.(2015)Hazan, Levy, and Shalev-Shwartz]{hazan2015beyond}
Elad Hazan, Kfir~Y Levy, and Shai Shalev-Shwartz.
\newblock Beyond convexity: Stochastic quasi-convex optimization.
\newblock \emph{arXiv preprint arXiv:1507.02030}, 2015.

\bibitem[Heusel et~al.(2017)Heusel, Ramsauer, Unterthiner, Nessler, and
  Hochreiter]{heusel2017gans}
Martin Heusel, Hubert Ramsauer, Thomas Unterthiner, Bernhard Nessler, and Sepp
  Hochreiter.
\newblock Gans trained by a two time-scale update rule converge to a local nash
  equilibrium.
\newblock \emph{Advances in neural information processing systems}, 30, 2017.

\bibitem[Jin et~al.(2020)Jin, Netrapalli, and Jordan]{jin2020local}
Chi Jin, Praneeth Netrapalli, and Michael Jordan.
\newblock What is local optimality in nonconvex-nonconcave minimax
  optimization?
\newblock In \emph{International Conference on Machine Learning}, pp.\
  4880--4889. PMLR, 2020.

\bibitem[Karras et~al.(2019)Karras, Laine, and Aila]{karras2019style}
Tero Karras, Samuli Laine, and Timo Aila.
\newblock A style-based generator architecture for generative adversarial
  networks.
\newblock In \emph{Proceedings of the IEEE/CVF conference on computer vision
  and pattern recognition}, pp.\  4401--4410, 2019.

\bibitem[Karras et~al.(2020)Karras, Laine, Aittala, Hellsten, Lehtinen, and
  Aila]{karras2020analyzing}
Tero Karras, Samuli Laine, Miika Aittala, Janne Hellsten, Jaakko Lehtinen, and
  Timo Aila.
\newblock Analyzing and improving the image quality of stylegan.
\newblock In \emph{Proceedings of the IEEE/CVF conference on computer vision
  and pattern recognition}, pp.\  8110--8119, 2020.

\bibitem[Kingma \& Ba(2014)Kingma and Ba]{kingma2014adam}
Diederik~P Kingma and Jimmy Ba.
\newblock Adam: A method for stochastic optimization.
\newblock \emph{arXiv preprint arXiv:1412.6980}, 2014.

\bibitem[Krizhevsky et~al.(2009)Krizhevsky, Hinton,
  et~al.]{krizhevsky2009learning}
Alex Krizhevsky, Geoffrey Hinton, et~al.
\newblock Learning multiple layers of features from tiny images.
\newblock 2009.

\bibitem[Lei et~al.(2020)Lei, Lee, Dimakis, and Daskalakis]{lei2020sgd}
Qi~Lei, Jason Lee, Alex Dimakis, and Constantinos Daskalakis.
\newblock Sgd learns one-layer networks in wgans.
\newblock In \emph{International Conference on Machine Learning}, pp.\
  5799--5808. PMLR, 2020.

\bibitem[Levy(2016)]{levy2016power}
Kfir~Y Levy.
\newblock The power of normalization: Faster evasion of saddle points.
\newblock \emph{arXiv preprint arXiv:1611.04831}, 2016.

\bibitem[Li \& Dou(2020)Li and Dou]{li2020making}
Yuanzhi Li and Zehao Dou.
\newblock Making method of moments great again?--how can gans learn
  distributions.
\newblock \emph{arXiv preprint arXiv:2003.04033}, 2020.

\bibitem[Liang(2017)]{liang2017well}
Tengyuan Liang.
\newblock How well can generative adversarial networks learn densities: A
  nonparametric view.
\newblock \emph{arXiv preprint arXiv:1712.08244}, 2017.

\bibitem[Liu et~al.(2019)Liu, Mroueh, Ross, Zhang, Cui, Das, and
  Yang]{liu2019towards}
Mingrui Liu, Youssef Mroueh, Jerret Ross, Wei Zhang, Xiaodong Cui, Payel Das,
  and Tianbao Yang.
\newblock Towards better understanding of adaptive gradient algorithms in
  generative adversarial nets.
\newblock \emph{arXiv preprint arXiv:1912.11940}, 2019.

\bibitem[Mescheder et~al.(2018)Mescheder, Geiger, and
  Nowozin]{mescheder2018training}
Lars Mescheder, Andreas Geiger, and Sebastian Nowozin.
\newblock Which training methods for gans do actually converge?
\newblock In \emph{International conference on machine learning}, pp.\
  3481--3490. PMLR, 2018.

\bibitem[Nesterov(1984)]{nesterovngd}
Y.E. Nesterov.
\newblock Minimization methods for nonsmooth convex and quasiconvex functions.
\newblock \emph{Econ. Mat. Met.}, 20:\penalty0 519--531, 01 1984.

\bibitem[Orvieto et~al.(2021)Orvieto, Kohler, Pavllo, Hofmann, and
  Lucchi]{orvieto2021vanishing}
Antonio Orvieto, Jonas Kohler, Dario Pavllo, Thomas Hofmann, and Aurelien
  Lucchi.
\newblock Vanishing curvature and the power of adaptive methods in randomly
  initialized deep networks, 2021.

\bibitem[Radford et~al.(2015)Radford, Metz, and
  Chintala]{radford2015unsupervised}
Alec Radford, Luke Metz, and Soumith Chintala.
\newblock Unsupervised representation learning with deep convolutional
  generative adversarial networks.
\newblock \emph{arXiv preprint arXiv:1511.06434}, 2015.

\bibitem[Vaswani et~al.(2017)Vaswani, Shazeer, Parmar, Uszkoreit, Jones, Gomez,
  Kaiser, and Polosukhin]{vaswani2017attention}
Ashish Vaswani, Noam Shazeer, Niki Parmar, Jakob Uszkoreit, Llion Jones,
  Aidan~N Gomez, {\L}ukasz Kaiser, and Illia Polosukhin.
\newblock Attention is all you need.
\newblock In \emph{Advances in neural information processing systems}, pp.\
  5998--6008, 2017.

\bibitem[Yu et~al.(2016)Yu, Seff, Zhang, Song, Funkhouser, and
  Xiao]{yu2016lsun}
Fisher Yu, Ari Seff, Yinda Zhang, Shuran Song, Thomas Funkhouser, and Jianxiong
  Xiao.
\newblock Lsun: Construction of a large-scale image dataset using deep learning
  with humans in the loop, 2016.

\bibitem[Zbontar et~al.(2018)Zbontar, Knoll, Sriram, Murrell, Huang, Muckley,
  Defazio, Stern, Johnson, Bruno, et~al.]{zbontar2018fastmri}
Jure Zbontar, Florian Knoll, Anuroop Sriram, Tullie Murrell, Zhengnan Huang,
  Matthew~J Muckley, Aaron Defazio, Ruben Stern, Patricia Johnson, Mary Bruno,
  et~al.
\newblock fastmri: An open dataset and benchmarks for accelerated mri.
\newblock \emph{arXiv preprint arXiv:1811.08839}, 2018.

\bibitem[Zhang et~al.(2019)Zhang, Karimireddy, Veit, Kim, Reddi, Kumar, and
  Sra]{zhang2019adaptive}
Jingzhao Zhang, Sai~Praneeth Karimireddy, Andreas Veit, Seungyeon Kim,
  Sashank~J Reddi, Sanjiv Kumar, and Suvrit Sra.
\newblock Why are adaptive methods good for attention models?
\newblock \emph{arXiv preprint arXiv:1912.03194}, 2019.

\bibitem[Zhang et~al.(2017)Zhang, Liu, Zhou, Xu, and
  He]{zhang2017discrimination}
Pengchuan Zhang, Qiang Liu, Dengyong Zhou, Tao Xu, and Xiaodong He.
\newblock On the discrimination-generalization tradeoff in gans.
\newblock \emph{arXiv preprint arXiv:1711.02771}, 2017.

\bibitem[Zhou et~al.(2018)Zhou, Chen, Cao, Tang, Yang, and
  Gu]{zhou2018convergence}
Dongruo Zhou, Jinghui Chen, Yuan Cao, Yiqi Tang, Ziyan Yang, and Quanquan Gu.
\newblock On the convergence of adaptive gradient methods for nonconvex
  optimization.
\newblock \emph{arXiv preprint arXiv:1808.05671}, 2018.

\end{thebibliography}
\bibliographystyle{iclr2023_conference}
\pagebreak

\appendix

\section{Additional experiments}\label{sec:app_exp}

\subsection{Experiments with StyleGAN2 and WGAN-GP}

In this section, we put additional curves and images produced by WGAN and StyleGAN2. 

\begin{figure*}[h]
    \begin{subfigure}{0.45\textwidth}\centering
        \includegraphics[width=.8\linewidth]{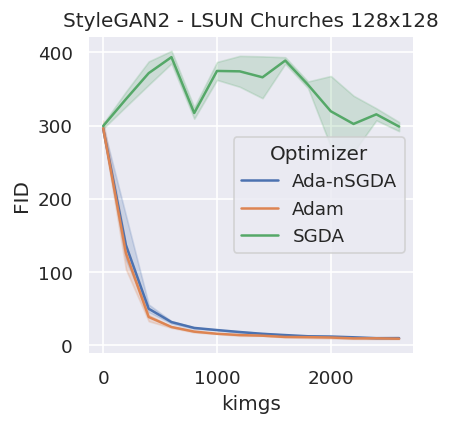} 
        \caption{}
        \label{fig:lsun_curve1}
    \end{subfigure}
    \begin{subfigure}{0.45\textwidth}
    \centering
        \includegraphics[width=.8\linewidth]{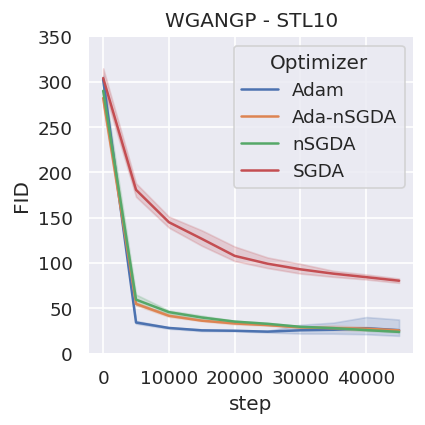}
        \caption{}
        \label{fig:stl10_curv1}
    \end{subfigure}
\caption{\ (a) is the FID curve of StyleGAN2 on LSUN-Churches and (b) the FID curve of WGAN on STL-10. These complement the figures of \autoref{sec:num_exp}.}\label{fig:wgangp_app}
\vspace{-3mm}
\end{figure*}

\begin{figure}
    \centering
    \includegraphics[width=\linewidth]{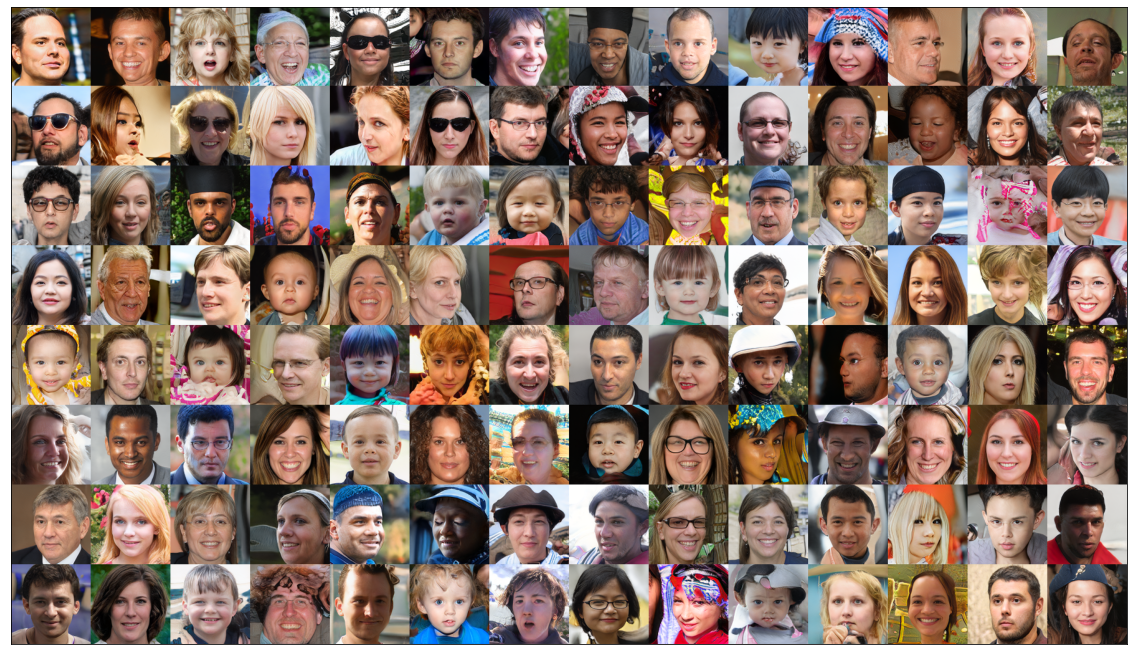}
    \caption{Images generated by a StyleGAN2 model trained with Adam for 2600 kimgs on FFHQ 128. Note that this is not convergence.}
    \label{fig:adam_ffhq}
\end{figure}

\begin{figure}
    \centering
    \includegraphics[width=\linewidth]{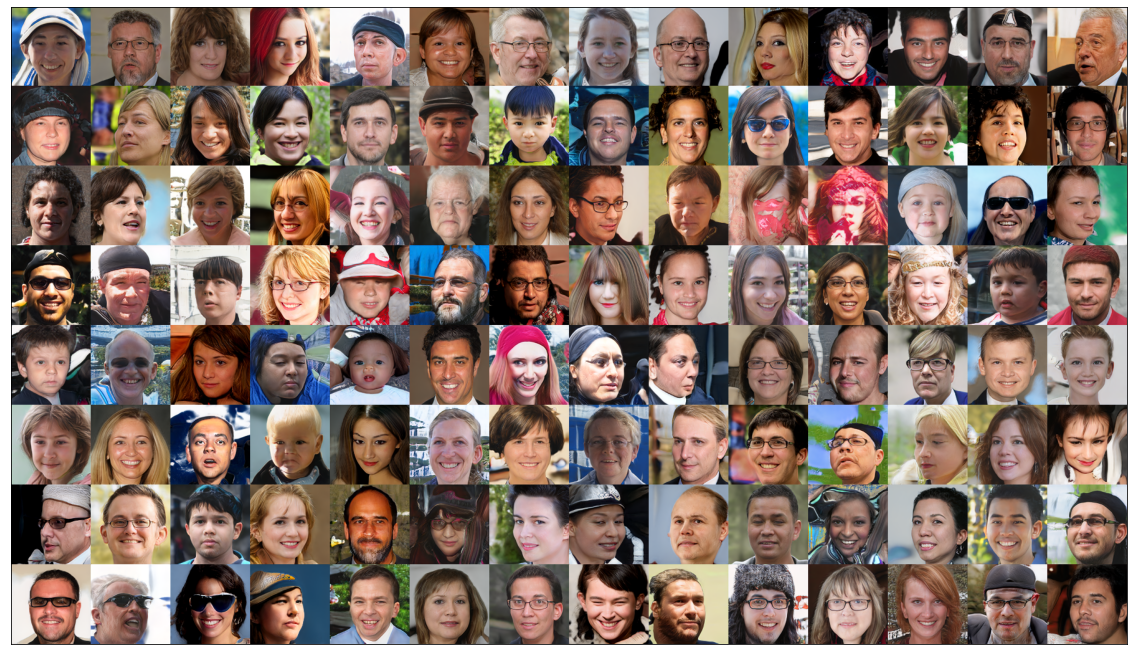}
    \caption{Images generated by a StyleGAN2 model trained with Ada-nSDGDA for 2600 kimgs on FFHQ 128. Note that this is not convergence.}
    \label{fig:adalr_ffhq}
\end{figure}

\begin{figure}
    \centering
    \includegraphics[width=\linewidth]{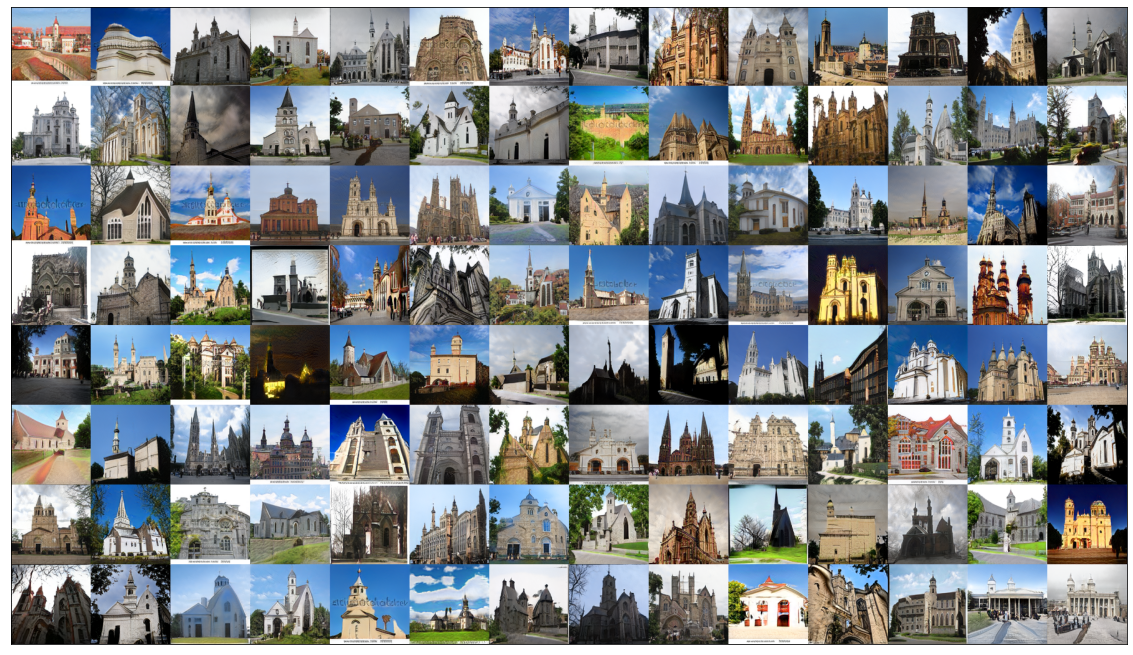}
    \caption{Images generated by a StyleGAN2 model trained with Adam for 2600 kimgs on LSUN Churches 128. Note that this is not convergence.}
    \label{fig:adam_lsun}
\end{figure}

\begin{figure}
    \centering
    \includegraphics[width=\linewidth]{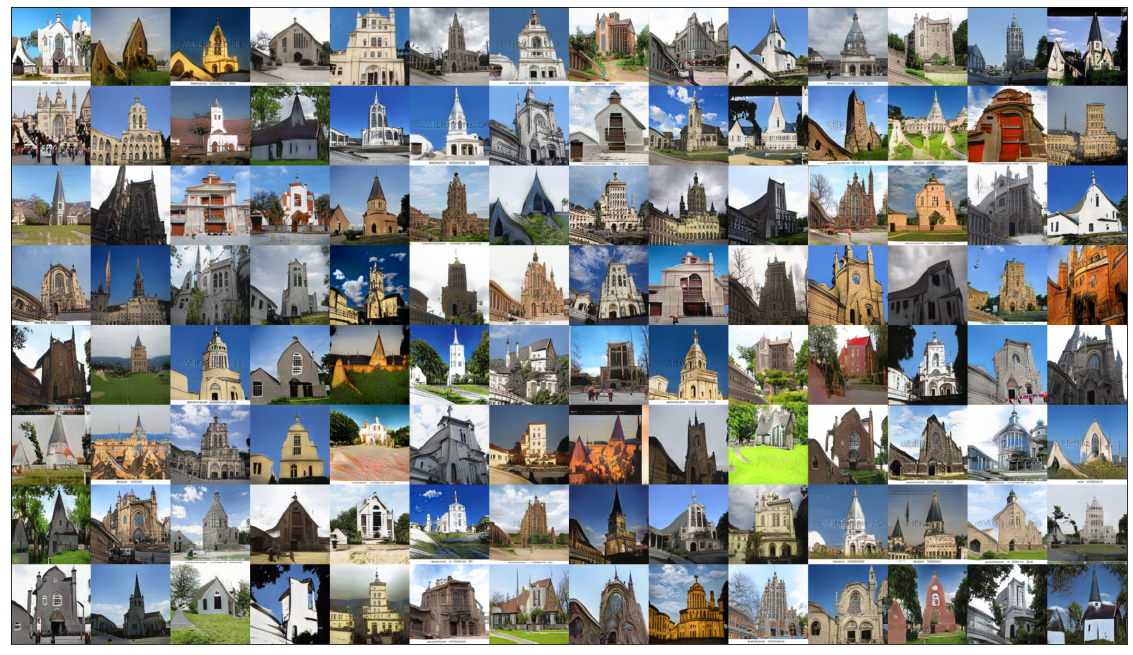}
    \caption{Images generated by a StyleGAN2 model trained with Ada-nSDGDA for 2600 kimgs on LSUN Churches 128. Note that this is not convergence.}
    \label{fig:adalr_lsun}
\end{figure}

\subsection{Experiments with DCGAN}

This section shows that experimental results obtained in \autoref{sec:num_exp} are also valid for other architectures such as DCGAN. Indeed, we observe that  nSGDA methods compete with Adam and nSGDA work when the batch size is small. In this section, lnSGDA refers to the layer-wise nSGDA and gnSGDA to the global nSGDA.

\begin{figure*}[h]
 \begin{subfigure}{0.245\textwidth}
         \includegraphics[width=1.\linewidth]{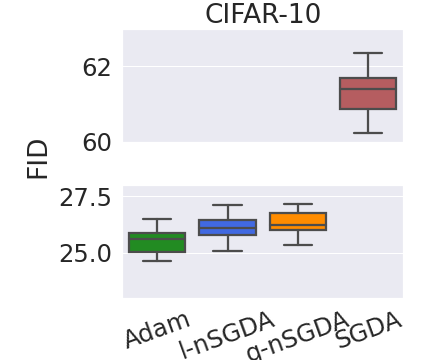} 
        \caption*{(a) CIFAR-10}
        \label{fig:cifar10_final}
    \end{subfigure}
 \begin{subfigure}{0.245\textwidth}
       
         \includegraphics[width=1.\linewidth]{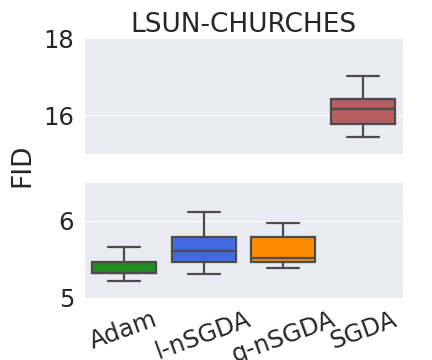}
        \caption*{(b) LSUN Churches}
        \label{fig:lsun_final}
    \end{subfigure}
 \begin{subfigure}{0.245\textwidth}
       
         \includegraphics[width=1.\linewidth]{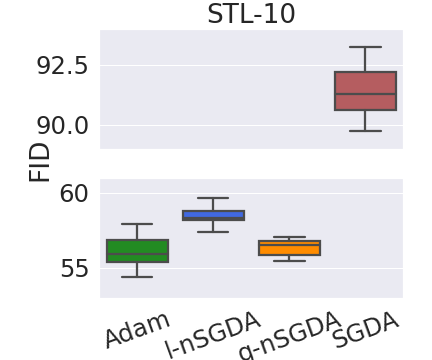}
        \caption*{(c) STL-10}
        \label{fig:stl10_final}
    \end{subfigure} 
 \begin{subfigure}{0.245\textwidth}
       
         \includegraphics[width=1.\linewidth]{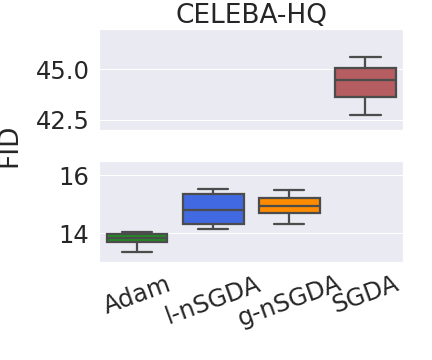}
        \caption*{(d) Celeba-HQ}
        \label{fig:celeba_final}
    \end{subfigure} 
\caption{\small 
FID scores obtained when training a Resnet WGAN-GP using Adam, l-nSGDA, g-nSGDA, and SGD on different datasets. In all these datasets, l-nSGDA, g-nSGDA and Adam perform approximately as well. SGDA performs much worse.}\label{fig:final plots}
\end{figure*}

\begin{figure}[!h]
 \begin{minipage}{0.245\textwidth}
         \includegraphics[width=1\linewidth]{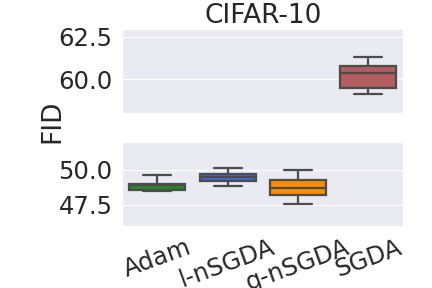} 
        \caption*{(a) CIFAR-10}
    \end{minipage}
 \begin{minipage}{0.245\textwidth}
       
         \includegraphics[width=1\linewidth]{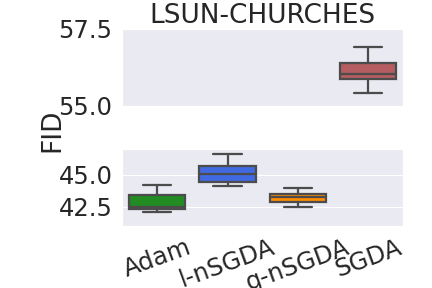}
        \caption*{(b) LSUN Churches}
    \end{minipage}
 \begin{minipage}{0.245\textwidth}
       
         \includegraphics[width=1.\linewidth]{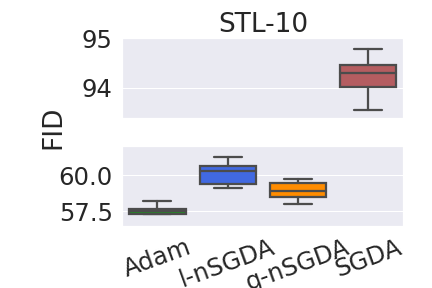}
        \caption*{(c) STL-10}
    \end{minipage} 
 \begin{minipage}{0.245\textwidth}
       
         \includegraphics[width=1.\linewidth]{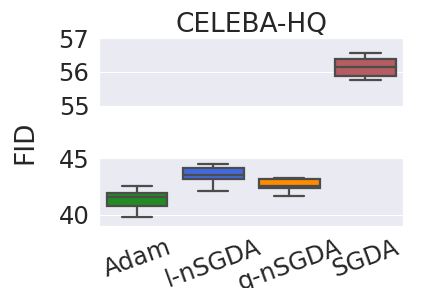}
        \caption*{(d) Celeba-HQ}
    \end{minipage} 
\caption{\small 
FID scores obtained when training a DCGAN using Adam, lnSGDA, gnSGDA and SGD on different datasets. In all these datasets, lnSGDA, gnSGDA and Adam perform approximately as well. As expected, SGDA performs much worse than the other optimizers. The models are trained with batch-size 64 --which is the usual batch-size used for DCGAN.} 
\end{figure}

In this section, we display the images obtained when training the Resnet WGAN-GP model from \autoref{sec:num_exp}.

\begin{figure}[!h]
\begin{subfigure}{0.33\textwidth}
 \includegraphics[width=.95\linewidth]{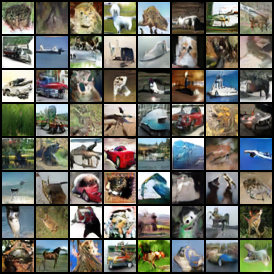} \caption{Adam} 
\end{subfigure}
\begin{subfigure}{0.33\textwidth}
 \includegraphics[width=.95\linewidth]{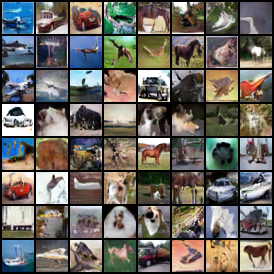}
\caption{lnSGDA} 
\end{subfigure}
\begin{subfigure}{0.33\textwidth}
\hspace{.1cm}

\includegraphics[width=.95\linewidth]{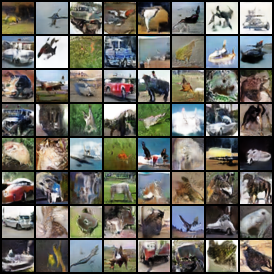}
\caption{gnSGDA} 
\end{subfigure}
\caption{\small  CIFAR-10 images generated by a Resnet WGAN-GP model} 
\end{figure}

\begin{figure}[!h]
\begin{subfigure}{0.33\textwidth}
 \includegraphics[width=.95\linewidth]{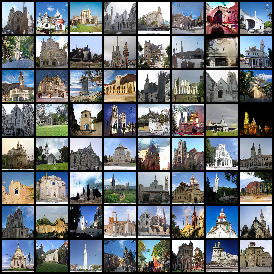} \caption{Adam} 
\end{subfigure}
\begin{subfigure}{0.33\textwidth}
 \includegraphics[width=.95\linewidth]{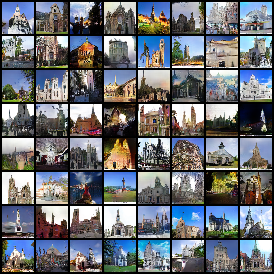}
\caption{lnSGDA} 
\end{subfigure}
\begin{subfigure}{0.33\textwidth}
\hspace{.1cm}

\includegraphics[width=.95\linewidth]{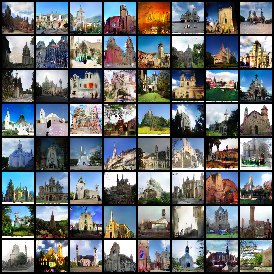}
\caption{gnSGDA} 
\end{subfigure}
\caption{\small  LSUN-Churches images generated by a Resnet WGAN-GP model} 
\end{figure}

\begin{figure}[!h]
\begin{subfigure}{0.33\textwidth}
 \includegraphics[width=.95\linewidth]{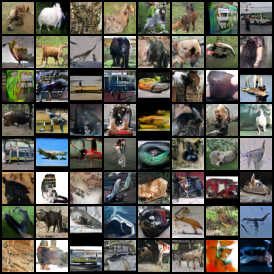} \caption{Adam} 
\end{subfigure}
\begin{subfigure}{0.33\textwidth}
 \includegraphics[width=.95\linewidth]{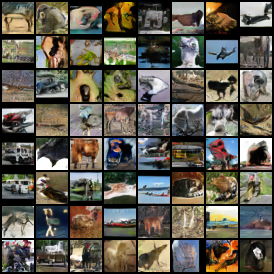}
\caption{lnSGDA} 
\end{subfigure}
\begin{subfigure}{0.33\textwidth}
\hspace{.1cm}

\includegraphics[width=.95\linewidth]{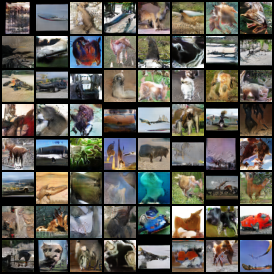}
\caption{gnSGDA} 
\end{subfigure}
\caption{\small  STL-10 images generated by a Resnet WGAN-GP model} 
\end{figure}

\begin{figure}[!h]
\begin{subfigure}{0.33\textwidth}
 \includegraphics[width=.95\linewidth]{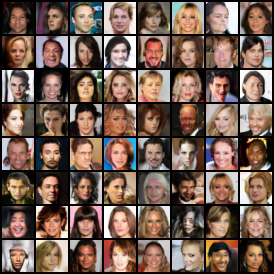} \caption{Adam} 
\end{subfigure}
\begin{subfigure}{0.33\textwidth}
 \includegraphics[width=.95\linewidth]{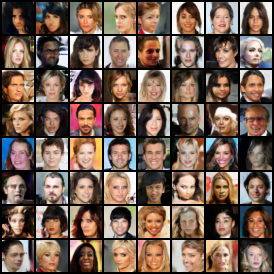}
\caption{lnSGDA} 
\end{subfigure}
\begin{subfigure}{0.33\textwidth}
\hspace{.1cm}

\includegraphics[width=.95\linewidth]{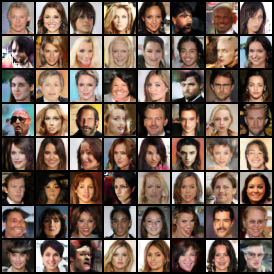}
\caption{gnSGDA} 
\end{subfigure}
\caption{\small  Celeba-HQ images generated by a Resnet WGAN-GP model} 
\end{figure}

\newpage
\section{Technical statements in the theory section}\label{app:technical}

In this section, we provide the technical version of the statements made in \autoref{sec:theory}.

\subsection{Setting}

The distribution $p_{data}$ we consider is more general than \autoref{ass:data} in the main paper.

\begin{assumption}[Data structure]\label{ass:data_technical} Let $\gamma=\frac{1}{\mathrm{polylog}(d)}$. The coefficients $s_1,s_2$ and modes $u_1,u_2$ of the distribution $p_{data}$ respect one of the following conditions: 
\begin{enumerate}[label=\textnormal{\arabic*.}]
    \item Correlated modes:  $\langle u_1, u_2 \rangle = \gamma$ and the generated data point is either $X=u_1$ or $X=u_2.$\label{ass:corrmode}
    
   \item Correlated coefficients: $\mathbb{P}[s_1=s_2=1]=\gamma$ and the modes are orthogonal, $\langle u_1, u_2\rangle =0$. 
\end{enumerate}
\end{assumption}

We now present a more technical version of \autoref{ass:latent_dist}.
\begin{assumption}[$p_{z}$ structure]\label{ass:latent_dist_app}
Let $z\sim p_{z}$.  
We assume that for $i,j\in[m_G]$, 
\begin{equation}
    \begin{aligned}\label{eq:latent_distribution}
    \Pr[z_i = 1 ] &= \Theta\left(\frac{1}{m_G} \right),\; \Pr[z_i= z_j = 1] &= \frac{1}{m_G^2\mathrm{polylog}(d)}
\end{aligned}
\end{equation}
\end{assumption}

We set $\Pr[z_i= z_j = 1] = \frac{1}{m_G^2\mathrm{polylog}(d)}$ to ensure that that the output of the generator is only made of one mode with probability $1 - o(1)$.

In the proof, we actually consider a more complicated version of the discriminator  
\begin{equation}
\hspace*{-.2cm}
    \begin{aligned}\label{eq:disc_app}
    D_{\mathcal{W}}(X)&=\mathrm{sigmoid}\left( a   \sum_{i \in [m_D]} \sigma(\langle w_i, X \rangle)   +  \lambda b \right),\; \text{where } \sigma(z)&=\begin{cases}
                 z^3 & \text{if } |z|\leq \Lambda\\
                 3\Lambda^2z-2\Lambda^3 & \text{if } z>\Lambda\\
                 3\Lambda^2z+2\Lambda^3 & \text{otherwise}
              \end{cases},
\end{aligned}
\end{equation}

where $\Lambda=d^{0.2}$.  $\sigma(\cdot)$ is the truncated degree-3 activation function---it is thus made Lipschitz, which is only needed in the proof to \emph{deal} with the case where the generator is trained much faster than the discriminator. Note that this latter case is uncommon in practice.

We now present the technical version of \autoref{parametrization}.
\begin{parametrization}\label{parametrization_technical}
When running SGDA and nSGDA on~\ref{eq:gan_formulation_theory}, we set the parameters as

\vspace{.1cm}
 
   \hspace{.1cm}-- \textbf{Initialization}:   $b^{(0)}=0,\; a^{(0)} \sim \mathcal{N}\left(0, \frac{1}{m_D\mathrm{polylog}(d)} \right), \; w_i^{(0)} \sim \mathcal{N}\left(0, \frac{1}{d}\mathbf{I}\right),\; v_j^{(0)}\sim\mathcal{N}\left(0, \frac{1}{d^2}\mathbf{I}\right) $ for $i\in [m_D]$, $j\in[m_G].$  
   
   \vspace{.1cm}

\hspace{.1cm}-- \textbf{Number of iterations}: we run SGDA for $t\leq T_0$ iterations where $T_0$ is the first iteration such that $\|\nabla \mathbb{E}[L_{\mathcal{V}^{(T_0)},\mathcal{W}^{(T_0)}}(X,z)]\|_2\leq1/\mathrm{poly}(d)$. . 
For nSGDA, we run for $T_1=\Tilde{\Theta}\left(\frac{1}{\eta_D} \right)$iterations.
   \vspace{.1cm}

\hspace{.1cm}-- \textbf{Step-sizes}: For SGDA, $\eta_D,\eta_G\in (0,\frac{1}{\mathrm{poly}(d)})$. For nSGDA, $\eta_D\in (0,\frac{1}{\mathrm{poly}(d)}]$, $\eta_G=\frac{\eta_D}{\mathrm{polylog}(d)}.$

\vspace{.1cm}

\hspace{.1cm}-- \textbf{Over-parametrization}: For SGDA,  $m_D, m_G = \mathrm{polylog}(d)$ are arbitrarily chosen i.e. $m_D$ may be larger than $m_G$ or the opposite. For nSGDA, we set $m_D = \log(d)$ and $m_G = \log(d)\log\log d$.

\end{parametrization}
 Regarding initialization, the  discriminator's weights are sampled from a standard normal and its bias is set to zero. The weights of the generator are initialized from a normal with variance $1/d^2$ (instead of the $1/d$ in standard normal). Such a choice is explained as follows. In practice, the target $X\sim p_{data}$ is a 1D image, thus has entries in $[0,1]^d$ and norm $O(\sqrt{d}).$ Yet, we sample the initial  generator's weights from $\mathcal{N}(0,\mathbf{I}_d/d)$ in this case.  In our case, since $\|u_i\|_2=1,$ the target $X=s_1u_1+s_2u_2$ has norm $O(1)$. Therefore, we scale down the variance in the normal distribution by a factor of $1/d$ to match the configuration encountered in practice. Therefore, we also set $\lambda = \frac{1}{\sqrt{d}\mathrm{polylog}(d)}$ in (\ref{eq:disc_app}) to ensure that the weights and the bias in the discriminator learn at the same speed.

\textbf{Remark}: In our theory, we consider the global version of nSGDA; $\|\mathbf{g}_{\mathcal{W}}^{(t)}\|_2$ in the update refers to $\|\mathbf{g}_{\mathcal{W}}^{(t)}\|_2 = \|\mathbf{g}_{a}^{(t)}\|_2+\|\mathbf{g}_{b}^{(t)}\|_2+\|\mathbf{g}_{W}^{(t)}\|_2$, where $\mathbf{g}_{a}^{(t)}$ is the stochastic gradient with respect to $a$, $\mathbf{g}_{b}^{(t)}$ with respect to $b$ and $\mathbf{g}_{W}^{(t)}$ with respect to $W.$

\subsection{Main results}

We state now the technical version of \autoref{thm:sgda} and \autoref{thm:nsgda}.

\begin{theorem}[SGDA suffers from mode collapse]\label{thm:sgda_technical} Let $T_0$, $\eta_G,\eta_D$ and the initialization as defined in \autoref{parametrization_technical}. Let $t$ be such that $t\leq T_0.$ Run SGDA for $t$ iterations with step-sizes $\eta_G,\eta_D$. Then, with probability at least $1 - o(1)$, for all $z \in \{0, 1\}^{m_G}$, we have:

\vspace*{-.5cm}

\begin{align*}
    G_{\mathcal{V}}^{(t)}(z) = \alpha^{(t)}(z) (u_1 + u_2) + \xi^{(t)}(z),
\end{align*}
where $\alpha^{(t)}(z) \in \mathbb{R}\; \text{ and }  \;\xi^{(t)}(z ) \in \mathbb{R}^d$ and for all $\ell \in [2]$, $|\langle \xi^{(t)}(z), u_{\ell} \rangle| = o(1)\|\xi^{(t)}(z)\|_2$ for every $z \in \{0,1\}^{m_G}$.




 
 \vspace{-.1cm}

In the specific case where $\eta_D = \frac{\sqrt{d}\eta_G}{\mathrm{polylog}(d)}$,  the model mode collapses i.e. $\|\xi^{(T_0)}(z)\|_2 = o(\alpha^{(T_0)}(z))$. 




\end{theorem}
\autoref{thm:sgda} indicates that with SGDA and any step-size configuration, the generator either does not learn the modes at all -- when $\alpha^{(t)}(z)=0,\;  G_{\mathcal{V}}^{(t)}(z) =\xi^{(t)}(z) $ -- or  learns an average of the modes -- when $\alpha^{(t)}(z)\neq 0,\; G_{\mathcal{V}}^{(t)}(z)  \approx \alpha^{(t)}(z)(u_1 + u_2)$.  

\begin{theorem}[nSGDA recovers  modes separately]\label{thm:nsgda_technical} Let $T_1$, $\eta_G,\eta_D$ and the initialization as defined in \autoref{parametrization_technical}. Run nSGDA for $T_1$ iterations with step-sizes $\eta_G,\eta_D$. Then, the generator learns both modes $u_1,u_2$ i.e., for $\ell\in\{1,2\}$ 

\vspace*{-.4cm}

\begin{equation*}
    \Pr{}_{z\sim p_z}\left( \big\|\tfrac{G_{\mathcal{V}}^{(T_1)}(z)}{\|G_{\mathcal{V}}^{(T_1)}(z)\|_2} - u_{\ell} \big\|_2 = o(1)\right) = \tilde{\Omega}(1).
\end{equation*}

\end{theorem}

\title{Normalized GD vs GD in GAN training}
\author{}
\date{}



\section{Notations}
\label{appendix:theory}

Let us also write $\tau_b = \lambda$ as the scaling factor of the bias. We can easily observe that at every step, all of $w_i^{(t)}$ and $v_i^{(t)}$ lies in the span of $\{ w_j^{(0)}, v_j^{(0)}, u_1, u_2 \}$. Therefore, let us denote 
$$w_i^{(t)} = \sum_{j \in [m_D]} \alpha(w_i,  w_j, t) \frac{w_j^{(0)}}{\| w_j^{(0)} \|_2} +  \sum_{j \in [m_G]} \alpha(w_i,  v_j, t) \frac{v_j^{(0)}}{\| v_j^{(0)} \|_2} +  \sum_{j \in [2]} \alpha(w_i,  u_j, t) \frac{u_j}{\|u_j\|}  $$

and $v_i^{(t)}$ as $\alpha(v_i, *, t)$, where $\alpha(*, *, *) \in \mathbb{R}$.

Let us denote 
$$f(X) = a \left( \sum_{i \in [m_D]} \sigma(\langle w_i, X \rangle) \right) +  \tau_b b $$

as the function in discriminator without going through sigmoid, and define $h(X)=  \sum_{i \in [m_D]} \sigma(\langle w_i, X \rangle)$.

\paragraph{Gradient}

The gradient of $L(X, z)$ is given as:
$$\nabla_a L(X, z) = - \sig(- f(X)) h(X) + \sig(f(G(z))) h(G(z)) $$
$$\nabla_b L(X, z) = - \sig(- f(X))  + \sig(f(G(z)))  $$
$$\nabla_{w_i} L(X, z) = - \sig(- f(X)) a \sigma'(\langle w_i, X\rangle) X + \sig(f(G(z))) a \sigma'(\langle w_i, G(z) \rangle) G(z) $$
$$ \nabla_{v_i} L(X, z) = -  1_{z_i = 1} \sig(f(G(z))) a \sum_{j \in [m_D]} \sigma'(\langle w_i, G(z) \rangle) w_i $$


We use $a^{(t)}, b^{(t)}, w_i^{(t)}, v_i^{(t)}$ to denote the value of those weights at iteration $t$.

\textbf{{We use $a = b \pm c$ for $c \in \mathbb{R}^{*}$ to denote: (1). $a \in [b - c, b+c]$ if $a, b \in \mathbb{R}$, (2). $\| a - b \|_2 \leq c$ if $a , b$ are vectors. }}

For simplicity, we focus on the case when all $\Pr[z_i = 1]$ are equal. The other cases can be proved similarly (by replacing the $1/m_G$ factor in the generators update by the exact value of $\Pr[z_i = 1]$). 

\section{Initialization conditions and three regime of learning}

We first show the following Lemma regarding initialization:  

Let $$A_{i, \ell} = \frac{1}{2} \sigma'(\langle w_i^{(0)}, u_{\ell} \rangle) \sign(\langle w_i^{(0)}, u_{\ell} \rangle) $$
and $$B_{i, j} =  \frac{1}{m_G}\sigma'(\langle w_i^{(0)}, v_{j}^{(0)} \rangle) \sign (\langle w_i^{(0)}, v_{j}^{(0)} \rangle)$$
and $$C_{i, \ell} = \sigma'(\langle v_i^{(0)}, u_{\ell} \rangle) $$

Let $A = \max_{i \in [m_D], \ell \in [2]} A_{i, \ell}$, $B = \max_{i \in [m_D], j \in [m_G]} B_{i, j}  $, $C = \max_{i \in [m_G], \ell \in [2]} C_{i, \ell}$, we have: Using a corollary of Proposition $G.1$ in ~\cite{az1}:
\begin{lemma}\label{lem:init0}
For every $\eta_D, \eta_G > 0$, we have that: with probability at least $1 - o(1)$, we have that: $A  = \frac{\polyloglog(d)}{\sqrt{d}}, B = \frac{\polyloglog(d)}{d m_G}$. Moreover, with probability at least $1 - o(1)$, one and only one of the following holds:
\begin{enumerate}
    \item (Discriminator trains too fast): $\eta_G B < \frac{1}{\polylog(d)} \eta_D A$;
    \item (Balanced discriminator and generator): $\eta_G B > \frac{1}{\polylog(d)} \eta_D A$, $\eta_D A > \eta_G B (1 + \frac{1}{\polyloglog( d)})$;
    \item (Generator trains too fast): $\eta_D A < \eta_G B (1 - \frac{1}{\polyloglog(d)})$.
\end{enumerate}
\end{lemma}

This Lemma implies that in case 2, $\eta_G = \tilde{\Theta}(\sqrt{d}) \eta_B$.

We will show the following induction hypothesis for each case. 
Intuitively, in case one we have the following learning process: (too powerful $D$).
\begin{enumerate}
\item At first $D$ starts to learn, then because of the learning rate of $G$ is too small, so $D$ just saturate the loss to make the gradient to zero. 
\end{enumerate}

In case two we have: (``balanced'' $D$ and $G$ but still not enough).
\begin{enumerate}
\item At first $D$ starts to learn one $u_j$ in each of the neuron.
\item However, the generator still could not catch up immediate after $D$ learns one $u_j$, so $D$ starts to a mixture of $u_1, u_2$ in its neurons since $u_1, u_2$ are  positively correlated. 

\item After that $G$ starts to learn, however since $D$ already stuck at the mixtures of $u_1, u_2$, so $G$ is only able to learn mixtures of $u_1, u_2$ as well. 
\end{enumerate}

In case three we have: (Too powerful $G$)
\begin{enumerate}
\item $G$ starts to learn without $D$ learning any meaningful signal yet, so $G$ aligns its outputs with the (close to random) weights of $D$ and just pushes the discriminator to zero. In this case, $G$ simply learns something random to fool $D$ instead  of learning the signals. 
\end{enumerate}

Moreover, similar to Lemma~\ref{lem:init0}, we also have the following condition regarding the gap between the top one and the second largest one in terms of correlation:
\begin{lemma} \label{lem:init}
Let $$i_D, \ell_D = \argmax_{i \in [m_D], \ell \in [2]} A_{i, 
\ell} $$

Let $$i_G, j_G = \argmax_{i \in [m_D], j \in [m_G]}B_{i, j} $$


Then with probability at least $1 - o(1)$ over the random initialization, the following holds:
$$ \forall i, \ell \not= i_D, \ell_D: A_{i_D, \ell_D} \geq  A_{i, \ell} \left(1 + \frac{1}{\polyloglog(d)} \right)$$
$$ \forall i, j \not= i_G, j_G: B_{i_G, j_G} \geq  B_{i, j} \left(1 + \frac{1}{\polyloglog(d)} \right)$$

and
$$A = \frac{\Theta(\sqrt{\log \log (d)})}{\sqrt{d}}, \quad B, C = \frac{\Theta(\sqrt{\log \log (d)})}{d}$$
\end{lemma}

For simplicity, we also define $i^* = i_D$.

\section{Critical Lemma}

The proof heavily relies on the following Lemma about tensor power method, which is a corollary of Lemma C.19 in~\cite{az1}.
\begin{lemma} \label{lem:TPM}
For every $\delta \in (0, 0.1)$, every $ C > 10$ , for every sequence of $x_t, y_t > 0$ such that $x_0 > (1 + 10\delta) y_0 $, suppose there is a sequence of $S_t \in [0, C]$ such that  for $\eta \in \left(0, \frac{1}{\poly(C/ \delta)} \right)$:
$$x_{t + 1} \geq x_t + \eta S_t x_t^2$$
$$y_{t} \leq y_t + \eta S_t (1 + \delta) y_t^2$$

For every $\tau > 0$, let $T_0$ be the first iteration where $x_{t} > \tau$, then we must have:
$$y_{T_0} \leq \frac{y_0}{\poly(\delta)}$$

Moreover, if all $S_t \geq H$ for some $H > 0$, then $T_0 \leq O\left( \frac{1}{\eta H x_0} \right)$.

\end{lemma}

Similar to the Lemma above, one can easily show the following auxiliary Lemma:
\begin{lemma} \label{lem:powerful233}

Suppose there are sequences $a_t, b_t \in \mathbb{R}^d$ such that $a_0, b_0 > 0$ with $a_0 < 0.82 b_0$. Suppose there exists a sequence of $C_t \in (0, d) $  such that 
$$a_{t + 1} \leq a_t -  \eta_D C_t b_t $$
$$b_{t + 1} \geq b_t -  1.0000001 \eta_D C_t a_t $$

Then we must have that  for every $t \leq T$ where $T$ is the first iteration such that $a_T \leq 0 $, then the following holds:
$$a_t = a_0 - \Theta\left( \eta \sum_{s \leq t - 1} C_s \right) $$
$$\sum_{s \leq t} |a_t C_t \eta_D | \leq 0.49 b_0$$

Moreover, if in addition that $a_0 < \frac{1}{C} b_0$ for any $C > 100$, then we must have:
$$\sum_{s \leq t} |a_t C_t \eta_D | \leq \frac{10}{C} b_0$$

\end{lemma}

In the end, we have the following comparison Lemma, whose proof is obvious:
\begin{lemma}\label{lem:comp}

Suppose $a_t, b_t > 0$ satisfies that $a_0, b_0 \leq 1$, and the update of $a_t, b_t$ is given as: For some values $C> 0$ and $C_t \in [ 0, \poly(d)]$: 
\begin{align}
    a_{t + 1} &= a_t + \eta_D C_t
    \\
    b_{t + 1} &= b_t + \eta_D \left[ \frac{1}{C}, 1 \right] \times C_t
\end{align}

Then let $T$ be the first iteration where $a_T \geq 2C$, we must have:
$$b_T \in \left[ 1, 2C + 1 \right]$$

\end{lemma}

Using this Lemma, we can directly prove the following Lemma:
\begin{lemma} \label{lem:powerful}

For every $\eta_D, \eta_G \in \left(0,  \frac{1}{\poly(d)} \right]$ such that $\eta_G = \eta_D \Gamma$ for $\Gamma = \tilde{\Theta}(\sqrt{d})$, suppose there are vectors $p_t, q_{i, t} \in \mathbb{R}^d$ ($i \in [m_G]$) and a value $a_t \in \mathbb{R}, H > 0$ satisfies that for a sequence of $H_{i, t } \in [H, 1]$ for $i\in [m_G]$, $G_t = \tilde{\Theta}(\sum_{i \in [m_G]} H_{i, t})$, a value $\tau = \tilde{O}(d^{-0.5})$, and a vector $\beta_t \in \text{span}\{u_1, u_2 \}$ with $\| \beta_t \|_2 = O(1)$: For all $i \in [m_G]$ and $t \geq 0$:
$$\|q_{i, 0}\|_2  = \tilde{\Theta}(d^{-0.49}),  \| p_0 \|_2 = \log^{\Theta(1)}(d),  0 < a_0 \leq 0.819 \| p_0 \|_2 \quad  \frac{\langle q_{i, 0}, p_0 \rangle}{\| q_{i, 0}\|_2, \| p_{i, 0} \|_2} \geq 1- o(1) $$

$$p_t = p_t - \eta_D  \sum_{i \in [m_G]} G_{i, t}  a_t \sigma'(\langle p_t, q_{i, t} \rangle) q_{i,t} + \tilde{O}(\eta_D a_t \gamma_t ) \beta_t $$
$$a_t = a_t - \eta_D  \sum_{i \in [m_G]}  G_{i, t} \sigma(\langle p_t, q_{i, t} \rangle)  \pm \tilde{O}(\eta_D \gamma_t ) $$
$$q_{i, t} = \left(q_{i,t} +\eta_G H_{i, t}  a_t \left( \sigma'(\langle p_t, q_{i, t} \rangle)   + \sum_{j \in [m_G]} \gamma_{i, j, t} \sigma'(\langle p_t, q_{j, t} \rangle )  \right)p_t \pm \eta_G |a_t| \tilde{O}\left(\tau \| q_{i, t} \|_2 \right)^2 \right) $$

In addition, we have: $\gamma_{i, j, t} = \tilde{O}(1)$, and
$$\max_{i \in [m_G]} \| q_{i,t }\|_2 \in \left(0, \frac{1}{\polylog(d)} \right]\cup [\polylog(d), + \infty) \implies \forall i, j \in [m_G], H_{i, t} = \tilde{\Theta}(G_t), \gamma_{i,j, t} = \tilde{\Theta}(1)$$


Then we must have that: let $T$ be the first iteration where $a_T \leq 0$, we have: for every $t \leq T$: there is a scaling factor $\ell_t = \Theta(1)$ such that
$$\| p_t - \ell_t p_0\|_2 \leq o(1) \|p_0 \|_2, \quad \| \Pi_{\text{span}\{u_1, u_2, p_0 \}^{\bot}} (p_t - p_0)\|_2 \leq d^{-0.6} \|p_0 \|_2$$

Moreover, for every $i, j \in [m_G]$, $\|q_{i, t} \|_2 = \tilde{\Theta}(\|q_{j, t} \|_2)$ and
$\| q_{i, T} \|_2 \geq \tilde{\Theta}(\sqrt{\Gamma})$, and as long as $\max_{i \in [m_G]} \|q_{i, t} \|_2 \geq \polylog(d)$, we have that $a_t \|q_{i, t} \|_2 \geq \polylog(d)$.

Moreover,
$$ \| \Pi_{\text{span}\{u_1, u_2, p_0 \}^{\bot}} (q_{i,t} - q_{i, 0})\|_2 \leq d^{-0.6} \|q_{i, t} \|_2$$

\end{lemma}

\begin{proof}[proof of Lemma~\ref{lem:powerful}]

For simplicity we consider the case when $H = \tilde{\Omega}(1)$, the other cases follow similarly. 

To proof this result, we maintain the following decomposition of $p_t$ and $q_{i, t}$ as:

$$p_t = \alpha(t) p_0 + \beta(t) + \gamma(t)$$

Where $\beta(t) \in \text{span}\{u_1, u_2 \}$ and $\gamma(t) \bot \text{span}\{u_1, u_2, p_0\} $. Note that $\alpha(0) = 1, \beta(0) = \gamma(0) = 0$.

$$q_{i, t} = \alpha(i, t) p_0 + \beta(i, t) + \gamma(i, t)$$

Where $\beta(i, t) \in \text{span}{u_1, u_2}$ and $\gamma(i, t) \bot \text{span}\{p_0, u_1, u_2\}$.


We maintain the following induction hypothesis (which we will prove at the end): For some $\mu = 0.00001$ and $C_1 = d^{-0.1}  , C_2 = d^{-0.6}$, we have:
\begin{enumerate}

    \item Through out the iterations, $\alpha(t) \geq 0.5$ and $\| \beta(t) \|_2 \leq 0.5(1 -  \mu) C_1, \| \gamma(t)\|_2 \leq 0.5(1 -  \mu)C_2  $.
    \item $\alpha(i, t) \in ( 0, \tilde{O}\sqrt{\Gamma} )$ and $\|\beta(i, t) \|_2 \leq C_1 \alpha(i, t) +  \| \beta(i, 0) \|_2, \quad  \|\gamma(i, t) \|_2  \leq C_2 \alpha(i, t) + \|\gamma(i, 0)\|_2 $

\end{enumerate}

The induction hypothesis implies that through out the iterations, $\langle q_{i, t} , p_t \rangle = \tilde{\Omega}(\|q_{i, t}\|_2)$.

We can now write down the update of $a_t, \alpha's, \beta's$ and $\gamma's $ as:
\begin{align} \label{eq:01}
   a_{t + 1} =  a_t - \eta_D \left( \sum_{i \in [m_G]} G_{i, t}\sigma(\langle p_t, q_{i, t} \rangle)  \pm \tilde{O}(1) \right)
\end{align}

\begin{align} \label{eq:a1}
    \alpha(t + 1) =  \alpha(t) - \eta_D a_t \sum_{i \in [m_G]} G_{i, t}\sigma'(\langle p_t, q_{i, t} \rangle)\alpha(i, t)
\end{align}

\begin{align} \label{eq:b1}
    \beta(t + 1) =  \beta(t) - \eta_D a_t \sum_{i \in [m_G]} G_{i, t}\sigma'(\langle p_t, q_{i, t}\rangle)\beta(i, t)  \pm \tilde{O}(\eta_D  |a_t| )
\end{align}

\begin{align} \label{eq:c1}
    \gamma(t + 1) =  \gamma(t) - \eta_D a_t \sum_{i \in [m_G]} G_{i, t}\sigma'(\langle p_t, q_{i, t} \rangle)\gamma(i, t)  
\end{align}

By the induction hypothesis, we know that 
$$ \sigma'(\langle p_t, q_{j, t} \rangle) \geq  \tilde{\Omega} \left(\| q_{j, t} \|_2^2 \times \frac{\Lambda^2}{\Gamma} \right)$$

Moreover, we have that let $h_{i, t} := \left( \sigma'(\langle p_t, q_{i, t} \rangle)  +  \sum_{j \in [m_G]} \tilde{\Theta}( \sigma'(\langle p_t, q_{j, t} \rangle) )  \right) $
\begin{align} \label{eq:a2}
    \alpha(i, t + 1) = \left( \alpha(i, t) + \eta_G H_{i, t} a_t  h_{i, t} (1 \pm \tilde{O}(\tau^2\Lambda^2/\Gamma)  \alpha(t) \right) 
\end{align}

\begin{align} \label{eq:b2}
    \beta(i, t + 1) = \left( \beta(i, t) + \eta_G H_{i, t} a_t  h_{i, t} \left(\beta(t)  \pm \tilde{O}(\tau^2 \Lambda^2/\Gamma) \right) \right)
\end{align}

\begin{align} \label{eq:c2}
    \gamma(i, t + 1) =   \left( \alpha(i, t) + \eta_G H_{i, t} a_t  h_{i, t} \left(\gamma(t) \pm \tilde{O}(\tau^2\Lambda^2/\Gamma) \right) \right)
\end{align}

From these formula, we can easily that as long as (1).   $\alpha(t) \geq 0.5$ and $\| \beta(t) \|_2 \leq 0.5(1  -  \mu) C_1, \| \gamma(t)\|_2 \leq 0.5(1 -  \mu)C_2 $, (2). $C_1, C_2 = \tilde{\Omega}(\tau^2\Lambda^2/\Gamma)$, we must have that
$\alpha(i, t) > 0$ and $\|\beta(i, t) \|_2 \leq C_1 \alpha(i, t) +  \| \beta(i, 0) \|_2, \quad  \|\gamma(i, t) \|_2  \leq C_2 \alpha(i, t) + \|\gamma(i, 0)\|_2 $. Therefore, it remains to only prove (1) in the induction hypothesis.  Moreover, it is easy to observe that $\alpha(i, t) = \tilde{\Theta}(\alpha(j, t))$ for all $i, j \in [m_G]$ and all $t$. 

Now, we divide the update process into two stages:
\paragraph{Before all $ \|q_{i, t} \|_2  = \Omega(\Lambda)$. Let's call these iterations $[T_{1}]$}

Let us consider $T_{i, 1}$ such that for all $t \in [T_{i, 1}]$ when $q_{i, t} = O(\Lambda)$ and $a_t = \Omega(1)$. In these iterations, by the update rule, we have
$$q_{i, t} =q_{i,t} + \tilde{\Omega} (\eta_G)   \sigma'(\langle p_t, q_{i, t} \rangle) p_t \pm \tilde{O}(\eta_G \tau^2 \| q_{i, t} \|_2^2) $$

By the induction hypothesis, we can simplify the update as:
$$\langle q_{i, t}, p_0 \rangle \geq  \langle q_{i,t} , p_0 \rangle + \tilde{\Omega} \left( \eta_G   \sigma'( \langle q_{i, t}, p_0 \rangle )  \right) $$


Therefore, we know that $T_{i, 1} \leq \tilde{O } \left( \frac{d^{0.49}}{\eta_G } \right)$ and

\begin{align}
\sum_{t \leq T_{i, 1}} \sigma'( \langle q_{i, t}, p_0 \rangle ), \sum_{t \leq T_{i, 1}} \sigma'( \langle q_{i, t}, p_t \rangle ) \leq \tilde{O}\left( \frac{\Lambda}{\eta_G} \right)
\end{align}

Together with the  induction hypothesis, the fact that $\alpha(i, t) = \tilde{\Theta}(\alpha(j, t))$, the fact that $\sigma(\langle p_t, q_{i, t} \rangle) = \tilde{\Theta}(\sigma'(\langle p_t, q_{i, t} \rangle) \| q_{i, t} \|_2)$ and update formula Eq~\eqref{eq:01}~\eqref{eq:a1}~\eqref{eq:b1}~\eqref{eq:c1}, we know that for all $t \leq \max \{ T_{i, 1} \}$:
\begin{align}
    & a_t = a_0 \pm \tilde{O} \left( \frac{\eta_D \Lambda^2}{\eta_G} \right) = a_0 \pm \tilde{O}(d^{-0.01}) \label{eq:c55}
    \\
    &\alpha(t )  = \alpha(0) \pm \tilde{O} \left( \frac{\eta_D \Lambda^2}{\eta_G} \right) = \alpha(0) \pm \tilde{O}(d^{-0.01}) \label{eq:c66}
    \\
    & \| \beta(t) \|_2 \leq  \tilde{O} \left( \frac{\eta_D \Lambda^2}{\eta_G} \right)  C_1 + \tilde{O}\left( \frac{\eta_D \| \beta(i, 0)\|_2 \Lambda}{\eta_G} \right) \leq \tilde{O}(d^{-0.01} ) C_1
    \\
    & \| \gamma(t) \|_2  \leq  \tilde{O} \left( \frac{\eta_D \Lambda^2}{\eta_G} \right)  C_2  + \tilde{O}\left( \frac{\eta_D \| \gamma(i, 0) \|_2 \Lambda}{\eta_G} \right) \leq \tilde{O}(d^{-0.01})  C_2
\end{align}

\paragraph{When all $\| q_{i, t} \|_2 = \Omega(\Lambda)$:} In this case, since $\|p_0 \|_2 = \omega(1)$, we know that $\langle p_t, q_{i, t} \rangle = \omega(\Lambda)$, so $\sigma(\langle p_t, q_{i, t}  \rangle)$ acts on the linear regime, which means that:
$$ \sigma(\langle p_t, q_{i, t}  \rangle) = (1 \pm o(1)) 3 \Lambda^2 \langle p_t, q_{i, t} \rangle, \quad \sigma'(\langle p_t, q_{i, t}  \rangle) = (1  \pm o(1)) 3 \Lambda^2$$

Therefore, we know that:
\begin{align}
    &a_{t + 1} \leq a_t - (1 - o(1)) \eta_D \left( \sum_{i \in [m_G]}  G_{i, t} 3 \Lambda^2 \|q_{i, t} \|_2  \right) \alpha(t) \| p_0 \|_2
    \\
    &\alpha(t + 1) \| p_0 \|_2 \geq \alpha(t) \| p_0 \|_2 - (1 + o(1)) \eta_D\left( \sum_{i \in [m_G]}  G_{i, t} 3 \Lambda^2 \|q_{i, t} \|_2  \right) a_t
\end{align}

Now, using the fact that $a_0 \leq 0.819 \alpha(0)$ and with Eq~\eqref{eq:c55} and Eq~\eqref{eq:c66}, apply Lemma~\ref{lem:powerful233} we have that until $a_t \leq 0$,
\begin{align}\label{eq:c2}
   \sum_{t }  \eta_D \left( \sum_{i \in [m_G]}   G_{i, t} 3 \Lambda^2 \|q_{i, t} \|_2  \right) a_t  \leq 0.49 \| p_0 \|_2 
\end{align}

Plug in to the update rule:

\begin{align} \label{eq:a1}
 \alpha(t + 1) &=  \alpha(t)  \pm  (1 + o(1)) \eta_D a_t \sum_{i \in [m_G]} G_{i, t}3 \Lambda^2 \alpha(i, t)
 \\
 &= \alpha(t)  \pm  (1 + o(1)) \eta_D a_t \sum_{i \in [m_G]} G_{i, t}3 \Lambda^2  \frac{\| q_{ i, t }\|_2}{ \|p_0 \|}
\end{align}

\begin{align} \label{eq:b11}
    \| \beta(t + 1)  \|_2 & \leq  \| \beta(t) \|_2 +   (1 + o(1))\eta_D a_t \left( \sum_{i \in [m_G]} G_{i, t} 3 \Lambda^2 \beta(i, t)  + \tilde{O}(1) \right) 
    \\
   & \leq  \| \beta(t) \|_2 + \eta_D   (1 + o(1)) a_t \left( \sum_{i \in [m_G]} G_{i, t} 3 \Lambda^2 \frac{\| q_{i, t} \|_2 C_1}{ \| p_0 \|_2}  + \tilde{O}(1) \right) 
   \\
    &\leq  \| \beta(t) \|_2 + \eta_D   (1 + o(1)) a_t \left( \sum_{i \in [m_G]} G_{i, t} 3 \Lambda^2 \frac{\| q_{i, t} \|_2 C_1}{ \| p_0 \|_2}  \right) 
\end{align}

\begin{align} \label{eq:c11}
     \| \gamma(t + 1) \|_2 &\leq  \| \gamma(t) \|_2 + \eta_D   (1 + o(1)) a_t \sum_{i \in [m_G]} G_{i, t} 3 \Lambda^2 \gamma(i, t)
     \\
     & \leq   \| \gamma(t) \|_2 + \eta_D   (1 + o(1)) a_t \sum_{i \in [m_G]} G_{i, t} 3 \Lambda^2 \frac{\| q_{i, t} \|_2 C_2}{ \| p_0 \|_2} 
\end{align}

We directly complete the proof of the induction hypothesis using Eq~\eqref{eq:c2}.

Now it remains to prove that $\|q_{i, T} \|_2 = \Omega(\sqrt{\Gamma})$. Compare the update rule of $q_{i, t}$ and $a_t$ we have:
\begin{align} 
a_{t + 1} = a_t -\tilde{\Theta} (\eta_D) G_t \left( \sum_{i \in [m_G]}   \Lambda^2 \|q_{i, t} \|_2  \right)
\end{align}

and 
\begin{align}
   \sum_{i \in [m_G] }\| q_{i, t + 1}\|_2 =  \| q_{i, t}\|_2 + \tilde{\Theta} (\eta_G )G_t   \Lambda^2 a_t 
\end{align}

We can directly conclude that $\|q_{i, t} \|_2 \leq \tilde{O}\left(\sqrt{\Gamma} \right)$ and $\|q_{i, T} \|_2 = \tilde{\Omega}(\sqrt{\Gamma})$ . 

\end{proof}

\begin{lemma}\label{lem:powerful_777}
For every $\eta_D, \eta_G \in \left( 0, \frac{1}{\poly(d)} \right]$ such that $\eta_G = \eta_D \Gamma$ for $\Gamma \geq \tilde{\Omega}(\sqrt{d})$, suppose for sufficiently large $C = \poly(\log(d) m_D)$ there are vectors $\{q_{i, t} \}_{i \in [m_G]}$, $\{ p_i \}_{i \in [m_D]}$  in $\mathbb{R}^d$ such that $\|p_i \|_2 = 1, \langle p_i, p_{i'} \rangle \leq \tilde{O}(1/\sqrt{d})$ for $i, i'$, values $H_{i, j, t}, G_{i, t} \in  \left[\frac{1}{C^{2}}, C^{2 } \right]$ and a value $a_0 \geq  0$ satisfies that:
\begin{align}
    a_0 &= \frac{1}{\polylog(d)},  \| q_{j, 0} \|_2 = \tilde{\Theta}(\Lambda); \quad q_{j, 0} = \sum_{i \in [m_D]} a_i p_i + \xi_j, a_i \geq 0, \| \xi_i\|_2 \leq \frac{1}{C} \| q_{j, 0}\|_2
    \\
    a_{t + 1} &= a_t - \eta_D  \left(  H_{i, j, t} \sum_{i \in [m_D], j \in [m_G]} \sigma(\langle p_i, q_{j, t} \rangle) \right)
    \\
    q_{i,t + 1} &= q_{i, t} + \eta_G a_t G_{i, t}\sum_{j \in [m_D]}\left(  \sigma'(\langle p_j, q_{i, t} \rangle) \left(p_j \pm \frac{1}{C} \right) \right)
\end{align}

Then we must have: within $T = \tilde{O } \left(  \frac{\sqrt{\Gamma}}{\eta_G} \right)$ many iterations, we must have that $a_t \leq 0$ and $\max_{j \in [m_G]}\|q_{j, T} \|_2 = \tilde{\Theta}(\sqrt{\Gamma})$. Moreover, for every $t \leq T$, we have: for every $j \in [m_G]$,
$$\sum_{i \in [m_D]} \sigma(\langle p_i, q_{j, t} \rangle ) = \Omega \left( \max_{i \in [m_D]} \sigma'(\langle p_i, q_{j, t} \rangle) \| q_{j, t} \|_2 \right)$$

and 
$$ \max_{i \in [m_D]} \langle p_i, q_{j, t} \rangle \geq \left(1 - \frac{1}{C^{0.2}} \right) \| q_{j, t} \|_2$$

\end{lemma}

\begin{proof}[Proof of Lemma~\ref{lem:powerful_777}]

Let us denote $r_{i, t} = \max_{j \in [m_D]}\{ \langle p_j, q_{i, t} \rangle \}$. 

By the update rule, we can easily conclude that:
$$r_{i, t + 1} = r_{i, t} + \eta_G G_{i, t} \left( 1 - \frac{1}{C^{0.5}} \right) \sigma'(r_{i, t}) $$

On the other hand, let us write $q_{i, t} = \sum_{j \in [m_D]} \alpha_{i, j, t} q_j + \xi_{i, t}$, where $ \alpha_{i, j, t} \geq 0$. We know that:
\begin{align}
   \| \xi_{i, t + 1} \|_2 \leq \|  \xi_{i, t} \|_2 + \eta_G G_{i, t} \frac{m_D}{C} \sigma'(r_{i, t})
\end{align}

By the comparison Lemma~\ref{lem:comp} we can easily conclude that for every $t$, $$\| \xi_{i, t} \|_2 \leq \frac{1}{C^{0.5}} r_{i, t}$$

This implies that: there exists values $u_t \in [1/C^{2}, C^{2}]$ such that
\begin{align}
    a_{t + 1}  = a_t - \eta_D u_t \sum_{i \in [m_G]}\sigma(r_{i, t})
\end{align}

Comparing this with the update rule of $r_{i, t}$, we know that for every $t $ with $a_t \geq 0$, we must have: $$r_{i, t} = \tilde{O}\left( \sqrt{\Gamma} \right), \quad r_{i, T} = \tilde{\Theta}(\sqrt{\Gamma})$$

\end{proof}

\begin{lemma}[Auxiliary Lemma]~\label{lem:aux_1}

For every $g > 0$ we must have:
$\sig(-g x - b) x$ is a decreasing function of $x$ as long as $gx > 1$ and $gx + b > 0$.

\end{lemma}

\begin{lemma}\label{lem:powerful_666}

For $a_t, b_t, c_t, d_t \in \mathbb{R}^d$ be defined as:  $a_0, c_0, d_0 = \frac{1}{\polylog(d)}$, $|b_t| \leq O(\log d)$ and $ |b_t | \leq \min \{ a_t c_t^3, a_t d_t^3 \}$.
\begin{align}
    a_{t + 1} &= a_t + \eta_D \frac{1}{2} \left( \left( 1 \pm \frac{1}{\polylog(d)} \right) \sig(- a_t c_t^3 - b_t) c_t^3 + \left( 1 \pm \frac{1}{\polylog(d)} \right) \sig(- a_t d_t^3 - b_t)d_t^3  \right)
    \\
     c_{t + 1} &= c_t + \eta_D \frac{3}{2} \left( \left( 1 \pm \frac{1}{\polylog(d)} \right) \sig(- a_t c_t^3 - b_t) c_t^2 a_t \right)
     \\
      d_{t + 1} &= d_t + \eta_D \frac{3}{2} \left( \left( 1 \pm \frac{1}{\polylog(d)} \right) \sig(- a_t d_t^3 - b_t)  d_t^2 a_t  \right)
\end{align}

Then we have: for every $t \in \left( \frac{\polylog(d)}{\eta_D}, \frac{\poly(d)}{\eta_D} \right]$, we must have:
\begin{align}  
    a_t &= \sqrt{\frac{2}{3}}\left(1 \pm \frac{1}{\polylog(d)} \right) c_t
    \\
    c_t &= \left(1 \pm \frac{1}{\polylog(d)} \right) d_t
\end{align}

\end{lemma}

\begin{proof}[Proof of Lemma~\ref{lem:powerful_666}]

By the update formula, we can easily conclude that for $t \leq \frac{\poly(d)}{\eta_D}$, we have that $a_t, c_t, d_t \in \left[ \frac{1}{\polylog(d)}, \polylog(d) \right]$. This implies that for every $t \in \left[  \frac{\polylog(d)}{\eta_D},  \frac{\poly(d)}{\eta_D} \right]$, we have that 
$$a_tc_t^3, a_t d_t^3 \in [1, O(\log d) ]$$
Apply Lemma~\ref{lem:aux_1} we have that: As long as $a_t > 3 (c_t + d_t)$, we must have that 
$$\sig(- a_t c_t^3 - b_t) c_t^3 +  \sig(- a_t d_t^3 - b_t)d_t^3   <  \sig(- a_t c_t^3 - b_t) c_t^2 a_t + \sig(- a_t d_t^3 - b_t) d_t^2 a_t   $$

This implies that
$$\frac{a_{t + 1}}{3} - \frac{a_t}{3} < c_{t + 1} + d_{t + 1} - c_t - d_t$$

Note that initially, $a_0 ,c_0, d_0 = \frac{1}{\polylog(d)} $. This implies that when $t \geq\frac{\polylog(d)}{\eta_D} $, we must have that $a_t \leq 4(c_t + d_t)$, therefore $c_t + d_t = \Omega(1)$. Similarly, we can prove that $a_t \geq 0.1 \min\{ c_t , d_t\}$.

as long as $c_t > d_t$, we must have:
$$ \sig(- a_t c_t^3 - b_t) c_t^2 a_t <  \sig(- a_t d_t^3 - b_t) d_t^2 a_t $$

Which implies that:
\begin{align}
        \frac{ c_{t + 1} }{1 + 1/\polylog(d)} -  \frac{ c_{t } }{1 + 1/\polylog(d)} < d_{t + 1} - d_t
\end{align}

Note that initially, $c_0, d_0 = \frac{1}{\polylog(d)} $ and when $t \geq\frac{\polylog(d)}{\eta_D} $, $c_t + d_t = \Omega(1)$. This implies that for every $t \in \left[  \frac{\polylog(d)}{\eta_D},  \frac{\poly(d)}{\eta_D} \right]$, we have: $c_t = \left( 1 \pm \frac{1}{\polylog(d)}\right) d_t$. Which also implies that $c_t, d_t \leq O(\log d)$. 

Similarly, we can prove the bound for $a_t$.

\end{proof}

\section{Induction hypothesis}

\subsection{Case 1: Balanced generator and discriminator}

In this section we consider the case 2 in Lemma~\ref{lem:init0}. Here we give the induction hypothesis to prove the case of balanced generator and discriminator, this is the most difficult case and other cases are just simple modification of this one.  Without loss of generality (by symmetry), let us assume that $a^{(0)}  > 0$ and $a^{(0)} = \frac{1}{\polylog(d)}$ (this happens with probability $1 - o(1)$).

We divide the training into five stages: For a sufficiently large $C = \polylog(d)$
\begin{enumerate}
    \item Stage 1: Before one of the $\alpha(w_i, u_j, t) \geq 1/C $. Call this exact iteration $T_{B, 1}$.
    \item Stage 2: After $T_{B, 1}$, before $T_{B, 2} = T_{B, 1} + \frac{1}{ \eta_D 2^{\sqrt{\log (d)}}}$.
    
    \item Stage 3: After $T_{B, 2}$, before one of the $\alpha(v_i, u_j , t) \geq d^{-0.49}$. Call this exact iteration $T_{B, 3}$.
    
    \item Stage 4: After $T_{B, 3}$, before $a^{(t)} \leq \tilde{O}\left( \frac{1}{\Lambda^2 d^{1/4} } \right)$. Call this exact iteration $T_{B, 4}$.
    
    \item Stage 5: After $T_{B, 4}$, until convergence.

\end{enumerate}
We maintain the following things about $\alpha$ and $a, b$ during each stage:

\paragraph{Stage 1}: We maintain: For every $t \leq T_{B, 1}$:
\begin{enumerate}
    \item (B.1.0). For all but the $i^* \in [m_D]$, and for all $j \in [m_G]$ (Below $*$ can be $w_{i'}, v_{j'} , u_{\ell}$ for every $i' \in [m_D], j' \in [m_G] $ and $\ell \in [2]$).
    
    $$\forall * \not= u_1, u_2: | \alpha(w_i , * , t) -  \alpha(w_i , * , 0) |\leq \frac{1}{d^{0.9}},  \quad | \alpha(w_i , u_{\ell} , t) -  \alpha(w_i , u_{\ell} , 0) |\leq \frac{C}{\sqrt{d}} $$
     \item (B.1.1). For all $j \in [m_G]$:
    $$ | \alpha(v_j , * , t)  -  \alpha(v_j , * , 0)| \leq \frac{\polyloglog(d)}{d} $$
   $$|\alpha(v_j , u_{\ell} , t) |\leq  \frac{1}{d} $$
    
    \item (B.1.2). For  $i^* = i$, we have that: for all $* \not= u_{1}, u_{2}$:
    $$| \alpha(w_i , * , t) -  \alpha(w_i , * , 0) |\leq \frac{1}{d^{0.9}}$$
    \item (B.1.3). $a$ and $b$ remains nice: $$a^{(t)} \in (1 - 1/C, 1 + 1/C) a_0, |b^{(t)}| \leq \frac{1}{C}$$
\end{enumerate}

\paragraph{Stage 2}: For every $t \in [T_{B, 1}, T_{B, 2}]$.
\begin{enumerate}
    \item  (B.1.0), (B.1.1), (B.1.2) still holds.
    \item (B.2.1): For $i = i^*$, we have: $$a^{(t)}, \alpha(w_i, u_{\ell}, t) =  \tilde{\Theta}(1)$$
\end{enumerate}


\paragraph{Stage 3}: For every $t \in [T_{B, 2}, T_{B. 3}]$.
\begin{enumerate}
    \item  (B.1.0),  (B.1.2) still holds.
    
    
    \item (B.3.2): For every $j \in [m_G]$: For $* \not= w_{i^*}, u_1, u_2$, we have:
     $$ | \alpha(v_j , * , t)  -  \alpha(v_j , * , 0)| \leq \frac{C^3}{\sqrt{d}} \| v_j^{(t)} \|_2$$
     
     and 
 $$\alpha(v_i, u_{\ell}, t) \geq -O\left(\frac{1}{d} \right)$$
      
     Moreover, let $\alpha(t) := \max_{j \in [m_G], \ell \in [2]} \langle v_j^{(t)}, u_{\ell} \rangle$, we have that:
   $$ |  \alpha(v_j , w_{i^*} , t) | \leq \polyloglog(d) \alpha(t), \quad |\langle v_j^{(t)}, u_{\ell} \rangle| \leq O(\alpha(t))$$
   
   \item (B.3.3): Balanced update: for every $X$, 
   $$ \sig\left( - a^{(t)} \langle w_{i^*}^{(t)}, X \rangle^3  - b^{(t)}\right)  \in \left[ \frac{1}{\sqrt{d} \polylog(d)} , \frac{1}{\polylog(d)}\right] \sig\left( b^{(t)}\right) $$
   
   and 
   $$ \sig\left( - a^{(t)} \langle w_{i^*}^{(t)}, u_1 \rangle^3  - b^{(t)}\right) = \left( 1 \pm \frac{1}{\polylog(d)} \right) \sig\left( - a^{(t)} \langle w_{i^*}^{(t)}, u_2 \rangle^3  - b^{(t)}\right) $$
     
\end{enumerate}

\paragraph{Stage 4}: For every $t \in [T_{B, 3}, T_{B. 4}]$.
\begin{enumerate}
    \item  (B.3.1), (B.3.2) still holds.
    
    \item (B.4.1) for $i = i^*$, we have that 
    for all $* \not= u_{1}, u_{2}, w_i$:
    $$| \alpha(w_i , * , t) -  \alpha(w_i , * , 0) |\leq \frac{C}{\sqrt{d}}$$
    
    For all $* \in [u_1, u_2, w_i]$:
    $$\alpha (w_i, *, t) = \Theta(\alpha (w_i, *, T_{B, 3}))$$
    \item For every $i, j \in [m_G]$:  $\| v_i^{(t)} \|_2 = \tilde{\Theta}(\|v_j^{(t)} \|_2)$ and after $t = T_{B, 4}$, we have that for every $i \in [m_G]$,  $\| v_i^{(t)} \|_2 = \tilde{\Theta}(d^{1/4})$.
    
    \item $|a^{(t)}|, |b^{(t)}| = O(\log(d))$. 
     
\end{enumerate}

\paragraph{Stage 5}: For every $t \in [T_{B, 4}, T_{0}]$.
\begin{enumerate}

    \item  For every $i \in [m_D]$, 
      $$| \alpha(w_i, *, T_{B, 4}) - \alpha(w_i, *, t ) | \leq d^{-0.1}$$
    
    \item For every $i \in [m_G]$, 
    $$| \alpha(v_i, *, T_{B, 4}) - \alpha(v_i, *, t ) | \leq d^{0.2}$$
    
    \item $|a^{(t)}| \leq \tilde{O}\left(\frac{1}{\Lambda^2 d^{1/4}} \right)$, and for every $z$: $$\langle w_{i^*}^{(t)}, G^{(t)}(z) \rangle = \tilde{\Omega}(d^{1/4})$$
     
\end{enumerate}

\subsection{Case 2: Generator is dominating}

We now consider another case where the generator's learning rate dominates that of the discriminator. This corresponds to case 3 in Lemma~\ref{lem:init0}. In this case, we divide the learning into four stages: For a sufficiently large $C = 2^{\sqrt{\log d}}$:
\begin{enumerate}
    \item Before $\alpha(v_{j_G}, w_{i_G}, t) \geq d^{-0.49}$. Call this iteration $T_{G, 1}$.
    \item After $T_{G, 1}$, before $\alpha(v_{j_G}, w_{i_G}, t) \geq \Lambda$. Call this iteration $T_{G, 2}$.
    \item After iteration $T_{G, 2}$, before $a_t \leq 0$. Call this iteration $T_{G, 3}$.
    \item After $T_{G, 3}$.
\end{enumerate}

We maintain the following induction hypothesis:
\paragraph{Stage 1}: In this stage, we maintain the following induction hypothesis: Let $\alpha(t) := \alpha(v_{j_G}, w_{i_G}, t)$, for every $t \leq T_{G, 1}$:

\begin{enumerate}
    \item (G.1.1). For all $i \in [m_D]$, and for all $j \in [m_G]$:
    $$| \alpha(w_i , * , t) -  \alpha(w_i , * , 0) |\leq \frac{C}{\sqrt{d}} $$
    \item (G.1.2). For all $j \in [m_G]$, for all $* \not= w_{i_G}$:
    $$ | \alpha(v_j , * , t)  -  \alpha(v_j , * , 0)| \leq \frac{C}{\sqrt{d}} \alpha(t)  $$

\end{enumerate}

\paragraph{Stage 2}: In this stage, we maintain: for every $t \in [T_{G, 1}, T_{G, 2}]$
\begin{enumerate}
    \item (G.2.1). For every $i \in [m_D]$, we have:
    $$|\alpha(w_i, *, t ) - \alpha(w_i, *, 0 ) | \leq \frac{1}{C} $$
    \item (G.2.2). For every $j \in [m_G]$, $\alpha(v_i, w_{i_G}, t) \geq d^{-0.49}$.
    \item For every $i \in [m_G]$, we have: for every $* \not= w_{i_G}$:
    $$|\alpha(v_i, *, t ) - \alpha(v_i, *, 0 ) | \leq \frac{2}{C} |\alpha(v_i, w_{i_G}, t)| $$
\end{enumerate}

\paragraph{Stage 3}: In this stage, we maintain: For every $t \in [T_{G,2}, T_{G, 3}]$: 
\begin{enumerate}
    \item (G.2.1), (G.2.2) still holds.
    \item For every $i \in [m_G]$, we have: for every $* = v_r $ or $* = u_{\ell}$:
    $$|\alpha(v_i, *, t ) - \alpha(v_i, *, 0 ) | \leq \frac{2}{C} \|v_i^{(t)} \|_2$$
    
\end{enumerate}

\paragraph{Stage 4}: In this stage, we maintain: For every $t \in [T_{G,3}, T_1]$:
\begin{enumerate}
    \item (G.2.1) still holds.
    \item For every $i \in [m_G]$, we have: 
    $$|\alpha(v_i, *, t ) - \alpha(v_i, *, T_{G, 3} ) | \leq \frac{1}{C} \|v_i^{(T_{G, 3})} \|_2 $$
    \item  $|\alpha_{t} | = \tilde{O}\left( \frac{1}{\Lambda^2 \sqrt{\eta_G/\eta_D}} \right)$, $ \|v_i^{(t)}\|_2 = \tilde{\Theta}(\sqrt{\eta_G/\eta_D})$, and for all $z \not= 0$, $\sum_{i \in [m_D]} h(G^{(t)}( z )) = \tilde{\Theta} ( \Lambda^2 \sqrt{\eta_G/\eta_D})$.
\end{enumerate}



\section{Proof of the learning process in balanced case}\label{app:sgda1}

For simplicity, we are only going to prove the case when $u_1 \bot u_2$ and $\Pr[s_1 = s_2 = 1] = \gamma$. The other case can be proved identically. 

\subsection{Stage 1}

In this stage, by the induction hypothesis we know that $\| v_i^{(t)} \|_2 \leq \tilde{O}(1/\sqrt{d})$. Therefore, the update of $w_i^{(t)}$ can be approximate as:
\begin{lemma} \label{lem:B1.1}
For every $t \leq T_{B, 1}$, we know that: when the random samples are $(X, z)$: 
\begin{align}
    w_i^{(t + 1)} = w_i^{(t)} +  \eta_D a^{(0)} \left(1 \pm \frac{1}{\polylog(d)} \right) \frac{3}{2}  \langle w_i^{(t)}, X \rangle^2 X \pm \eta_D \tilde{O} \left( \frac{1}{d^{1.5}} \right)
\end{align}

Moreover, we have that if $z_i = 1$:
\begin{align}
    v_i^{(t + 1)} 
    &= v_i^{(t)} +  \eta_G a^{(0)}  \frac{3}{2} \left(1 \pm \frac{1}{\polylog(d)} \right) \sum_{j \in [m_D]}   \left( \langle w_j^{(0)}, G^{(t)}(z) \rangle \pm \frac{1}{C^{0.5} d} \right)^2 w_j^{(t)}
    \\
    &= v_i^{(t)} + \eta_G a^{(0)}  \frac{3}{2} \left(1 \pm \frac{1}{\polylog(d)} \right) \sum_{j \in [m_D]}  \left( \langle w_j^{(0)}, G^{(t)}(z) \rangle \pm \frac{1}{C^{0.5} d} \right)^2 w_j^{(0)} \pm  \eta_G {O}\left(\frac{1}{C^{0.5}d^2} \right) 
\end{align}
\end{lemma}

Taking expectation of the above Lemma, we can easily conclude that:
\begin{align}
      \E[w_i^{(t + 1)}] &= \E[w_i^{(t)} ]+ \eta_D  a^{(0)}\left(1 \pm \frac{1}{\polylog(d)} \right) \frac{3}{4}  \left( \langle w_i^{(t)}, u_1 \rangle^2 u_1 + \langle w_i^{(t)}, u_2 \rangle^2 u_2  + \Theta(\gamma) \langle w_i^{(t)}, u_1 + u_2 \rangle^2 (u_1 + u_2)  \right) \label{eq:B1.33}
      \\
      &\pm \eta_D \tilde{O} \left( \frac{1}{d^{1.5}} \right)
\end{align}

and 
\begin{align}
    \E[ \langle w_j^{(0)}, G^{(t)}(z) \rangle^2 \mid z_i = 1 ] &= \langle w_j^{(0)}, v^{(t)}_i \rangle^2 \pm O\left( \frac{1}{m_G \polylog (d)} \sum_{i \in [m_G]} | \langle w_j^{(0)}, v^{(t)}_i \rangle  | \right)^2
\end{align}

Therefore, let $\zeta_t = \max_{i \in [m_G], j \in [m_D]} \langle v_i^{(t)}, w_j^{(0)} \rangle$, $\Upsilon_t = \max_{j \in [m_D], \ell \in [2]} \langle w_j^{(t)}, u_{\ell} \rangle$, we have that:
\begin{align}
    \E[\Upsilon_{t + 1}]&= \Upsilon + \eta_D a^{(0)} \frac{3}{4} \left( 1 \pm \frac{1}{\polylog(d)} \right)\Upsilon_t^2   \label{eq:b1.2}
    \\
    \E[\zeta_{t + 1}] &= \zeta_t + \eta_G  a^{(0)}\frac{3}{2m_G} \left( 1 \pm \frac{1}{\polylog(d)} \right) \zeta_t^2
\end{align}

\begin{proof}[Proof of Lemma~\ref{lem:B1.1}]

 By the gradient formula, we have:
 $$\nabla_{w_i} L(X, z) =- \sig(- f(X)) a \sigma'(\langle w_i, X\rangle) X + \sig(f(G(z))) a \sigma'(\langle w_i, G(z) \rangle) G(z) $$
$$ \nabla_{v_i} L(X, z) = -  1_{z_i = 1} \sig(f(G(z))) a \sum_{j \in [m_D]} \sigma'(\langle w_i, G(z) \rangle) w_i $$

At iteration $t$, by induction hypothesis, we have that $a^{(t)} = a^{(0)} (1 \pm 1/C)$. 

Moreover, by the induction hypothesis again, we have hat $|f(X)| = \tilde{O}(d^{-1.5})$ and $\|G(z) \|_2 \leq{O}(d^{-0.5})$. Together with $\| w_i^{(t)} \|_2 = \tilde{O}(1)$,  this implies that 
$$\|  \sig(f(G(z))) a \sigma'(\langle w_i, G(z) \rangle) G(z) \|_2 = \tilde{O}(d^{-1.5})$$

This proves the update formula for $w_i^{(t)}$. As for $v_i$, we observe that by the induction hypothesis and notice that w.h.p. over the randomness of initialization, $|\langle v_i^{(0)}, u_{\ell} \rangle| \leq \frac{\log d}{d} $, therefore, we can conclude that
\begin{align}
    \langle w_j^{(t)} , G^{(t)}(z) \rangle  =  \langle w_j^{(0)} , G^{(t)}(z) \rangle \pm \tilde{O} \left( \frac{1}{d^{1.35}} \right) \pm O \left( \frac{ \log d}{  C d} \right) =  \langle w_j^{(0)} , G^{(t)}(z) \rangle  \pm \frac{1}{C^{0.5 } d}
\end{align}

Note that by induction hypothesis, $ \| w_j^{(t)} - w_j^{(0)} \|_2 \leq \frac{1}{C}$ and $\langle w_j^{(0)} , G^{(t)}(z) \rangle \leq \frac{m_G (C^{0.1} + \log d )}{d}$. This implies that
 \begin{align}
      \langle w_j^{(t)}, G^{(t)}(z) \rangle^2 w_j^{(t)} &= \left(   \langle w_j^{(0)}, G^{(t)}(z) \rangle \pm \frac{1}{C^{0.5} d} \right)^2 w_j^{(t)}
      \\
      &=\left(   \langle w_j^{(0)}, G^{(t)}(z) \rangle \pm \frac{1}{C^{0.5} d} \right)^2 w_j^{(0)}+ {O} \left( \frac{1}{C^{0.5} d^{2}}  \right)
 \end{align}

\end{proof}

Now, apply Lemma~\ref{lem:powerful} and the fact that w.p. $1 - o(1)$, $\zeta_0 = \frac{\polyloglog(d)}{d}$, we have that:
\begin{lemma}\label{lem:B1.5}
\begin{align}
\sum_{t \leq T_1} \eta_G \zeta_t^2 \leq O\left(\frac{m_G \polyloglog(d) }{a^{(0)} d } \right)
\end{align}
\end{lemma}

In the end, we can show the following Lemma:
\begin{lemma}\label{lem:B1.7}

When $t = T_{B, 1}$, we have that: for both $\ell \in [2]$,
$$\alpha(w_{i^*}, u_{\ell}, t) = \frac{1}{\polylog(d)}$$

\end{lemma}

\begin{proof}[Proof of Lemma~\ref{lem:B1.7}]

By the update formula in Eq~\eqref{eq:B1.33}, and the fact that $\Pr[X = u_1 + u_2] \geq \frac{1}{\polylog(d)}$ and the induction hypothesis, we know that for $i = i^*$, for $t \leq T_{B, 1}$ we have that:
$$\alpha(w_i, u_{\ell}, t + 1) \geq \alpha(w_i, u_{\ell}, t) + \tilde{\Omega}(\eta_D) \times  \left(\alpha(w_i, u_{3 - \ell}, t)  - \frac{1}{d} \right)^2$$

This implies that at the end of Stage 1, when $\alpha(w_i, u_{3 - \ell}, t) \geq \frac{1}{C}$, we must have $\alpha(w_i, u_{\ell}, t) \geq \tilde{\Omega}(1)$ as well. 
\end{proof}


\subsection{Stage 2 and Stage 3}

At this stage, by the induction hypothesis, we can approximate the function value as:
\begin{align}
   & \sum_{i \in [m_D]}  \sigma \left( \langle w_i^{(t)}, X \rangle \right) =  \langle w_{i^*}^{(t)}, X \rangle^3  \pm \tilde{O}\left( \frac{1}{d^{1.5}} \right)
    \\
    & \left| \sum_{i \in [m_D]}  \sigma \left( \langle w_i^{(t)}, G^{(t)}(z) \rangle \right)  \right| \leq   \tilde{O} \left( \|  G^{(t)}(z) \|_2 \right)^3 \leq \tilde{O} \left( \frac{1}{d^{1.45}} \right)
\end{align}

Therefore, at this stage, we can easily approximate the update of $W_D^{(t)}$ as:
\begin{lemma} \label{lem:update_w_222}

When the sample is $(X, z)$, we have: for every $t \in (T_{B, 1}, T_{B, 3}]$, the following holds:
\begin{align}
a^{(t + 1)} &= a^{(t)} + \eta_D \left( 1 \pm \tilde{O}\left(\frac{1}{d} \right)\right)\sig\left(-a^{(t)} \langle w_{i^*}^{(t)}, X \rangle^3 - b^{(t)} \right) \langle w_{i^*}^{(t)}, X \rangle^3
\\
&\pm \eta_D   \tilde{O}\left( \frac{1}{d^{1.45}} \right) \sig\left( b^{(t)}\right)
\\
w_{i}^{(t + 1)} &= w_{i}^{(t)} + 3 \eta_D\left( 1 \pm \tilde{O}\left(\frac{1}{d} \right)\right)\sig\left(-a^{(t)} \langle w_{i}^{(t)}, X \rangle^3  - b^{(t)}\right) a^{(t)}\langle w_{i}^{(t)}, X \rangle^2 X
\\
&\pm \eta_D  \tilde{O}\left( \frac{1}{d^{1.45}} \right) \sig\left( b^{(t)}\right)
\\
  b^{(t)} &= b^{(t)} + \eta_D \tau_b  \left( 1 \pm \tilde{O}\left(\frac{1}{d} \right)\right)  \sig\left(- a^{(t)} \langle w_{i^*}^{(t)}, X \rangle^3  - b^{(t)}\right) 
  \\
  &-   \eta_D \tau_b \left( 1 \pm \tilde{O}\left(\frac{1}{d^{1.45}} \right)\right) \sig\left( b^{(t)}\right)
\end{align}

\end{lemma}



Moreover, the update formula also let us bound $a^{(t)}, \alpha(w_{i^*}, u_1, t)$ as:
\begin{lemma} \label{lem:B2.2}
Let $\alpha_t , a_t$ be updated as: for $t = T_{B, 2}$,  $\alpha_{t} =\alpha(w_{i^*}, u_1, t) $ and $a_t = a^{(t)}$, such that
$$a_{t + 1} = a_{t} +  \eta_G  \sig( - a_t \alpha_t^{3} - b_t) \alpha_t^3 $$
$$\alpha_{t + 1} = \alpha_{t} + \frac{3}{2} \eta_G a_t \sig( - a_t \alpha_t^{3} - b_t) \alpha_t^2 $$
Where $b_t$ be updated as: for $t = T_{B, 2}$, $b_{t} = b^{(t)}$ and update as: $$b_{t + 1} = b_t - \eta_D \tau_b \sig(b_t)$$

Then we have: for every $t \in [ T_{B, 2}, T_{B, 3} ]$
$$a_t = \left( 1 \pm \frac{1}{\polylog(d)} \right) a^{(t)}, \quad \alpha_t =  \left( 1 \pm \frac{1}{\polylog(d)} \right) \alpha(w_{i^*}, u_1, t)  $$

$$\sig(b^{(t)}) = \left( 1 \pm \frac{1}{\polylog(d)} \right) \sig(b_t)$$

Moreover, when $t = T_{B, 3}$, we have: $a_{t} \leq 0.819 \alpha_{t}$.
\end{lemma}

\begin{proof}[Proof of Lemma~\ref{lem:B2.2}]

By the update formula in Lemma~\ref{lem:update_w_222} and the bound in induction hypothesis  (B.3.3), we can simplify the update of $a^{(t)}, b^{(t)}$ and $\alpha(w_i, u_{\ell}, t)$ as: for $i = i^*$, when $X = u_{\ell}$:
\begin{align}
a^{(t + 1)} &= a^{(t)} + \eta_D \left( 1 \pm \frac{1}{\polylog(d)}\right)\sig\left(- a^{(t)}  \alpha(w_i, u_{\ell}, t)^3 - b^{(t)} \right)  \alpha(w_i, u_{\ell}, t)^3
\\
\alpha(w_i, u_{\ell}, t + 1) &= \alpha(w_i, u_{\ell}, t )
\\
&+ 3 \eta_D\left( 1 \pm \frac{1}{\polylog(d)}\right)\sig\left(- a^{(t)}  \alpha(w_i, u_{\ell}, t)^3  - b^{(t)}\right) a^{(t)} \alpha(w_i, u_{\ell}, t)^2
\\
  b^{(t)} &= b^{(t)}   -   \eta_D \tau_b \left( 1 \pm \frac{1}{\polylog(d)}\right) \sig\left( b^{(t)}\right)
\end{align}

\t
By the last inequality, we know that when $ \sig(b^{(t)}) > \left( 1 + \frac{1}{\polylog(d)} \right) \sig(b_t) $, then $b^{(t)}$ must be decreasing faster than $b_t$, otherwise if $ \sig(b^{(t)}) > \left( 1 - \frac{1}{\polylog(d)} \right) \sig(b_t) $, then $b_t$ must be decreasing faster than $b^{(t)}$, which proves the bound of $b^{(t)}$. Moreover, the update formula of $b_t$ also gives us that for every $t \leq \poly(d)$, we have: $|b_t|  = O(\log d) $. This implies that for every $Z$: 
$$\sig(Z + b^{(t)}) = \left( 1 \pm \frac{1}{\polylog(d)}\right) \sig(Z + b_t)$$
To obtain the bound of $a^{(t)}$ and $\alpha(w_i^*, u_1, t)$, notice that when $X = u_1 + u_2  $, we have that:
$$\sig\left(- a^{(t)}  (\alpha(w_i, u_{1}, t) + \alpha(w_i, u_{2}, t))^3 - b^{(t)} \right)  \leq \min_{\ell \in [2]}\sig\left(- a^{(t)}  \alpha(w_i, u_{\ell}, t)^3 - b^{(t)} \right)  $$

Therefore, we can conclude:
\begin{align}
\E[a^{(t + 1)}] &= a^{(t)}
\\
&+ \eta_D \frac{1}{2}\left( 1 \pm \frac{1}{\polylog(d)}\right)\left( \sum_{\ell \in [2]}\sig\left(-a^{(t)}  \alpha(w_i, u_{\ell}, t)^3 - b_t \right)  \alpha(w_i, u_{\ell}, t)^3 \right)
\\
\E[\alpha(w_i, u_{\ell}, t + 1)] &= \alpha(w_i, u_{\ell}, t ) 
\\
&+ \frac{3}{2} \eta_D\left( 1 \pm \frac{1}{\polylog(d)}\right)\sig\left(- a^{(t)}  \alpha(w_i, u_{\ell}, t)^3  - b_t \right) a^{(t)} \alpha(w_i, u_{\ell}, t)^2
\end{align}

Using Lemma~\ref{lem:B1.7} we can conclude that 
$$\alpha(w_{i^*}, u_{\ell}, T_{B, 1}) = \frac{1}{\polylog(d)}$$

and now apply Lemma~\ref{lem:powerful_666}, we have: $a^{(t)} = \Theta(\alpha(w_{i^*}, u_{1}, t) )$ and 
$$\alpha(w_{i^*}, u_{1}, t) = \left[\alpha(w_{i^*}, u_{1}, t) \right] \left( 1 \pm \frac{1}{\polyloglog(d)} \right) $$

This implies that:
\begin{align}
\E[a^{(t + 1)}] &= a^{(t)} + \eta_D \left( 1 \pm \frac{1}{\polylog(d)}\right)\left( \sig\left(-a^{(t)}  \alpha(w_i, u_{1}, t)^3 - b_t \right)  \alpha(w_i, u_{\ell}, t)^3 \right)
\\
\E[\alpha(w_i, u_{1}, t + 1)] &= \alpha(w_i, u_{1}, t )
\\
&+ \frac{3}{2} \eta_D\left( 1 \pm \frac{1}{\polylog(d)}\right)\sig\left(- a^{(t)}  \alpha(w_i, u_{1}, t)^3  - b_t \right) a^{(t)} \alpha(w_i, u_{1}, t)^2
\end{align}

Apply Lemma~\ref{lem:powerful_666} again, we know that when 
$$a^{(t)}  \alpha(w_i, u_{1}, t)^3 > \left(1 \pm \frac{1}{\polylog(d)} \right) a_t \alpha_t^3$$

We must have that $a^{(t)} \geq a_t$ and $ \alpha(w_i, u_{1}, t) >  \alpha_t$. Therefore, apply Lemma~\ref{lem:aux_1} we know that  in this case:
$$ \sig\left(-a^{(t)}  \alpha(w_i, u_{1}, t)^3 - b_t \right)  \alpha(w_i, u_{\ell}, t)^3 \leq \sig( - a_t \alpha_t^{3} - b_t) \alpha_t^3 $$

and 
$$\sig\left(- a^{(t)}  \alpha(w_i, u_{1}, t)^3  - b_t \right) a^{(t)} \alpha(w_i, u_{1}, t)^2 \leq  \sig( - a_t \alpha_t^{3} - b_t) a_t \alpha_t^2 $$

Combine this with the update rule we can directly complete the proof. 





\end{proof}

The Lemma~\ref{lem:B2.2} immediately implies that the $\alpha(w_{i^*}, u_{\ell}, t)$ will be balanced after a while:
\begin{lemma}\label{lem:B2.5}
We have that for every $t \in [T_{B, 2}, T_{B, 3}]$, the following holds:
$$\alpha(w_{i^*}, u_{1}, t) = \left[\alpha(w_{i^*}, u_{1}, t) \right] \left( 1 \pm \frac{1}{\polyloglog(d)} \right) $$

and $$\alpha(w_{i^*}, u_{1}, t)  \geq \log^{0.1} (d)$$
\end{lemma}

Using Lemma~\ref{lem:B2.5}, we also have the Lemma that approximate the update of $ v_i^{(t)} $ as:
\begin{lemma} \label{lem:B2.1}

Let us define $\alpha(t) := \max_{j \in [m_G], \ell \in [2]} \langle v_j^{(t)}, u_{\ell} \rangle$.
For  every $t \in  [T_{B, 2}, T_{B, 3}]$, we have:  for $j \not= i^*$:
\begin{align}\label{eq:B2.10}
    \langle w_j^{(t)}, G^{(t)}(z ) \rangle^2  w_j^{(t)} =\langle w_j^{(t)}, G^{(t)}(z ) \rangle^2  w_j^{(0)} \pm \frac{C^2}{d^{0.5}} \langle w_j^{(t)}, G^{(t)}(z ) \rangle^2
\end{align}

For $j = i^*$:
\begin{align}\label{eq:B2.11}
    \langle w_j^{(t)}, G^{(t)}(z ) \rangle^2  w_j^{(t)} =\langle w_j^{(t)}, G^{(t)}(z ) \rangle^2 \left(   w_j^{(0)} + \alpha(w_j, u_1, t) u_1 + \alpha(w_j, u_2, t) u_2  \right) \pm \frac{C}{d^{0.9}} \langle w_j^{(t)}, G^{(t)}(z ) \rangle^2
\end{align}

Now, for $\langle w_j^{(t)}, G^{(t)}(z ) \rangle$ we have: For $j \not= i^*$:
\begin{align}
    \mathbb{E}_z [\langle w_j^{(t)}, G^{(t)}(z ) \rangle^2  \mid z_i = 1] &= \left( 1 \pm \frac{1}{\polylog(d)} \right) \left( \langle w_j^{(0)}, v^{(t)}_i \rangle \pm \frac{C^2}{ \sqrt{d}} \alpha(t) \right)^2 \pm \frac{\alpha(t)^2}{\polylog(d)}
\end{align}

For $j= i ^*$:
\begin{align}
    \mathbb{E}_z [\langle w_j^{(t)}, G^{(t)}(z ) \rangle^2  \mid z_i = 1] &= \left( 1 \pm \frac{1}{\polyloglog(d)} \right)  \alpha(w_{i^*}, u_1, t)^2 \left\langle (u_1 + u_2), v^{(t)}_i \right\rangle^2 
    \\
    &\pm \frac{\alpha(t)^2}{\polyloglog(d)}  \alpha(w_{i^*}, u_1, t)^2 
\end{align}

Moreover, the update of Sigmoid can be approximate as:
$$ \sig \left(f^{(t)} \left( G^{(t)}(z) \right) \right) = \left( 1 \pm \tilde{O}\left(\frac{1}{d^{1.45}} \right)\right) \sig\left( b^{(t)}\right)$$

\end{lemma}

\begin{proof}[Proof of Lemma \ref{lem:B2.1}]
 
 The first half of the lemma regarding $w_j^{(t)}$ follows trivially from the induction hypothesis, we only need to look at 
 $\langle w_j^{(t)}, G^{(t)}(z ) \rangle$.

 We know that  for $j \not= i^* $, we have that by the induction hypothesis,
\begin{align}
    \langle w_j^{(t)} , G^{(t)}(z) \rangle  =  \langle w_j^{(0)} , G^{(t)}(z) \rangle \pm \tilde{O} \left( \frac{1}{d^{1.35}} \right) \pm  O\left( \frac{C m_G \alpha(t) }{ \sqrt{d}} \right)
\end{align}

 For $j = i^*$, we have that 
 \begin{align}
    \langle w_j^{(t)} , G^{(t)}(z) \rangle  &=  \langle w_j^{(0)} , G^{(t)}(z) \rangle + \alpha(w_{j}, u_1, t) \langle u_1, G^{(t)}(z) \rangle + \alpha(w_{j}, u_2, t) \langle u_2, G^{(t)}(z) \rangle    \pm \tilde{O} \left( \frac{1}{d^{1.35}} \right)
    \\
    &= \alpha(w_j, u_1, t) \langle u_1 + u_2, G^{(t)}(z) \rangle \pm \frac{1}{\polyloglog(d)} \alpha(w_j, u_1, t) \alpha(t) \| z\|_1   \pm \tilde{O} \left( \frac{1}{d^{1.35}} \right)
    \\
    &= \alpha(w_j, u_1, t) \langle u_1 + u_2,v_i^{(t)} \rangle \pm \tilde{O} \left( \frac{1}{d^{1.35}} \right) \pm  O( \alpha(w_j, u_1, t)  \alpha(t) ) (\| z \|_1 - 1)
\end{align} 

Taking expectation we can complete the proof.
 
\end{proof}



With Eq~\eqref{eq:B2.10} and Eq~\eqref{eq:B2.11} in lemma~\ref{lem:B2.1}, together with the induction hypothesis, we immediately obtain
\begin{lemma}\label{lem:B3.3}

For every $t \in [T_{B, 2}, T_{B, 3}]$, we have that: for every $i \in [m_G]$:
\begin{align}
    v_i^{(t)} &= v_i^{(T_{B, 2})} + \sum_{\ell \in [2]} \alpha_{i, \ell}^{(t)}u_{\ell} + \sum_{j \in [m_D]} \beta_{i, j }^{(t)}  w_{j}^{(0)} \pm \xi_{i, t}
\end{align}

Where $\alpha_{i, \ell}^{(t)}, \beta_{i, j }^{(t)}  > 0$ and  $\alpha_{i, \ell}^{(t)} = (1 \pm o(1)) \alpha_{i, 3 - \ell}^{(t)}$; $\| \xi_{i, t } \|_2^2 \leq \tilde{O}(1/d) \left( \sum_{\ell, j}(\alpha_{i, \ell}^{(t)})^2 + (\beta_{i, j }^{(t)})^2 \right)$
\end{lemma}



We now can immediately control the update of $v_i^{(t)}$ using the following sequence:
\begin{lemma} \label{lem:B3.8}
Let $v_t$ be defined as: for $t = T_{B, 2}$
$$v_{t} = \max_{i \in [m_G]} \langle v_i^{(t)}, u_1 + u_2 \rangle \left( 1 + \frac{1}{\polyloglog(d)} \right)$$

and the update of $v_t$ is given as: for $\alpha_t$ defined as in Lemma~\ref{lem:B2.2}
$$v_{t+ 1} = v_t + \frac{3}{m_G} \sig(b_t) \alpha_t^2 v_t^2 $$

Then we must have: for every $t \in [T_{B, 2}, T_{B, 3}]$: 
\begin{align}
   \max_{i \in [m_G]}\langle v_i^{(t)}, u_1 + u_2 \rangle \leq  v_{t}
\end{align}

On the other hand, if for $t = T_{B, 2}$, 
$$v_{t} = \max_{i \in [m_G]} \langle v_i^{(t)}, u_1 + u_2 \rangle \left( 1 - \frac{1}{\polyloglog(d)} \right)$$

Then we must have: for every $t \in [T_{B, 2}, T_{B, 3}]$: 
\begin{align}
   \max_{i \in [m_G]}\langle v_i^{(t)}, u_1 + u_2 \rangle \geq v_{t}
\end{align}


\end{lemma}

\begin{proof}[Proof of Lemma~\ref{lem:B3.8}]

In the setting of Lemma~\ref{lem:B2.1}, let us define $beta(t):= \max_{j \in [m_G]} \langle v_j^{(t)}, u_1 + u_2 \rangle$. We have that:
\begin{align}
   & \beta(t + 1) = \beta(t) + \eta_G \alpha_t \left( 1 \pm \frac{1}{\polylog(d)} \right) \frac{3}{m_G}\beta(t)^2 \pm \frac{1}{\polyloglog(d)} \alpha(t)^2
\end{align}

By the induction hypothesis we know that for all $j \in [m_G], \ell \in [2]$:
\begin{align}
    \langle v_j^{(t)}, u_{\ell} \rangle \geq   \langle v_j^{(0)}, u_{\ell} \rangle - O\left(\frac{1}{d} \right) \geq - \frac{\log\log^2 d}{d}
\end{align}

This implies that
$\beta(t) \geq \alpha(t) - \frac{\log\log^2(d)}{\sqrt{d}}$ and $\beta(T_{B, 2}) \geq \frac{1}{d}$, $\alpha(T_{B, 2}) \leq\frac{\polyloglog(d)}{\sqrt{d}} $. This implies that:
\begin{align}
   & \beta(t + 1) = \beta(t) + \eta_G \alpha_t \left( 1 \pm \frac{1}{\polyloglog(d)} \right) \frac{3}{m_G}\beta(t)^2
\end{align}

This completes the proof by applying Lemma~\ref{lem:TPM}. 
\end{proof}

Now, by the comparison Lemma~\ref{lem:TPM}, we know that one of the following event would happen (depending on the initial value of $v_t$ at iteration $T_{B , 2}$):
\begin{lemma}\label{lem:B3.4}
With probability $1 - o(1)$, one of the following would happen:
\begin{enumerate}
    \item $T_{B, 3} \geq T_0$. 
    \item $T_{B, 3} < T_0$, moreover, at iteration $T_{B, 3}$, we have that $\sig(b_t) \geq \frac{1}{\polylog(d)}$.
\end{enumerate}
\end{lemma}


In the end, we can easily derive an upper bound on the sum of $\sig$ as below, which will be used to prove the induction hypothesis. 
\begin{lemma} \label{lem:B3.1}
For every $t \in (T_{B, 1},  T_{B, 3}]$, we have that: for every $X, z$:
\begin{align}
  &\sum_{t\in (T_{B, 1},  T_{B, 3}] }\eta_D  \sig\left(a^{(t)} \langle w_{i^*}^{(t)}, X \rangle^3  + b^{(t)} \right) = \tilde{O}(1)
  \\
 & \sum_{t\in (T_{B, 1},  T_{B, 3}] } \eta_D \tau_b \left(b^{(t)}\right) \leq \tilde{O}(1)
\end{align}

\end{lemma}

We will also show the following Lemma regarding all the $v_i^{(t)}$ at iteration $T_{B, 3}$:
\begin{lemma} \label{lem:B3.7}
For all $i \in [m_G]$, if we are in case 2 in Lemma~\ref{lem:B3.4}, we have that:
\begin{align}
    \langle v_i^{(t)}, u_1 \rangle,   \langle v_i^{(t)}, u_2 \rangle  = \tilde{\Omega}(d^{-0.49})
\end{align}
\end{lemma}

\begin{proof}[Proof of Lemma~\ref{lem:B3.7}]

Since with probability at least $1/\poly(d)$, $z_i  = z_j = 1$, so we have: By the update Lemma~\ref{lem:B2.1} of $v$, we know that for all $j \in [m_G]$: Let $\alpha(t)$ be defined as in Lemma~\ref{lem:B2.1}:
\begin{align}
   \E[ \alpha(v_j, u_{\ell}, t + 1)] &\geq \alpha(v_j, u_{\ell}, t) + \eta_G \frac{1}{\polylog(d)} \left(\alpha(t) - \tilde{O}\left( \frac{1}{d} \right) \right)^2 \pm \tilde{O}\left( \frac{1}{d^{0.8}} \right) \alpha(t)^2
   \\
    \E[ \alpha(v_j, u_{\ell}, t + 1)]  &\leq \alpha(v_j, u_{\ell}, t) + \eta_G \polylog(d) \alpha(t)^2
\end{align}

By Lemma~\ref{lem:B3.4} we know that $\alpha(t) = \tilde{\Theta}(d^{-0.49})$ at iteration $t = T_{B, 3}$, which implies what we want to prove. 
\end{proof}

\subsection{Stage 4 and 5}

In Stage 4 we can easily calculate that by induction hypothesis,  for every $i \in [m_G]$ and for every $j \in [m_D]$, $j \not= i^*$:
$$|\langle v_i^{(t)}, w_j^{(t)} \rangle | \leq \tilde{O} \left( \frac{\| v_i^{(t)} \|_2 }{\sqrt{d}} \right)$$

Let $$S_{i, t} = \E_{z} \left[ \sig \left(a^{(t)} \sigma \left( \langle w_{i^*}^{(t)}, G^{(t)}  (z) \rangle \right) + b^{(t)}  \right)  \mid z_i = 1 \right]$$
Note that by induction hypothesis, $|a^{(t)}|, b^{(t)} = O(\log(d))$. Which implies that as long as $\max_{i \in [m_G]}\| v_i^{(t) } \|_2 \leq \frac{1}{\polylog(d)}$ or for all $i \in [m_G]$, $a^{(t)} \sigma'(\langle v_i^{(t)}, w_{i^*} \rangle ) \geq \tilde{\Omega}(\log(d))$, we have that: for all $z, z'$ we have that:
$$ \sig \left(a^{(t)} \sigma \left( \langle w_{i^*}^{(t)}, G^{(t)}  (z) \rangle \right) + b^{(t)}  \right)    = \Theta(1) \times  \sig \left(a^{(t)} \sigma \left( \langle w_{i^*}^{(t)}, G^{(t)}  (z') \rangle \right) + b^{(t)}  \right)  $$

We can immediately obtain the following Lemma: 
\begin{lemma}\label{lem:B4.9}

The update of $v_i^{(t)}$ is given as: 
\begin{align}
   \E[ v_{i}^{(t + 1)}] =  v_i^{(t)} +  \tilde{\Theta}(\eta_G) a^{(t)} S_{i, t} \left(   \left( \sigma'(\langle w_{i^*}^{(t)}, v_i^{(t)} \rangle  + \sum_{j \in [m_G]}\gamma_{i, j, t} \sigma'(\langle w_{i^*}^{(t)}, v_j^{(t)} \rangle ) \right)w_i^{(t)} \pm \tilde{O} \left( \frac{\| v_i^{(t)} \|_2 }{\sqrt{d}} \right)^2 \right)
\end{align}

Where $\gamma_{i,j,t} > 0$; $\gamma_{i, j, t} = \tilde{\Theta}(1)$ if $\max_{i \in [m_G]}\| v_i^{(t) } \|_2 \leq \frac{1}{\polylog(d)}$ or for all $i \in [m_G]$, $a^{(t)} \sigma'(\langle v_i^{(t)}, w_{i^*} \rangle ) \geq \tilde{\Omega}(\log(d))$, and $\gamma_{i, j, t} = \tilde{O}(1)$ otherwise. 

\end{lemma}

Here the additional $\sigma'(\langle w_{i^*}^{(t)}, v_j^{(t)} \rangle )$ part comes from $\Pr[z_i, z_j = 1] = \frac{1}{\polylog(d)}$. The remaining part of this stage follows from simply apply Lemma~\ref{lem:powerful}.

In stage 5, we bound the update of $a^{(t)}, b^{(t)} $ as:

Let $$S_t = \E_{z} \left[ \sig \left(a^{(t)} \sigma \left( \langle w_{i^*}^{(t)}, G^{(t)}  (z) \rangle \right) + b^{(t)}  \right) \right]$$

In this stage, with the induction hypothesis, we can easily approximate the sigmoid as:
\begin{lemma}
For $t \geq T_{B, 4}$, the sigmoid can be approximate as: For every $X, z$
$$\sig(- f^{(t)} (X)) =\left( 1 \pm \frac{1}{\polylog(d)} \right) \sig( - b^{(t)}) $$

$$\sig( f^{(t)} (G^{(t)}(z))) =\left( 1 \pm \frac{1}{\polylog(d)} \right)\sig \left(a^{(t)} \sigma \left( \langle w_{i^*}^{(t)}, G^{(t)}  (z) \rangle \right) + b^{(t)}  \right)$$
\end{lemma}

Then by the update rule, we can easily conclude that:
\begin{lemma}
For $t \geq T_{B, 4}$, the update of $a^{(t)}, b^{(t)} $ is given as:
$$a^{(t + 1)} =  a^{(t)}  + \tilde{O}(\eta_D)  \sig \left(  - b^{(t)} \right) - \tilde{\Omega}(\eta_D) S_t \Lambda^2d^{1/4} $$

$$\E[b^{(t + 1)} ]= b^{(t)} + \eta_D \tau_b \left( 1 \pm \frac{1}{\polylog(d)} \right) \sig \left( -  b^{(t)} \right) - \eta_D \tau_b \left( 1 \pm \frac{1}{\polylog(d)} \right) S_t  $$
\end{lemma}

This Lemma, together with the induction hypothesis, implies that:
\begin{lemma} \label{lem:B4.2}
We have:
\begin{align}
    &\sum_{t \geq T_{B, 4}} S_t \leq  \tilde{O} \left(\frac{1}{ \eta_D\tau_b \Lambda^2 d^{1/4}}  \right)
    \\
    & \sum_{t \geq T_{B, 4}} \sig(b^{(t)}) \leq \tilde{O} \left( \frac{1}{\eta_D \tau_b} \right)
\end{align}
\end{lemma}

\begin{proof}[Proof of Lemma~\ref{lem:B4.2}]

Let us denote $R = \sum_{t = T_{B, 4}}^{T_0}  \sig \left(  - b^{(t)} \right) $ and $Q =  \sum_{t = T_{B, 4}}^{T_0}  S_t$

Sum the update up for $t = T_{B, 4}$ to $T_{0}$, we have that:
\begin{align}
  a^{(T_0)} - a^{(T_{B, 4})} &= \tilde{O}(\eta_D) R - \tilde{\Omega}(\eta_D) Q \Lambda^2 d^{1/4}
  \\
  \E[b^{(T_0)}] - b^{(T_{B, 4})} &= \Theta(\eta_D \tau_b) R - \Theta(\eta_D \tau_b) Q
\end{align}

By the induction hypothesis that $|a^{(t)} | \leq \frac{\tilde{O}(1)}{\Lambda^2 d^{1/4}}$ and $|b^{t}| = \tilde{O}(1)$, we have that:
\begin{align}
& | \tilde{O}(\eta_D) R - \tilde{\Omega}(\eta_D) Q \Lambda^2 d^{1/4}| \leq \frac{\tilde{O}(1)}{\Lambda^2 d^{1/4}}
\\
& |\Theta(\eta_D \tau_b) R - \Theta(\eta_D \tau_b) Q|  \leq \tilde{O}(1)
\end{align}

Thus, we have:
\begin{align}\tilde{\Omega}(\eta_D) Q \Lambda^2 d^{1/4} &\leq \tilde{O}(\eta_D) R + \frac{\tilde{O}(1)}{\Lambda^2 d^{1/4}}  \leq \tilde{O}(\eta_D) \left( \frac{1}{\eta_D \tau_b} + Q \right)+ \frac{\tilde{O}(1)}{\Lambda^2 d^{1/4}} 
\end{align}

Therefore we have that $\tilde{\Omega}(\eta_D) Q \Lambda^2 d^{1/4} \leq \tilde{O}\left( \frac{1}{\eta_D \tau_b} + \frac{1}{\eta_D \Lambda^2 d^{1/4}}  \right)$, which implies that
$$ Q \leq \tilde{O}\left(\frac{1}{ \eta_D\tau_b \Lambda^2 d^{1/4}}  \right)$$

Similarly, we can show that
\begin{align}
R &\leq \tilde{O } \left( \frac{1}{\eta_D \tau_b} + Q \right)  \leq \tilde{O } \left( \frac{1}{\eta_D \tau_b} + \frac{1}{(\Lambda^2 d^{1/4} )^2} + \frac{1}{\Lambda^2 d^{1/4} } R  \right)
\end{align}

This implies that $R \leq \tilde{O} \left( \frac{1}{\eta_D \tau_b} \right)$.

\end{proof}

\subsection{Proof of the induction hypothesis and the final theorem}

The final theorem follows immediately from the induction hypothesis ($v$ part) together with Lemma~\ref{lem:B3.3}.

Now it remains to prove the induction hypothesis. We will assume that all the hypothesises are true until iteration $t$, then we will prove that they are true at iteration $t + 1$. 

\paragraph{Stage 1}.

To prove the induction hypothesis at Stage 1, for $w$, we have that by Lemma~\ref{lem:B1.1},  we know that: for $* \not= u_1, u_2$, 
\begin{align}
    |\alpha(w_i, *, t + 1) - \alpha(w_i, *, t ) | \leq \eta_D \tilde{O}\left( \frac{1}{d^{1.5}} \right)
\end{align}

By $T_1 \leq {O}\left( \frac{\sqrt{d}}{\eta_D a^{(0)}} \right)$ we can conclude that 
\begin{align}
    |\alpha(w_i, *, t + 1) - \alpha(w_i, *, 0 ) | \leq \eta_D \tilde{O}\left( \frac{1}{d^{1.5}} \right) \times {O}\left( \frac{\sqrt{d}}{\eta_Da^{(0)} } \right) \leq \frac{1}{d^{0.9}}
\end{align}

On the $v$ part, again by Lemma~\ref{lem:B1.1}, we know that for $* \notin  \{w_j \}_{j \in [m_D]}$:
\begin{align}
    |\alpha (v_i, *, t) -\alpha(v_i, *, 0 ) | \leq \eta_G \left( \frac{1}{C^{0.5} d^2} \right) \times {O}\left( \frac{\sqrt{d}}{\eta_D a^{(0)}} \right) \leq \frac{1}{d}
\end{align}

On the other hand, we know that for $w_j$:
\begin{align}
    |\alpha (v_i, w_j, t + 1) -\alpha(v_i, w_j, t ) | \leq \eta_G \left( 1 + \frac{1}{\polylog(d)} \right) \frac{3}{2 m_G} \zeta_t^2
\end{align}
Apply Lemma~\ref{lem:B1.5} we complete the proof using Lemma~\ref{lem:comp}.

As for the $a^{(t)}, b^{(t)}$ part, we know that:
\begin{align}
   |  a^{(t + 1)} - a^{(t)} | \leq O\left( \eta_D m_D \Upsilon_t^3 \right), |b^{(t)}| \leq \tau_b \eta_D T_1
\end{align}

Combine with the update rule of $\Upsilon$ in Eq~\eqref{eq:b1.2}, we complete the proof.

\paragraph{Stage 2 and 3}

For the $w$ part, we know that by Lemma~\ref{lem:update_w_222}, we have that for every $* \not= u_1, u_2$
\begin{align}
   | \alpha(w_i, *, t + 1) -  \alpha(w_i, *, t ) | \leq \eta_D \tilde{O}\left( \frac{1}{d^{1.45}} \right)  \sig(b^{(t)})
\end{align}

Now, by Lemma~\ref{lem:B3.1} we have that:
$$\sum_{t\in (T_{B, 1},  T_{B, 3}] } \eta_D \tau_b \left(b^{(t)}\right) \leq \tilde{O}(1)$$

This implies that 
$$ | \alpha(w_i, *, t + 1) -  \alpha(w_i, *, T_{B, 1} ) | \leq \eta_D \tilde{O}\left( \frac{1}{d^{1.45}} \right)  \times \frac{1}{\eta_b \eta_D} \leq \frac{1}{d^{0.9}}$$

For the $v$ part for $t \leq T_{B, 2}$, since $T_{B ,2} - T_{B , 1}  = \tilde{O}(d^{o(1)}/\eta_D)$, we can easily prove it for $t \leq T_{B, 2}$ as in stage 1. On the other hand, for $t \in (T_{B ,2}, T_{B, 3}]$: By Lemma~\ref{lem:B2.1} and Lemma~\ref{lem:B2.5}, we have that define
$$\alpha(t)  := \max_{j \in [m_G], \ell \in [2]} \langle v_j^{(t)}, u_{\ell} \rangle, \quad \beta(t) := \max_{j \in [m_G], j' \in [m_G], j \not= j'; i \in [m_D], i \not= i^*} |\langle v_j^{(t)}, w_i^{(0)} \rangle |  + |\langle v_j^{(t)}, v_{j'}^{(0)} \rangle|$$

We have that:
\begin{align}
    \E[\alpha(t + 1)] \geq  \alpha(t) +\eta_G \Omega\left(\frac{1}{m_G}  \right) \sig(b^{(t)}) \alpha(t)^2 \log^{0.1}(d)
\end{align}

and
\begin{align}
    \E[\beta(t + 1)] \leq  \beta(t) +\eta_G O\left(\frac{1}{m_G}  \right) \sig(b^{(t)}) \left(\beta(t)^2 + \frac{C^2}{\sqrt{d}} \alpha(t)^2 \right)
\end{align}

By Lemma~\ref{lem:init} and Lemma~\ref{lem:TPM} we can show that $\beta(t) = O\left( \beta(0) + \frac{C^2}{\sqrt{d} } \alpha(t)  \right)$, which complete the proof that for all $* \not= w_{i^*}, u_1, u_2$: 
$$|\alpha(v_j, *, t)| \leq \frac{C^3}{\sqrt{d}} \| v_j^{(t)} \|_2$$.

\paragraph{Stage 4 and 5}

At stage 4 we simply use Lemma~\ref{lem:powerful}, the only remaining part is to show that $|b^{(t)}| = O(\log(d))$. To see this, we know that by the update formula:

$$\nabla_b L(X, z) = - \sig(- f(X))  + \sig(f(G(z)))  $$

By our induction hypothesis, we know that $a^{(t)} \left( \sum_{i \in [m_D]} \sigma(\langle w_i^{(t)}, X \rangle) \right)  > 0$

and  $a^{(t)} \left( \sum_{i \in [m_D]} \sigma(\langle w_i^{(t)},G(z) \rangle) \right)  > 0$ . Therefore, $b < O(\log(d))$ is immediate. Now it remains to show that $b > - O(\log d)$: By the update formula, we have:
$$-b^{(t+ 1)} \leq  -b^{(t)} + \eta_D \tau_b \sum_{i \in [m_G]} S_{i, t}$$

and by Lemma~\ref{lem:B4.9} and the proof in Lemma~\ref{lem:powerful}, we have that:
\begin{align}
   \sum_{i \in [m_G]}\langle v_i^{(t + 1)} , w_{i^*}^{(0)}  \rangle \geq  \sum_{i \in [m_G]}\langle v_i^{(t )} , w_{i^*}^{(0)}  \rangle + a^{(t)} \tilde{\Omega}(\eta_G) \left(\sum_{i \in [m_G]} S_{i, t} \right) \left( \sum_{i \in [m_G]} \sigma'(\langle v_i^{(t )} , w_{i^*}^{(0)}  \rangle ) \right)
\end{align}

Compare this two updates we can easily obtain that $|b^{(t)}| = O(\log(d))$.

At stage 5, we have that since $|a^{(t)}| = \tilde{O} \left( \frac{1}{d^{1/4} \Lambda^2} \right)$: For every $j \in [m_G], i \in [m_D]$
\begin{align}
    \| v^{(t + 1)}_j  - v^{(t)}_j \|_2 &\leq \tilde{O}(\eta_G) S_t \Lambda^2 \times \frac{1}{\Lambda^2 d^{1/4}}
    \\
    \| w^{(t + 1) }_i - w^{(t)}_i \|_2 &\leq \tilde{O}(\eta_D) \sig(-b^{(t)}) \frac{1}{d^{1/4} \Lambda^2} &
\end{align}

Apply Lemma~\ref{lem:B4.2} we have that:
\begin{align}
    &\|v^{(t + 1)}_j - v^{(T_{B, 4})}_j \|_2 \leq \tilde{O}(\eta_G)  \times \frac{1}{d^{1/4}}  \times \frac{1}{\eta_D \tau_b \Lambda^2 d^{1/4}} \leq d^{0.15}
    \\
    & \| w^{(t + 1) }_i - w^{(T_{B, 4})}_i \|_2 \leq \tilde{O}(\eta_D) \frac{1}{d^{1/4} \Lambda^2}  \times \frac{1}{\eta_D \tau_b} \leq \frac{1}{d^{0.1}}
\end{align}

Which proves the induction hypothesis. 

\section{Proof of the learning process in other cases}\label{app:sgda2}

We now consider other cases, in case 1 of Lemma~\ref{lem:init0}, the proof is identical to case 2, the only difference is at Stage 3, we have that $T_{B, 3} > T_0$.

In case 2, the Stage 1 is identical to the Stage 1, 2, 3 in the balanced case. For Stage 3, its identical to Stage 4 in the balanced case (the only difference is to apply Lemma~\ref{lem:powerful_777} and the case 2 of Lemma~\ref{lem:powerful233} instead of Lemma~\ref{lem:powerful}). For Stage 4, its identical to Stage 5 in the balanced case. 

At Stage 2, by the induction hypothesis, we know that for $j \not= i_G$, we have that $|\langle v_j^{(t)}, w_j^{(t)} \rangle| \leq \tilde{O}\left(\frac{1}{C} \| v_j^{(t)} \|_2 \right)$. Thus, we can approximate the update of $w, v$ as: 
\begin{align}
    w_i^{(t + 1)} &= w_i^{(t)} \pm \tilde{O}(\eta_D \sum_{j \in [m_G]} \| v_{j}^{(t)} \|_2^2) \Lambda
    \\
    v_j^{(t + 1)} &= v_j^{(t)} + \tilde{\Theta}(\eta_G  \sum_{j \in [m_G]} \| v_{j}^{(t)} \|_2^2 )  w_{i_G}^{(0)} \pm \frac{1}{C^2}\tilde{\Theta}(\eta_G  \sum_{j \in [m_G]} \| v_{j}^{(t)} \|_2^2 )
\end{align}

Using the fact that $\eta_G \geq \tilde{\Omega}(\sqrt{d}) \eta_D$ in case 3 we immediately proves the induction hypothesis.

The proof of the theorem follows immediately from the induction hypothesis on $v$ in this case $v$ only learns noises (linear combinations of $w_i^{(0)}$).

\section{Normalized SGD}\label{app:nsgda_app}

In this section we look at the update of normalized SGD. 

Let us define:
$$i_1^* = \argmax_{i \in [m_D]} \{ \langle w_i^{(0)}, u_1 \rangle \}$$
$$i_2^* = \argmax_{i \in [m_D]} \{ \langle w_i^{(0)}, u_2 \rangle \}$$

Let us define:
$$g_j^* = \argmax_{i \in [m_D]} \{ \langle v_j^{(0)}, w_i^{(0)} \rangle \} $$

Then we first show the following Lemma about initialization:
\begin{lemma} \label{lem:init2}
With probability at least $1 - o(1)$ over the randomness of the initialization, the following holds:
\begin{enumerate}
    \item For all $\ell \in [2]$, for all $i \in [m_D]$ such that $i \not= i^*_{\ell}$,  we have:
    $$\langle w_{i_{\ell}^*}^{(0)}, u_{\ell} \rangle \geq \left( 1- \frac{1}{\polyloglog(d)} \right) \langle w_{i}^{(0)}, u_{\ell} \rangle$$
    \item For all $j \in  [m_G]$, we have that for all $i \in [m_D]$ such that $i \not= g_j^*$, 
    $$\langle v_j^{(0)}, w_{g_j^*}^{(0)} \rangle \geq \left( 1 - \frac{1}{\log^4(d)} \right) \langle v_j^{(0)}, w_{i}^{(0)} \rangle$$
    
    \item $\{g_j^*\}_{j \in [m_G]} = [m_D]$. 
    
\end{enumerate}
\end{lemma}

We now divide the training stage into two: For a sufficiently large $C  = \polylog(d)$, consider the case when $\eta_G = \eta_D * C^{-0.6}$. 
\begin{enumerate}
    \item Stage 1: When both $\alpha(w_{i*_1}, u_1, t), \alpha(w_{i*_2}, u_2, t) \leq \frac{1}{C^{0.95}} $. Call this iteration $T_{N, 1}$.
    \item Stage 2: After $T_{N, 1}$, before $T_1$
\end{enumerate}

\subsection{Induction Hypothesis} 

We will use the following induction hypothesis:  for a

\paragraph{Stage 1:} for every $t \leq T_{N, 1}$: Let $\alpha(t) := \max_{\ell \in [2]}\alpha(w_{i_{\ell}^*}, u_{\ell}, t)  $, $\beta(t) :=  \max_{i \in [m_G]}\alpha(v_i, w_{g_i^*}, t)$. 
\begin{enumerate}
    \item Domination: For every $i \in [m_G]$, we have:
    $$|\alpha(v_i, *, t) - \alpha(v_i, *, 0) | \leq \min \left\{ \frac{1}{C} \alpha(t), \beta(t ) \right\}$$

    For every $i \in [m_D]$, $i \not= i_1^*, i_2^*$, we have that for $* \not= u_1, u_2$:
    $$|\alpha(w_i, *, t) - \alpha(w_i, *, 0) | \leq \frac{1}{C} \alpha(t) $$
    and 
    $$|\alpha(w_i, u_1, t)| , | \alpha(w_i, u_2, t) | \leq \alpha(t)$$
    
    For  $i_1^*$, we have that for every $* \not= u_1$, 
     $$|\alpha(w_{i_1^*}, *, t) - \alpha(w_{i_1^*}, *, 0) | \leq \frac{1}{C} \alpha(t)$$
    
     For  $i_2^*$, we have that for every $* \not= u_2$, 
     $$|\alpha(w_{i_2^*}, *, t) - \alpha(w_{i_2^*}, *, 0) | \leq \frac{1}{C} \alpha(t)$$
     
     
     \item (N.1.2): Growth rate: we have that  for every $i \in [m_D]$
     $$\alpha(w_{i_{\ell}^*}, u_{\ell}, t) \in \left( \Omega\left(\frac{1}{m_D } \right), 1 \right) \eta_D t $$
     
     and for every $i \in [m_G]$: $$\alpha(v_i, w_{g_i^*}, t ) \in \left(\Omega\left(\frac{1}{m_G^2} \right), 1 \right) \eta_G t$$
     
     Therefore by our choice of $\eta_D, \eta_G$ we have that $\beta(t) \in C^{-0.6} \left[ \frac{1}{\log^{5} d}, \log^5d \right] \times \alpha(t)$.
     \item $a^{(t)} = [0.5, 1] a^{(0)}, |b^{(t)}| \leq \frac{1}{d^{0.1}}$. 
\end{enumerate}

\paragraph{Stage 2:}  We maintain: For every $t \in [T_{N, 1}, T_1]$:
\begin{enumerate}
\item (N.1.2) still holds.
\item $ \alpha(w_{i_{\ell}^*}, u_{\ell}, t) \in \left[ \frac{1}{C}, \polylogloglog(d) \right]$,  $\beta(t) \in C^{-0.6} \left[ \frac{1}{\log^{5} d}, \log^5d \right] \times \alpha(t)$, $a^{(t)} = \Omega(\alpha(w_{i_{\ell}^*}, u_{\ell}, t))$.
    \item $w_i$'s are good:  For every $i \not= i_1^*, i_2^*$, for every $*$:
    $$|\alpha(w_i, *, t) - \alpha(w_i, *, 0)| \leq \frac{1}{C} \alpha(t)$$
    
    and for $\ell \in [2]$: for every $* \not= u_{\ell} $, we have:
     $$|\alpha(w_i, *, t) - \alpha(w_i, *, 0)| \leq \frac{1}{C} \alpha(t)$$
     
     \item $v_i$'s are good: For every $i \in [m_G]$ and every $j \in [m_D]$, $j \not= g_i^*$: 
     $$\langle v^{(t)}_i, w_{g_i^*}^{(t)} \rangle \geq C^{0.9} | \langle v^{(t)}_i, w_{j}^{(t) } \rangle|$$
     
     and for $g_{i}^* \not= i_{\ell}^*$, we have that:
     $$\langle v^{(t)}_i, w_{g_i^*}^{(t)} \rangle \geq C^{0.9} | \langle v^{(t)}_i, u_{\ell} \rangle|$$
     
     For $g_i^* = i_{\ell}$ , we have that $\langle v^{(t)}_i, u_{\ell} \rangle \geq - \frac{1}{C^{0.5}} \beta(t)$
\end{enumerate}

\subsection{Stage 1 training}

With the induction hypothesis, we can show the following Lemma:
\begin{lemma}\label{lem:N1.1}

For $t \leq T_{N, 1}$, for $\varepsilon_t := \frac{(\alpha(t) + d^{-0.5})^2 }{C^{1.5}} + C^{0.5} (\alpha(t) + d^{-0.5} )^3 $, when the sample is $X \in \left\{ u_1, u_2 \right\}$,  the update of $w_i^{(t)}$ can be approximate as: 
\begin{align}
w_i^{(t + 1)} &= w_i^{(t)} + \eta_D  \left(1 \pm \frac{1}{\polylog(d)} \right)  \frac{\sigma'(\langle w_i^{(t)}, X \rangle) X \pm\varepsilon_t }{ \sqrt{\sum_{j \in [m_D] } \sigma'(\langle w_j^{(t)}, X \rangle)^2 \| X \|_2^2  }} 
\end{align}

Which can be further simplified as:
\begin{align}
\E[\langle w_i^{(t + 1)} , u_{\ell} \rangle] &= \langle w_i^{(t )} , u_{\ell} \rangle + \eta_D \frac{1}{2} \left(1 \pm \frac{1}{\polylog(d)} \right)  \frac{\sigma'(\langle w_i^{(t)}, u_{\ell} \rangle)  \pm\varepsilon_t  }{ \sqrt{\sum_{j \in [m_D] } \sigma'(\langle w_j^{(t)}, u_{\ell} \rangle)^2  }} \pm \eta_D \gamma
\end{align}

When $z = e_i$, the update of $v$ can be approximate as: For $\delta_t := O\left( \frac{1}{C^{0.94}} (\beta(t) + \frac{1}{d} )^2 \right) $: 
\begin{align}
v_i^{(t + 1)} &= v_i^{(t)} + \eta_G \left( 1 \pm \frac{1}{\polylog(d)} \right) \frac{ \sum_{j \in [m_D]}\sigma'(\langle v_i^{(t)}, w_j^{(0)} \rangle) w_j^{(0)}  \pm \delta_t }{\| \sum_{j \in [m_D]} \sigma'(\langle v_i^{(t)}, w_j^{(0)} \rangle) w_j^{(0)}  \|_2}
\end{align}

Where we have:
$$ \langle v_i^{(t)}, w_j^{(t)} \rangle  =\langle v_i^{(t)}, w_j^{(0)} \rangle \pm  \frac{1}{C^{0.9}} \beta(t)  $$
\end{lemma}

\begin{proof}[Proof of the update Lemma~\ref{lem:N1.1}]

 By the induction hypothesis, We have that:
 \begin{align}
     \langle v_i^{(t)}, w_j^{(t)} \rangle  &=\langle v_i^{(t)}, w_j^{(0)} \rangle +  \langle v_i^{(t)}, w_j^{(t)} -  w_j^{(0)} \rangle 
     \\
     &= \langle v_i^{(t)}, w_j^{(0)} \rangle +  \langle v_i^{(0)}, w_j^{(t)} -  w_j^{(0)} \rangle +  \langle v_i^{(t)} -v_i^{(0)} , w_j^{(t)} -  w_j^{(0)} \rangle
     \\
     &=  \langle v_i^{(t)}, w_j^{(0)} \rangle \pm \tilde{O}\left( \frac{1}{\sqrt{d}} \alpha(t) \right)
     \pm O\left( \frac{1}{C^{0.94}} \beta(t) \right)
     \\
     &=  \langle v_i^{(t)}, w_j^{(0)} \rangle \pm \frac{1}{C^{0.9}} \beta(t)
 \end{align} 
 Here we use the fact that $\|w_j^{(0)} - w_j^{(t)} \|_2 \leq O\left( \frac{1}{C^{0.95}} \right)$ from the induction hypothesis. 
 
 Consider the update of $w_i$, we have that: at stage 1, we must have $|f(X)|, |f(G(z))| \leq \frac{1}{\polylog(d)}$. Therefore, 
\begin{align}
    \nabla_{w_i} L(X, z) &= \left( 1 \pm \frac{1}{\polylog(d)} \right) a^{(t)} \sigma'(\langle w_i^{(t)}, X \rangle) X - a^{(t)}  \sigma'(\langle w_i^{(t)}, G^{(t)}(z) \rangle) G^{(t)}(z)
\end{align}

By the induction hypothesis, we have that by $\beta(t) \leq \alpha(t)$, it holds that:
$$\|  \sigma'(\langle w_i^{(t)}, G^{(t)}(z) \rangle) G^{(t)}(z) \|_2 \leq  O\left( \frac{\alpha(t)}{C} + \frac{1}{C \sqrt{d}} \right)^2 m_G^2 \times m_G \left( \frac{1}{\sqrt{d}} + \frac{\alpha(t)}{C} \right)  \leq \epsilon_t $$

On the other hand, we must have that when $X = u_{\ell}$, we have
\begin{align}
    \left( 1 \pm \frac{1}{\polylog(d)} \right)  \sigma'(\langle w_{i^*_{\ell}}^{(t)}, X \rangle)  \geq \left( \frac{1}{ m_D} \alpha(t) + \frac{1}{\sqrt{d}} \right)^2  \geq \polylog(d) \epsilon_t
\end{align}

This completes the proof of the $w_i$ part.  For $v_i$ part the proof is the same using the fact that $\|w_j^{(0)} - w_j^{(t)} \|_2 \leq O\left( \frac{1}{C^{0.95}} \right)$ from the induction hypothesis.

\end{proof}


\subsection{Stage 2 training}

In this stage, we can maintain the following simple update rule: For $w_i$:
\begin{lemma}\label{lem:N2.1}
For every $t \in (T_{N, 1}, T_1]$, we have that:
for every $i \in [m_D]$,
for $i = i^*_{\ell}$:
$$\E[w_i^{(t + 1)}] = w_i^{(t)} + \Theta (\eta_D ) u_{\ell} \pm \eta_D \frac{1}{C^{1.501}} \pm \eta_D \gamma$$

and for $i \not=i^*_1, i^*_2 $, 
$$\E[w_i^{(t + 1)} ]= w_i^{(t)} \pm \eta_D \frac{1}{C^{1.501}} \pm \eta_D \gamma$$

For $v_i$:
\begin{align}
    \E[v_i^{(t + 1)}] = v_i^{(t)} + \left( 1 \pm \frac{1}{\polylog(d)} \right) \frac{1}{m_G}\eta_G \frac{w_{g_i^*}^{(t)} }{\|w_{g_i^*}^{(t)} \|_2} \pm \eta_G \frac{1}{C^{1.5}}
\end{align}

\end{lemma}

\begin{proof}[Proof of Lemma~\ref{lem:N2.1}]
This Lemma can be proved identically to Lemma~\ref{lem:N1.1}: By the induction hypothesis, we have
\begin{align}
   | \langle w_i^{(t)} , v_j^{(t)} \rangle | \leq \log^5 \beta(t)
\end{align}

Therefore, 
$$ \|\sigma'( \langle w_i^{(t)} ,G^{(t)}(z) \rangle) v_j^{(t) } \|_2 \leq C^{0.01} \beta(t)^3 \leq \frac{1}{C^{1.51}} \alpha(t)^2 $$

Which implies that:
\begin{align}
w_i^{(t + 1)} &= w_i^{(t)} + \eta_D    \frac{\sigma'(\langle w_i^{(t)}, X \rangle) X \pm \frac{1}{C^{1.51} \alpha(t)^2} }{ \sqrt{\sum_{j \in [m_D] } (\sigma'(\langle w_j^{(t)}, X \rangle)^2 \| X \|_2^2 + \sum_{j \in [m_D] } \left( \sigma'(\langle w_j^{(t)}, X \rangle)^2 \right)^3  }}
\end{align}

Where $\sum_{j \in [m_D] } \left( \sigma'(\langle w_j^{(t)}, X \rangle)^2 \right)^3$ comes from the gradient of $a^{(t)}$. By the induction hypothesis we have that $a^{(t)} = \Omega(\alpha(w_{i_{\ell}^*}, u_{\ell} , t))$, so we have
\begin{align}
w_i^{(t + 1)} &= w_i^{(t)} + \Theta(\eta_D)    \frac{\sigma'(\langle w_i^{(t)}, X \rangle) X \pm \frac{1}{C^{1.51} \alpha(t)^2} }{ \sqrt{\sum_{j \in [m_D] } (\sigma'(\langle w_j^{(t)}, X \rangle)^2 \| X \|_2^2  }}
\end{align}

On the other hand, by the induction hypothesis, for $\ell \in 2[$: For $i = i^*_{\ell}$:
$\langle w_i^*, u_{\ell} \rangle \geq  \frac{1}{m_D} \alpha(t)$, and for $i \not= i^*_1, i^*_2$:
$|\langle w_i^*, X \rangle| \leq  O\left( \frac{1}{C} \alpha(t) \right)$. 

This implies that: for $i = i^*_{\ell}$:
$$\E[w_i^{(t + 1)}] = w_i^{(t)} + \Theta(\eta_D) u_{\ell} \pm \eta_D \frac{1}{C^{1.5}} \pm \eta_D \gamma$$

and for $i \not=i^*_1, i^*_2 $, 
$$\E[w_i^{(t + 1)}] = w_i^{(t)} \pm \eta_D \frac{1}{C^{1.5}}\eta_D \gamma$$

Where the additional $\gamma$ factor comes from the case when $X = u_1 + u_2$ or $X = 0$.

On the other hand, we also know that:
\begin{align}
    &\sum_{i \in [m_D]} \sigma'( \langle w_i^{(t)} , v_j^{(t)}\rangle) w_i^{(t)}
    \\
    &= \sigma'( \langle w_{g_j^*}^{(t)} , v_j^{(t)}\rangle) w_{g_j^*}^{(t)} \pm m_D \left(\frac{1}{C^{0.9}}  \right)^2 \langle w_{g_j^*}^{(t)} , v_j^{(t)}\rangle^2 \polylogloglog(d)
    \\
    &= \sigma'( \langle w_{g_j^*}^{(t)} , v_j^{(t)}\rangle) w_{g_j^*}^{(t)} \pm  \sigma'( \langle w_{g_j^*}^{(t)} , v_j^{(t)}\rangle)\frac{1}{C^{1.6}}
\end{align}

Notice that $\| w_i^{(t)}  \|_2 = \Omega(1)$ so we complete the proof. 
\end{proof}

\subsection{Proof of the induction hypothesis}

Now it remains to prove the induction hypothesis:
\paragraph{Stage 1:}

In this stage, we will use the update Lemma~\ref{lem:N1.1}. By the induction hypothesis we know that for $X = u_{\ell}$,
\begin{align}
    \langle w_j^{(t)}, X\rangle &= \alpha(w_j, u_{\ell}, t) + \alpha(w_j, w_j, 0) \left\langle \frac{w_j^{(0)}}{\|w_j^{(0)}\|_2}, u_{\ell} \right\rangle \pm {O} \left( \frac{1}{C^{0.5}  \sqrt{d}} \right)
\end{align}

This implies that
$$\sum_{j} \sigma'(\langle w_j^{(t)}, X\rangle)^2 \|X\|_2^2 \geq \left( \frac{1}{m_D}\alpha(t) + \frac{1}{\sqrt{d}} \right)^2$$

Now, apply Lemma~\ref{lem:N1.1} we know that:
\begin{align}
  &  \alpha(t + 1) \geq \alpha(t)+ \Omega\left( \frac{1}{m_D} \right) \eta_D
  \\
  &
     \forall * \not= u_1, u_2: |\alpha(w_i, *, t + 1)| \leq |\alpha(w_i, *, t )| +\eta_D \frac{\epsilon_t}{\left( \frac{1}{m_D}\alpha(t) + \frac{1}{\sqrt{d}} \right)} \leq  |\alpha(w_i, *, t )| +\eta_D \frac{1}{C^{1.4}}
\end{align}

Compare these two updates we can prove the bounds on $w_j$ for $* \not= u_1, u_2$. For $* = u_1, u_2$, we can see that: By Lemma~\ref{lem:N1.1}, there exists  $S_{t, \ell} \in (0, \poly(d)]$ such that for $\ell \in [2]$ such that for every $i \in [m_D]$:
\begin{align}
   & \sum_{i \in [m_D]}  \langle w_i^{(t )}, u_{\ell} \rangle^4 = \frac{1}{S_{t, \ell}^2 }
    \\
    &\E[\langle w_i^{(t + 1)}, u_{\ell} \rangle ]= \langle w_i^{(t )}, u_{\ell} \rangle + \eta_D \left( 1 \pm \frac{1}{\polylog(d)} \right) S_{t, \ell} \langle w_i^{(t )}, u_{\ell} \rangle^2 \pm \eta_D \frac{1}{\polylog(d)} 
\end{align}

Apply Lemma~\ref{lem:powerful}  and Lemma~\ref{lem:init2} we can complete the proof that  $$|\alpha(w_i, u_1, t)| , | \alpha(w_i, u_2, t) | \leq \alpha(t), \quad \alpha(w_{i_{\ell}}, u_{\ell}, t) \geq \frac{\eta_D}{3m_D} $$

and at iteration $t = T_{N, 1}$, we have that: for all $i \not= i_1^*, i_2^*$, for all $\ell \in [2]$
 $$|\alpha(w_i, u_{\ell}, t)|  \leq \frac{1}{C}\alpha(t)$$
 
 Moreover, when $i = i_{\ell^*}$, $|\alpha(w_i, u_{3 - \ell}, t)| \leq \frac{1}{C}\alpha(t)$

The $v$ part can be proved similarly:
We have that there exists $S_{t, i} \in (0, \poly(d)]$ where $i \in [m_G]$ such that:
\begin{align}
   & \sum_{j \in [m_D]}  \langle v_i^{(t )}, w_{j}^{(0)} \rangle^4 = \frac{1}{S_{t, i}^2 }
    \\
    &\E[\langle v_i^{(t + 1)}, w_{j}^{(0)} \rangle ]= \langle v_i^{(t )}, w_{j}^{(0)}\rangle + \eta_G \frac{1}{m_G} \left( 1 \pm \frac{1}{\polylog(d)} \right) S_{t, i} \langle v_i^{(t )}, w_{j}^{(0)}\rangle^2 \pm \eta_G \frac{\log^5 (d)}{C}
\end{align}

Apply Lemma~\ref{lem:powerful}  and Lemma~\ref{lem:init2}, we have that 
$$|\alpha(v_i, w_j, t)| \leq \beta(t), \quad  \alpha(v_i, w_{g_i^*}, t) \geq \frac{\eta_G}{3m_G^2}$$. Moreover, at iteration $t = T_{N, 1}$, for all $i \in [m_G]$, $j \in [m_D], j \not= g_i^*$:
$$|\langle v_i^{(t)}, w_j^{(0)} \rangle| \leq C^{-0.95}\langle v_i^{(t)}, w_{g_i^*}^{(0)} \rangle$$

Similarly, we can show that for all $* \not= w_j$, $|\alpha(v_i, *, t)| \leq \beta(t)$ and at iteration $t  = T_{N, 1}$: 
$$|\alpha(v_i, *, t)| \leq \frac{1}{C^{0.95}} \beta(t)$$

Using the fact that $\|w_j^{(t)} - w_j^{(0)} \|_2 \leq \frac{1}{C^{0.94}}$
\begin{align}
    \langle v_i^{(t)}, w_j^{(t)} \rangle &=  \langle v_i^{(t)}, w_j^{(0)} \rangle +  \langle v_i^{(t)}, w_j^{(t)} - w_j^{(0)} \rangle = \langle v_i^{(t)}, w_j^{(0)} \rangle \pm  \frac{\beta(t)}{C^{0.93}} 
\end{align}

Notice that $\langle v_i^{(t)}, w_j^{(t)} \rangle  \geq \beta(t) C^{-0.01}$ so we show that  at iteration $t  = T_{N, 1}$:
$$|\langle v_i^{(t)}, w_j^{(t)} \rangle| \leq C^{-0.91} \langle v_i^{(t)}, w_{g_i^*}^{(t)} \rangle$$

Similarly, we can show that for every $\ell \in [2]$, $$|\langle v_i^{(t)}, w_j^{(t)} \rangle| \leq C^{-0.91} \langle v_i^{(t)}, u_{\ell} \rangle$$
\paragraph{Stage 2:}

It remains to prove that  for all $t \in [T_{N,1}, T_{N, 2}]$, we have that
$$|\langle v_i^{(t)}, w_j^{(t)} \rangle| \leq C^{-0.9} \langle v_i^{(t)}, w_{g_i^*}^{(t)} \rangle$$

The rest of the induction hypothesis follows trivially from Lemma~\ref{lem:N2.1}. (for the relationship between $a^{(t)}$ and $\alpha(w_{i_{\ell}^*}, u_{\ell}, t)$ we can use Lemma~\ref{lem:comp}).

To prove this, we know that by the update formula:
\begin{align}
    \langle  v_i^{(t + 1)}, w_j^{(t + 1)} \rangle &=   \langle  v_i^{(t + 1)}, w_j^{(t )} \rangle  + \langle  v_i^{(t + 1)}, w_j^{(t + 1)} - w_j^{(t )} \rangle 
    \\
    &= \langle  v_i^{(t )}, w_j^{(t )} \rangle  +\langle  v_i^{(t  + 1)} -  v_i^{(t )}, w_j^{(t )} \rangle +  \langle  v_i^{(t + 1)}, w_j^{(t + 1)} - w_j^{(t )} \rangle 
\end{align}

Taking expectation, we have that
\begin{align}
    \E[\langle  v_i^{(t + 1)}, w_j^{(t + 1)} \rangle]  &= \langle  v_i^{(t )}, w_j^{(t )} \rangle + \eta_G \left( 1 \pm \frac{1}{\polylog(d)} \right) \frac{1}{m_G} \frac{\langle w_{g_i^*}^{(t)}, w_j^{(t) } \rangle }{\|w_{g_i^*}^{(t)} \|_2} \pm \eta_G \frac{1}{C^{1.409}}
    \\
    &+ \sum_{\ell \in [2]}\frac{\eta_D}{2}\langle  v_i^{(t + 1)},  u_{\ell}  \rangle 1_{j = i_{\ell}^*}  \pm \eta_D \frac{1}{C^{1.501}}
\end{align}

and 
\begin{align}
    \E[\langle  v_i^{(t + 1)}, u_{\ell} \rangle]  &= \langle  v_i^{(t )},u_{\ell} \rangle + \eta_G \left( 1 \pm \frac{1}{\polylog(d)} \right) \frac{1}{m_G} \frac{\langle w_{g_i^*}^{(t)}, u_{\ell}\rangle }{\|w_{g_i^*}^{(t)} \|_2} \pm \eta_G \frac{1}{C^{1.5}}
\end{align}

by the induction hypothesis we know that  for every $t\leq T_{N, 2}$, we have that $\langle  w_j^{(t)}, w_{j - 1}^{(t)} \rangle \leq \frac{1}{C^{0.95}} $ and $\|w_j^{(t)} \|_2 = [\Omega(1), \polylogloglog(d)]$, we know that: when $j = g_i^*$
\begin{align}
     \E[\langle  v_i^{(t + 1)}, w_j^{(t + 1)} \rangle]  \geq  \langle  v_i^{(t )}, w_j^{(t )} \rangle  + \eta_G \frac{1}{2m_G \polylogloglog(d)} 
\end{align}

When $j \not= g_i^*$: using the fact that $\eta_G = \eta_D C^{-0.6}$, we have:
\begin{align}
     \E[|\langle  v_i^{(t + 1)}, w_j^{(t + 1)} \rangle|]  \leq | \langle  v_i^{(t )}, w_j^{(t )} \rangle | + \eta_G \frac{1}{C^{0.9001}} 
\end{align}

When $i_{\ell}^* \not= g_i^*$, we have that:
\begin{align}
     \E[|\langle  v_i^{(t + 1)}, u_{\ell} \rangle|] \leq |\langle  v_i^{(t )}, u_{\ell} \rangle| +  \eta_G \frac{1}{C^{0.95}}
\end{align}

Thus we complete the proof.

\subsection{Proof of the final theorem}

To prove the final theorem, notice that by Lemma~\ref{lem:N2.1}, we have that for every $t \in (T_{N, 1}, T_1]$, 
for $i = i^*_{\ell}$:
$$\E[w_i^{(t + 1)}] = w_i^{(t)} +\Theta(\eta_D)  u_{\ell} \pm \eta_D \frac{1}{C^{1.501}} \pm \eta_D \gamma$$

Together with the induction hypothesis, this implies that when $\|w_i^{(t)}\|_2 \geq \log\log\log(d)$, we have that $\langle w_i^{(t)}, u_{\ell} \rangle \geq (1 - o(1)) \|w_i^{(t)}\|_2$. Together with the update formal of $v_j^{(t)}$ we know that when $g_j^* = i^*_{\ell}$, we have that
\begin{align}
    \E[v_j^{(t + 1)}] = v_j^{(t)} + \left( 1 \pm \frac{1}{\polylog(d)} \right) \frac{1}{m_G}\eta_G \frac{w_{g_j^*}^{(t)} }{\|w_{g_j^*}^{(t)} \|_2} \pm \eta_G \frac{1}{C^{1.5}}
\end{align}

Together with the induction hypothesis, we know that when $\|w_i^{(t)}\|_2 = \polyloglog(d)$, we have that:
$\langle v_j^{(t)}, u_{\ell} \rangle \geq (1 - o(1)) \|v_j^{(t)}\|_2$. This proves the theorem.

\end{document}